\documentclass[12pt,letterpaper,fleqn]{article}

\usepackage[onehalfspacing]{setspace}
\usepackage{graphicx,hyperref,pgfplots,dcolumn,ctable,floatpag,ulem}
\usepackage{array,multirow,float,subcaption}
\usepackage{comment,etex,authblk,amsmath,fullpage,amsfonts,titletoc,titlesec,amssymb,theorem}

\usepackage{booktabs}   
\usepackage{array}      
\usepackage{adjustbox}

\newlength\mytemplength

\titleformat*{\section}{\large\bfseries}
\titleformat*{\subsection}{\normalsize\bfseries}
\usepackage{graphicx}

\hypersetup{colorlinks=true,
linkcolor=blue,
citecolor=blue,
urlcolor=blue}

\usetikzlibrary{pgfplots.groupplots}
\usetikzlibrary{patterns, arrows}
\usetikzlibrary{arrows,positioning,shadows,shapes,calc,fit,backgrounds}
\usetikzlibrary{decorations.pathreplacing,calc}
\tikzset{>=stealth}

\reversemarginpar

\floatpagestyle{empty}

\makeatletter
\def\pgfplots@drawtickgridlines@INSTALLCLIP@onorientedsurf#1{}
\makeatother

\usepackage{caption}
\usepackage{longtable}
\definecolor{jy}{RGB}{255,196,105}

\usepackage{mathrsfs}

\usepackage{natbib}
\makeatletter
\renewcommand{\bibsection}{
\begin{center}
\section*{\refname\@mkboth{\MakeUppercase{\refname}}
{\MakeUppercase{\refname}}}
\end{center}
}
\makeatother

\newtheorem{theorem}{Theorem}

\newtheorem{assumption}{Assumption}

\newtheorem{corollary}[theorem]{Corollary}

\newtheorem{lemma}{Lemma}

\newtheorem{proposition}[theorem]{Proposition}
\newtheorem{theorem-app}{Theorem}[section]
\newtheorem{lemma-app}[theorem-app]{Lemma}
\newtheorem{proposition-app}[theorem-app]{Proposition}

\usepackage{natbib}
\makeatletter
\renewcommand{\bibsection}{
\begin{center}
\section*{\refname\@mkboth{\MakeUppercase{\refname}}
{\MakeUppercase{\refname}}}
\end{center}
}
\makeatother

\newenvironment{proof}[1][\proofname]{
\par\normalfont\trivlist\item[\hskip\labelsep\textbf{#1}.]\ignorespaces}
{\hfill $\square$ 
\endtrivlist}
\newcommand{\proofname}{Proof}

\newcommand{\gs}{\sigma}

\newcommand{\ga}{\alpha}

\newcommand{\eps}{\varepsilon}

\newcommand{\R}{{\mathbb R}}

\newcommand{\Var}{{\operatorname{Var}}}

\newcommand{\cL}{{\mathcal L}}

\newcommand{\cN}{\mathcal N}
\newcommand{\gS}{{\Sigma}}

\newcommand{\cI}{{\mathcal I}}

\newcommand{\Cov}{{\rm Cov}}
\newcommand{\bX}{{\bf X}}

\newcommand{\cF}{{\mathcal F}}
\newcommand{\cE}{{\mathcal E}}
\newcommand{\cM}{{\mathcal M}}

\newcommand{\cK}{{\mathcal K}}
\newcommand{\cB}{{\mathcal B}}

\newcommand{\tr}{\operatorname{tr}}

\newcommand{\cV}{{\mathcal{V}}}

\newcommand{\Prob}{\operatorname{Prob}}

\newcommand{\um}{\underline{m}}
\newcommand{\tum}{\widetilde{m}}

\usepackage{bbm}
\usepackage{color}

\usepackage{mathtools}

\newcommand{\be}{\begin{equation}}

\newcommand{\ee}{\end{equation}}
\newcommand{\bea}{\begin{eqnarray}}
\newcommand{\eea}{\end{eqnarray}}
\newcommand{\bee}{\begin{equation*}}
\newcommand{\eee}{\end{equation*}}

\newcounter{procedure}
\newenvironment{procedure}[1][]{
\refstepcounter{procedure}
\par\medskip
\noindent \textbf{Procedure~\theprocedure. #1}\rmfamily
\par\medskip
}{\par\medskip}

\begin{document}

\title{Limits To (Machine) Learning}

\author{Zhimin Chen, Bryan Kelly, and Semyon Malamud\footnote{Zhimin Chen is at Nanyang Technological University. Bryan Kelly is at AQR Capital Management, Yale School of Management, and NBER. Semyon Malamud is at the Swiss Finance Institute, EPFL, and CEPR, and is a consultant to AQR. Semyon Malamud gratefully acknowledges the financial support of the Swiss Finance Institute and the Swiss National Science Foundation, Grant 100018\_192692. Zhimin Chen gratefully acknowledges the financial support of the Singapore MOE AcRF Tier 1 Grant \#024891-00001. We are grateful to Mikhail Chernov and Andreas Fuster for excellent comments and suggestions. The implementation code is available at \href{here}{\url{https://github.com/czm319319/CKM_LLG}}. AQR Capital Management is a global investment management firm that may or may not apply similar investment techniques or methods of analysis as described herein. The views expressed here are those of the authors and not necessarily those of AQR.}}

\date{\begin{footnotesize}This version: \today\end{footnotesize}}

\maketitle

\begin{abstract} 

\noindent

Machine learning (ML) methods are highly flexible, but their ability to approximate the true data-generating process is fundamentally constrained by finite samples. We characterize a universal lower bound, the Limits-to-Learning Gap (LLG), quantifying the unavoidable discrepancy between a model’s empirical fit and the population benchmark. Recovering the true population $R^2$, therefore, requires correcting observed predictive performance by this bound. Using a broad set of variables, including excess returns, yields, credit spreads, and valuation ratios, we find that the implied LLGs are large. This indicates that standard ML approaches can substantially understate true predictability in financial data. We also derive LLG-based refinements to the classic \cite{hansen1991implications} bounds, analyze implications for parameter learning in general-equilibrium settings, and show that the LLG provides a natural mechanism for generating excess volatility.

\bigskip\bigskip\bigskip

\noindent
\begin{small}Keywords: machine learning, asset pricing, predictability, big data, limits to learning, excess volatility, stochastic discount factor, kernel methods
\\
\end{small}

\end{abstract}

\renewcommand{\thefootnote}{\number\value{footnote}}

\pagenumbering{arabic}
\def\baselinestretch{1.617}\small\normalsize%

\clearpage

\defcitealias{gibbons1989test}{GRS}

\section{Introduction}

The modern approach to assessing whether an economic variable $y$ is predictable from a set of variables $X$ is to train a Machine Learning (ML) model and evaluate its out-of-sample performance. In practice, this approach often concludes that predictability is low or nonexistent. We show that such conclusions can be misleading. In high-dimensional, richly parameterized ML settings, estimators do not converge to the ground truth even in large samples, leading standard out-of-sample metrics to dramatically understate the true (population) degree of predictability. Our objective in this paper is to provide a simple and practical correction, thereby opening a new avenue for testing predictive relationships.

Classical econometric theory is designed for data-rich environments in which the number of parameters is small relative to the sample size. In such settings, the law of large numbers and central limit theorems ensure that estimators rapidly approach the true model.\footnote{In the presence of model misspecification, estimators converge to a pseudo-true parameter; see \citet{white1996estimation}.} Modern ML applications, however, operate in a fundamentally different regime. When the number of parameters is comparable to or exceeds the number of observations, classical asymptotics no longer guarantees consistency.

Under the data-generating process
\begin{equation}
y_{t+1}\ =\ f_t\ +\ \eps_{t+1},
\end{equation}
modern ML estimators $\hat f_t$ typically exhibit substantial estimation error, which leads to large out-of-sample (OOS) Mean Squared Error (MSE):
\begin{equation}
MSE_{OOS}\ =\ \underbrace{E[(y_{t+1}\,-\,\hat f_t)^2]}_{out-of-sample\ error}\ =\ \underbrace{E[(f_t\,-\,\hat f_t)^2]}_{estimation\ error}\ +\ \underbrace{\gs_\eps^2}_{irreducible\ noise},
\end{equation}
where $\gs_\eps^2 = E[\eps_{t+1}^2]$. In this paper, we derive a lower bound on the estimation error of any linear estimator:
\begin{equation}
E[(f_t\,-\,\hat f_t)^2]\ \ge\ \gs_\eps^2\,\widehat\cL\,,
\end{equation}
where $\widehat\cL$ is the {\it Limits-to-Learning Gap (LLG)}, a fully data-driven lower bound that depends only on the observed sample and can be computed without estimating any predictive model.

LLG has direct implications for measured predictability. 
It yields a sharp lower bound for the true (population) $R^2$ (henceforth, $R^2_*$) in terms of the realized out-of-sample $R^2$ (henceforth, $R^2_{OOS}$). Namely, we show that 
\begin{equation}\label{first-time-l}
R^2_*\ \ge\ \frac{R^2_{OOS}+\widehat\cL}{1+\widehat\cL}\,.
\end{equation}
We also derive a Central Limit Theorem (CLT), allowing us to construct a one-sided confidence interval for $R^2_*.$

Our bound provides a practical diagnostic tool. A researcher observing a large OOS MSE might conclude that there is no predictability. But if $\widehat\cL$ is large, the true $R^2_*$ may be far higher than the raw $R^2_{OOS}$ suggests—meaning that substantial predictability exists in principle, even if standard ML models fail to uncover it. LLG, therefore, helps distinguish between cases where predictability is genuinely absent and those where existing models lack the capacity or architecture to detect it. It can thus guide variable selection and model design.

We provide extensive simulation evidence showing that our lower bound (CLT-adjusted for finite-sample error) closely approximates the true $R^2_*$ from below. We then apply our framework to the classical dataset of \cite{welch2008comprehensive} to quantify the degree of predictability in US market returns and a broad set of financial indicators, including valuation ratios, bond yields, and spreads. We find that LLG is often large, even in settings traditionally viewed as weakly predictable. For example, although the $R^2_{OOS}$ for market returns is negative in our exercises, the LLG implies that the true population predictability is at least 20\%. This result stands in sharp contrast to the 1–2\% $R^2_{OOS}$ achieved by state-of-the-art ML models at the monthly horizon (see \cite{kelly2024virtue, kelly2025understanding}), and poses a significant challenge for standard asset-pricing models. Likewise, LLG corrections indicate that the AR(1) residuals of many valuation ratios and spreads are more than 30\% predictable, even though conventional linear ML methods fail to detect this structure.

These findings suggest that widely used ML methods may substantially understate predictability in macroeconomic and asset-return data. When the lower bound for $R^2_*$ is large, more flexible econometric or structural approaches may be capable of recovering the underlying signal. As an illustration, we apply nonlinear models with variable selection to forecast each of the \cite{welch2008comprehensive} variables. For several targets with a high lower bound \eqref{first-time-l}, more complex models uncover meaningful out-of-sample predictability. For instance, the lower bound implies that AR(1) residuals of the Treasury Bill rate are 10\% predictable, and our best nonlinear specification indeed attains an $R^2_{OOS}$ of 10\%.

Our results also speak to theories of learning in financial markets. A growing literature, including \cite{collin2016parameter} and \citet{farmer2024learning}, shows that parameter learning can be extremely slow, challenging full-information rational expectations. We derive an LLG-adjusted version of the \cite{hansen1991implications} bound that quantifies the role of parameter uncertainty in shaping asset prices.

In a simple equilibrium model with high-dimensional learning, we show that the LLG associated with predicting asset fundamentals naturally generates excess price volatility. This mechanism implies that rational overreaction to noisy signals can replicate the classic excess-volatility puzzle of \cite{shiller1981excess} without invoking behavioral departures from rationality.

Finally, our framework has implications for the economics of large-scale ML systems more broadly. Substantial computational and environmental resources are being devoted to scaling up ML models, yet empirical scaling laws \citep{kaplan2020scaling} show sharply diminishing returns, with convergence rates often following $T^{\alpha}$ for $\alpha$ near $0.1$, far below the classical $\sqrt{T}$ rate. Our results show that LLG explains this slowdown: the limits-to-learning gap declines only slowly with sample size. By quantifying this gap, LLG provides guidance on when additional data or computation is likely to generate meaningful improvements and clarifies the fundamental limits of what ML can extract from finite data.

\section{Literature Review} 

Machine learning methods have delivered remarkable performance in estimation, prediction, and portfolio optimization tasks within high-dimensional environments. See, for example, \cite{jurado2015measuring}, \cite{kleinberg2018human}, \cite{gu2020empirical}, \cite{bianchi2021bond}, \cite{cong2021alphaportfolio}, \cite{bianchi2022belief}, \cite{fan2022structural}, \cite{kaniel2023machine}, \cite{fuster2022predictably}, \cite{liao2023does}, \cite{chen2024deep}, \cite{kelly2024virtue}, \cite{didisheim2024apt}, and \cite{chernov2025test}. However, the high complexity of these models renders them prone to overfitting and bias. A substantial literature, therefore, has focused on debiasing techniques \citep{chernozhukov2018double,chernozhukov2019double} and on establishing technical conditions under which ML estimators achieve root-$T$ consistency \citep{chen1999improved,schmidt2020nonparametric,kohler2021rate, liao2025uncertainty}. 

A more recent line of research highlights fundamental statistical limits to the ability of economic agents to learn high-dimensional models. For example, \cite{da2022statistical} show that investors may be unable to efficiently exploit the full cross-section of alphas, while \cite{martin2021market} demonstrate that strong in-sample predictability may coexist with little or no out-of-sample performance. \cite{didisheim2024apt} and \cite{kelly2025understanding} refer to this phenomenon as “limits-to-learning.” In this paper, we formalize and quantify these insights by introducing a tractable, easy-to-estimate lower bound on model performance—the Limits-to-Learning Gap (LLG). We show that our estimator is super-consistent and apply it to a range of widely studied ML models. Using both simulated and real data, we find that LLG can be substantial even in models that incorporate sparsity-inducing structures, such as sparse neural networks \citep{chen1999improved,schmidt2020nonparametric}. In particular, we show how to compute LLG explicitly for any neural network trained via gradient descent, leveraging the recent empirical findings on the equivalence between neural networks and NTK-based ridge regressions. See \cite{fort2020deep, atanasov2021neural, lauditi2025adaptive, schwab2025}. Our results also contribute to the ML literature examining the subtle relationship between overfitting and out-of-sample performance  \citep{muthukumar2020harmless, holzmuller2020universality, belkin2020two, ghosh2023universal, mohsin2025fundamental, malach2025infinity}.

The LLG emerges due to the curse of dimensionality. The common way of dealing with this curse in modern econometrics is to assume sparsity of the underlying true model. However, recent theoretical and empirical evidence (see, e.g., \cite{giannone2021economic, avramov2025schrodinger}) suggests that sparsity can be difficult to detect in economic datasets. Moreover, even when sparsity is present, dense models tend to outperform sparse ones in low signal-to-noise ratio environments \citep{shen2024can}. Inference in dense, high-dimensional models, however, poses formidable challenges, as classical tools tend to fail in such settings. As \cite{belloni2018high} emphasize: {\it ``Most of the work in the recent literature on high-dimensional estimation and inference relies on approximate sparsity to provide dimension reduction and the corresponding use of sparsity-based estimators. Dense models are appealing in many settings and may be usefully employed in more moderate-dimensional settings."}
Our paper advances this literature by constructing pivotal confidence bands for key nuance parameters, $\gs_\eps^2$ and $R^2_*,$ in dense, high-dimensional environments.

Our findings also relate to the notion of “dark matter” in asset pricing introduced by \cite{chen2024measuring}, who quantify the extent to which asset pricing models rely on latent, unobservable structure inferred from internal model restrictions. A large LLG implies that much of the underlying structure of the data is effectively unlearnable from finite samples—even if it is theoretically exploitable—forcing ML models to rely on noisy approximations or implicitly inferred components, much like dark matter. This connection highlights how statistical limits to learning may underlie model fragility and overfitting, echoing the concerns raised in \cite{chen2024measuring}.

\section{Limits to Learning}\label{sec:pred}

\subsection{Environment}

We consider the problem of predicting an economic variable $y_{t+1}$ based on a (potentially) high-dimensional vector of signals $S_t\in \R^P$. The following assumption summarizes the basic properties of the data-generating process. While we focus our analysis on linear estimators, in Section~\ref{sec:ml} we demonstrate how our results naturally extend to nonlinear models, including kernel regressions and deep neural networks.

\begin{assumption}\label{ass1} We have 
\begin{equation}\label{d-st}
y_{t+1}\ =\ f_t\ +\ \eps_{t+1}\ ,\ t=0,\cdots,T+1\,,
\end{equation}
for some $f_t$, where $E[\eps_{t+1}]=0,\ E[\eps_{t+1}^2]=\gs_{\eps}^2,\ E[\eps_{t+1}^4]<\infty$ are independent and identically distributed, and are also independent of $(f_t,S_t)_{t\ge 0}$. Below, we frequently use the convenient matrix notation $y=(y_\tau)_{\tau=1}^T\in \R^T,$ and $S=(S_\tau)_{\tau=0}^{T-1}\in \R^{T\times P}$ for the in-sample (training) data, and we use time $T+1$-data as the out-of-sample realization. 
\end{assumption}

Importantly, no assumptions are made whatsoever about the nature of the joint dynamics of signals $S$ and the predictable component $f.$ In particular, this assumption allows for any form of non-stationarity, autocorrelation, heteroskedasticity, or model misspecification. 

\noindent By \eqref{d-st}, the second moment of $y_{t+1}$ admits the standard decomposition 
\begin{equation}
E[y_{t+1}^2]\ =\ \underbrace{E[f_t^2]}_{explained\ variance}\ +\ \underbrace{\gs_\eps^2}_{irreducible\ noise}\,.
\end{equation}
Given an estimator $\hat f_t$, we are interested in the behavior of the expected out-of-sample error
\begin{equation}\label{def-mse}
MSE(\hat f)\ =\ E[(y_{T+1}-\hat f_T)^2]\,,
\end{equation}
the corresponding $R^2$, 
\begin{equation}\label{eq:risk0}
\begin{aligned}
&R^2(\hat f)\ =\ 1 -\ \frac{MSE(\hat f)}{E[f_t^2] +\ \gs_\eps^2}\,,
\end{aligned}
\end{equation}
and its relationship with the {\it infeasible $R^2$}  given by 
\begin{equation}\label{r*2}
R^2_*\ =\ 1\ -\ \frac{\gs_\eps^2}{E[f_t^2]\ +\ \gs_\eps^2}\,
\end{equation}
and achievable by an econometrician who knows the true $f_t.$ 
Our goal is to study limits-to-learning, defined as the gap between $R^2_*$ and $R^2(\hat f).$

\subsection{Linear Estimators}

Our focus in this paper is on linear estimators, admitting a form 
\begin{equation}\label{f=ky}
\hat f_T\ =\ \cK(S_T,S)\,y,
\end{equation}
where $\cK(S_T,S)\in \R^{1\times T}$ is a vector function, with $(\cK(S_T,S))_t$ describing how much attention our estimator is paying to a particular observation at time $t$. A canonical example is the ridge-regularized least-squares estimator, 
\begin{equation}\label{4}
\begin{aligned}
\hat f_T\ =\ \underbrace{\left(S_T'(zI\ +\  \hat\Psi)^{-1} \frac{1}{T} S'\right)}_{\cK(S_T,S)}\,y\,,
\end{aligned}
\end{equation}
where 
\begin{equation}\label{eq: hat-psi-t}
\hat\Psi\ =\ \frac1{T} \sum_{\tau=0}^{T-1} S_\tau S_\tau'\ =\ \frac{1}{T}S'S
\end{equation}
is the sample covariance matrix of signals. This simple class of estimators has become a workhorse theoretical framework for understanding the behavior of high-dimensional machine learning models\footnote{\cite{hastie2019surprises} explain how it approximates gradient descent in the ``lazy training" regime, while \cite{jacot2018neural} show that very wide neural nets can be closely approximated by high-dimensional linear ridge regressions with signals given by the gradients of the Neural Network (the so-called Neural Tangent Kernel (NTK)). Recently, a sequence of papers \citep{fort2020deep, atanasov2021neural, lauditi2025adaptive, schwab2025} has made a surprising discovery that, in fact, the prediction of any neural net can be closely approximated by a kernel ridge regression. We use this discovery in Section \ref{sec:ml} to compute LLG for any neural network. } and has been shown to perform well even in environments with low signal-to-noise ratios \citep{shen2024can}. Another, closely related linear estimator is a kernel ridge regression: given a positive definite kernel $k(\cdot,\cdot),$\footnote{One popular choice is the Gaussian kernel, $k(x_1,x_2)=e^{-\|x_1-x_2\|^2}.$ See, e.g., \cite{kelly2024virtue, kelly2025understanding}.} we define 
\begin{equation}
\hat f_T\ =\ \underbrace{k(S_T,S) (zI+k(S,S))^{-1}}_{\cK(S_T,S)}y\,. 
\end{equation}
In Appendix \ref{bayes}, we show that when $y_{t+1}=\beta'S_t+\eps_{t+1}$ and $\beta$ is sampled at time zero so that $\beta_i$ are i.i.d. and $E[\beta_i]=0,\ E[\beta_i^2]=\gs_\beta^2,$ then the linear ridge regression estimator with $z= \dfrac{\gs_\eps^2}{\gs_\beta^{2}}\dfrac{P}{T}$ is Bayes-optimal; that is, {\it no other ML model (linear or nonlinear) can achieve better performance than ridge.} As a result, our LLG bound applies to all ML models in this setting. We also show how our results extend to a larger class of distributions with arbitrary covariance structures.

The following result presents a useful decomposition of the MSE \eqref{def-mse}. Given a partition of the data into $T$ in-sample observations and $T_{OOS}$ out-of-sample (OOS) observations, we define $E_{OOS}[X]=\dfrac1{T_{OOS}}\sum_{t=T}^{T+T_{OOS}-1}X_{t+1},$ and let $MSE_{OOS}$ be the realized out-of-sample MSE,
\begin{equation}\label{mse-oos-def}
MSE_{OOS}(\hat f)\ =\ \frac1{T_{OOS}}\sum_{t=T}^{T+T_{OOS}-1}(y_{t+1}-\hat f_t)^2\,,
\end{equation}
and let
\begin{equation}\label{MSE-OOS-R2}
\begin{aligned}
&R^2_{OOS}(\hat f)\ =\ 1\ -\ \frac{MSE_{OOS}(\hat f)}{E_{OOS}[y^2]}
\end{aligned}
\end{equation}
be the realized out-of-sample $R^2.$ The linearity of the estimator $\hat f_t$ leads to the natural decomposition, 
\begin{equation}
\hat f_t\ =\ \hat f^s_t\ +\ \hat f^\eps_t\ =\ \underbrace{\cK(S_t,S)f}_{signal}\ +\ \underbrace{\cK(S_t,S)\eps}_{noise}\, 
\end{equation}
for each OOS time period $t$. Substituting this decomposition into \eqref{mse-oos-def} leads to the following result. 

\begin{proposition}\label{dec} 
For the linear estimator, we have 
\begin{equation}\label{eq:risk3_expanded-main}
\begin{aligned}
&MSE_{OOS}(\hat f)\ =\ E_{OOS}[\eps^2]\ +\ \widehat\cB\ +\ \widehat\cI\ +\ \widehat\cV\,,
\end{aligned}
\end{equation}
where 
\begin{equation}\label{eq:risk3_expanded1-main}
\begin{aligned}
&\widehat\cB\ =\ \underbrace{E_{OOS}[(f-\hat f^s)^2]}_{bias}\ \ge\ 0\\
&\widehat\cV\ =\ \underbrace{E_{OOS}[(\hat f^\eps)^2]}_{variance}\ \ge\ 0,
\end{aligned}
\end{equation}
while 
\begin{equation}\label{eq:risk3_expanded2-main}
\widehat\cI\ =\ \underbrace{-2E_{OOS}[(f-\hat f^s)(\hat f^\eps-\eps)+\eps \hat f^\eps]}_{interaction}
\end{equation}
is the interaction term.
\end{proposition}

Standard law-of-large numbers arguments imply that the interaction term, $\widehat\cI,$ is asymptotically negligible, vanishing at the rate $O(T^{-1/2}+T_{OOS}^{-1/2})$. However, the terms $\widehat\cB,\ \widehat\cV$ are always nonnegative and do not vanish, creating limits-to-learning. Estimating the true bias term $\widehat\cB$ in high-dimensional settings is highly non-trivial and requires novel techniques and additional assumptions. Ignoring this term, we get the lower bound for the MSE: 
\begin{equation}
MSE_{OOS}(\hat f)\ \ge\ \gs_\eps^2\ +\ \widehat\cV\ +\ O(T^{-1/2})\,. 
\end{equation}
Let $\widehat \cK=(\cK(S_\tau,S))_{\tau=T}^{T+T_{OOS}-1}\in \R^{T_{OOS}\times T}\,.$ 
Our key observation is that $\widehat\cV$ can be written down as a {\it quadratic form,}
\begin{equation}
\widehat\cV\ =\ \frac1T \,\eps'\,\Big(\dfrac{T}{T_{OOS}}\widehat\cK'\widehat\cK \Big)\,\eps\ \underbrace{\approx}_{Law\ of\ Large\ Numbers}\ \frac1T\,\gs_\eps^2\,\tr\Big(\dfrac{T}{T_{OOS}}\widehat\cK'\widehat\cK \Big)\,.
\end{equation}
Making this argument rigorous, we arrive at the main result of this section. 

\begin{theorem}\label{llg1} Suppose that, in the limit as $T,T_{OOS}\to\infty$, 
\begin{equation}
\lim \frac1{T_{OOS}^2}\tr E \Bigr[(\widehat\cK'\widehat\cK)^2\Bigr]\ =\ \lim \frac1{T_{OOS}^2}E \Bigr[E_{OOS}[\|(f-\hat f^s)\|^2] \Bigr]\ =\ \lim \frac1{T_{OOS}^2}E \Bigr[E_{OOS}[\|(f-\hat f^s)\widehat\cK\|^2] \Bigr]\ =\ 0\,. 
\end{equation}
Then, the {out-of-sample} MSE from \eqref{mse-oos-def} satisfies
\begin{equation}\label{llg2}
\liminf \frac{MSE_{OOS}(\hat f)}{1+\widehat\cL}\ \ge\ \underbrace{\gs_\eps^2}_{infeasible}\,
\end{equation}
in probability,\footnote{We say that $X_T\ge Y_T$ holds in probability if $\lim_{T,T_{OOS}\to\infty}\Prob(X_T<Y_T)=0.$} 
where 
\begin{equation}\label{def-cL}
\widehat\cL\ =\ \frac{1}{T_{OOS}}\tr \Bigr(\widehat\cK'\widehat\cK\Bigl)
\end{equation}
is the Limits-to-Learning Gap (LLG). This bound turns into an identity if $f_t=0.$
\end{theorem}

\noindent Suppose now that $\lim_{T_{OOS}\to\infty}E_{OOS}[f^2]=E[f_t^2].$\footnote{This holds, for example, if $f_t^2$ is stationary ergodic.} Then, $\lim E_{OOS}[y^2]=E[f_t^2]+\gs_\eps^2$ and we can rewrite \eqref{llg2} as 
\begin{equation}\label{llg3-first}
R^2_*\ \ge\ \lim\sup \frac{R^2_{OOS}(\hat f)+\widehat\cL}{1+\widehat\cL}\,,
\end{equation}
or, equivalently, \eqref{llg3-first} as
\begin{equation}\label{llg3-onur}
\underbrace{R_*^2 - R^2_{OOS}(\hat f)}_{\textit{the gap}} \geq \widehat\cL (1-R_*^2)\ +\ o(1)\,.
\end{equation}
Thus, one can see explicitly how $\widehat\cL$ also controls the gap between the infeasible and feasible $R^2.$ This gap emerges because, in high-dimensional settings, regression overfits noise. Perhaps the most surprising implication of Theorem \ref{llg1} is that LLG {\it only depends on the signals $S$,} and is completely independent of the structure of $f_t.$ Thus, the nature of the dependent variable $y_{t+1}$ is irrelevant: The same LLG will emerge no matter what stands on the left-hand side of the regression. 

What defines the size of $\widehat\cL$? The underlying mechanisms are particularly explicit for the ridge estimator \eqref{4}. In this case, by direct calculation, we have 
\begin{equation}
\frac{T}{T_{OOS}}\widehat\cK'\widehat\cK\ =\ \frac{1}{T}S(zI+\hat\Psi)^{-1}\hat\Psi_{OOS}(zI+\hat\Psi)^{-1}S'
\end{equation}
where $\hat\Psi_{OOS}=E_{OOS}[SS']$ is the out-of-sample signal covariance matrix and the technical condition of Theorem \ref{llg1} holds when $E[\bigl\|E_T[\hat\Psi_{OOS}^2]\bigr\|] = o(\min\{T_{OOS},T\})$ as $T_{OOS}\to\infty$. Then, 
\begin{equation}\label{old-llg}
\widehat\cL(z)\ =\ \frac{1}{T}\tr \Bigr(\hat\Psi_{OOS}\hat\Psi (zI+\hat\Psi)^{-2} \Bigr)\,. 
\end{equation}
Here, $\hat\Psi_{OOS}\hat\Psi (zI+\hat\Psi)^{-2}$ is a $P\times P$ matrix, and, hence, its trace grows approximately like $P,$ so that $\widehat\cL=O(c),$ where we have defined statistical complexity $c= \dfrac{P}{T}$ as in \cite{kelly2024virtue}. When $P$ is negligible relative to $T,$ the LLG vanishes. By contrast, in the over-parametrized regime when $P>T,$ $\widehat\cL$ can get very large, creating significant limits to learning.

\subsection{Asymptotic Normality}

In this subsection, we establish asymptotic normality of our estimator, allowing us to compute confidence bands for the infeasible $R^2_*.$ Without imposing additional technical conditions on $f_t,$ we cannot determine the estimation error of its second moment, $E[f_t^2]-E_{OOS}[f^2].$ For this reason, instead of $R^2_*,$ we work with a slightly modified object,  
\begin{equation}
\tilde R_*^2\ =\ 1\ -\ \frac{\gs_\eps^2}{E_{OOS}[f^2]+\gs_\eps^2}
\end{equation}
The following is true. 

\begin{theorem}[Confidence Interval for The Infeasible $R^2$]\label{main-th-main-text} Suppose that $E[\eps_t^3]=0,\ E[\eps_t^4]=3.$  Then, 
\begin{equation}\label{llg3-2}
\frac{R^2_{OOS}(\hat f)+\widehat\cL}{1+\widehat\cL}\,,
\end{equation}
is a $T^{1/2}$-consistent lower bound for $\tilde R_*^2$ in the following sense: The event $\dfrac{R^2_{OOS}(\hat f)+\widehat\cL(z)}{1+\widehat\cL(z)}>R_*^2$ occurs with vanishing probability: 
\begin{equation}
\lim\sup_{T,T_{OOS}\to\infty} \Prob\left(\dfrac{T^{1/2}\left(\tilde R_*^2-\dfrac{R^2_{OOS}(\hat f)+\widehat\cL}{1+\widehat\cL}\right)}{\gs_{R^2}}<\ga\right)\ \le\ \Phi(\ga)
\end{equation}
for any $\ga\le 0,$ where $\Phi(\cdot)$ is the c.d.f. of the standard normal distribution. 
There exists an asymptotically consistent, pivotal estimator $\hat\gs_{R^2}=\hat\gs_{R^2}(y,S)$ of $\gs_{R^2}$ presented in the Appendix.\footnote{See Theorem \ref{main-th-r-ridge}. The expression for $\hat\sigma_{R^2}$ is too long for the main text.}
\end{theorem}
The proof of Theorem \ref{main-th-main-text} is non-trivial. A key challenge is that $\sigma_{R^2}$ is expressed in terms of the unobservable and impossible to estimate $f_t$ and $\sigma_\eps^2.$ Overcoming this challenge requires some techniques that we develop in the Appendix \ref{theBig}. 

\subsection{Limits to Learning for (Deep) Neural Networks} \label{sec:ml}

Deep Neural Networks (DNN) are complex, nonlinear families of functions, $f(x;\theta),$ where the parameter vector $\theta\in \R^P$ is typically very high-dimensional and $x\in \R^d$. Given some input data $X_t\in \R^d$ we define the in-sample MSE 
\begin{equation}
\cL(\theta)\ =\ \frac1T\sum_{t=0}^{T-1} (y_{t+1}-f(X_t;\theta))^2\,.
\end{equation}
Then, we pick a learning rate $\eta,$ and the parameter vector $\theta$ is optimized recursively to achieve a low $\cL(\theta)$ via gradient descent:
\begin{equation}
\theta_{e+1}\ =\ \theta_e\ -\ \eta\,\cL(\theta_e),\ e=0,\cdots,\cE\,.
\end{equation}
Here, the number of gradient descent steps $\cE$ is commonly referred to as the {\it number of epochs.} The prediction of the DNN is then given by $f(x;\theta_\cE).$ 
A surprising discovery made recently in a sequence of papers is that, in fact, the prediction of the DNN is closely approximated by the linear regression with signals 
\begin{equation}\label{st-ntk}
S_t\ =\ \nabla_\theta f(X_t;\theta_\cE)\in\R^P\,. 
\end{equation}
See, e.g., \cite{jacot2018neural, fort2020deep, geiger2020disentangling, liu2020linearity, atanasov2021neural, vyas2022limitations, lauditi2025adaptive, schwab2025}. Here, we appeal to the powerful result from \cite{yang2021tensor} showing that wide  DNNs trained by gradient descent are approximately linear in $y$. 

\begin{proposition}[LLG for Wide Neural Networks]\label{prop:yang} Under the hypothesis of \cite{yang2021tensor}, in the limit of large width, the prediction of a DNN trained by gradient descent is approximately linear in $y$ with the $\cK(S_T,S)$ in \ref{f=ky} that admits an analytical expression in terms of \eqref{st-ntk}; hence, the LLG can be computed for the DNN using \eqref{def-cL}. 
\end{proposition}

\section{Implications of Limits-to-Learning for Models of Dynamic Economies}

\subsection{Hansen-Jagannathan Bounds}

Suppose that $y_{t+1}=R_{t+1}$ is the excess return on a security (e.g., a market index such as the SP500). In this case, we can define the timing strategy 
\begin{equation}
R^\pi_{t+1}\ =\ \pi_t R_{t+1}\,,
\end{equation}
where $\pi_t$ is the conditionally efficient portfolio, 
\begin{equation}
\pi_t\ =\ \frac{E_t[R_{t+1}]}{\Var_t[R_{t+1}]}\ =\ \frac{f_t}{\gs_\eps^2}\,
\end{equation}
with the conditional squared Sharpe Ratio 
\begin{equation}
\frac{E_t[R_{t+1}^\pi]^2}{\Var_t[R^\pi_{t+1}]}\ =\ f_t^2\gs_\eps^{-2}
\end{equation}
satisfying (see \eqref{r*2})
\begin{equation}
E[f_t^2\gs_\eps^{-2}]\ =\ \frac{R_*^2}{1-R_*^2}\,. 
\end{equation}
This identity creates a direct link between the Sharpe ratio (a measure of economic significance) and $R_*^2$ (a measure of statistical significance).\footnote{It is instructive to relate our analysis to \cite{barillas2018comparing}. While \cite{barillas2018comparing} focus on comparing models, our approach instead identifies the potential of the best possible model and the gap relative to what is empirically achievable. In this sense, we provide a cardinal measure of an asset pricing model’s performance.} We now combine this intuition with Theorem \ref{llg1} to derive implications of LLG for asset pricing. 

Most equilibrium asset pricing models imply that a stochastic discount factor (SDF) can be constructed from the Intertemporal Marginal Rate of Substitution (IMRS) of economic agents. Such an SDF $\widetilde{\cM}_t$ satisfies the pricing equation
\begin{equation}\label{objective}
E_t[\widetilde{\cM}_{t+1} R^\pi_{t+1}]\ =\ 0\,,
\end{equation}
and the Cauchy-Schwarz inequality implies that
\begin{equation}\label{hj-1}
\frac{\Var_t[\widetilde{\cM}_{t+1}]}{E_t[\widetilde{\cM}_{t+1}]^2}\ \ge\  f_t^2\gs_\eps^{-2}\,. 
\end{equation}
This is commonly known as the \cite{hansen1991implications} bound. Taking the expectation of this inequality and combining it with Theorem \ref{llg1}, we arrive at the following result, assuming that $\lim_{T_{OOS}\to\infty}E_{OOS}[f^2]=E[f_t^2].$ 

\begin{proposition}[LLG for Hansen-Jagannathan bounds]\label{param-learning} Consider an equilibrium asset pricing model with $d_{t+1}$ satisfying Assumption \ref{ass1} and where an economic agent know $f_t$ in \eqref{d-st}. Then, this agent's IMRS is a non-tradable SDF $\widetilde{\cM}_t$ and it satisfies 
\begin{equation}\label{hjb-m}
E\left[\frac{\Var_t[\widetilde{\cM}_{t+1}]}{E_t[\widetilde{\cM}_{t+1}]^2}\right]\ \ge\ \lim\sup_{T,T_{OOS}\to\infty} \frac{R^2_{OOS}(\hat f)+\widehat\cL}{1-R_{OOS}^2(\hat f)}\,. 
\end{equation}
\end{proposition}

Proposition \ref{param-learning} shows how limits-to-learning directly translate into a lower bound on the variance of the SDF in a standard complete-information equilibrium. As an illustration, suppose that $y_{t+1}$ is the return on the CRSP value-weighted index. Although realized out-of-sample $R^2$ from predicting market returns is typically small—around $1$–$2\%$ at the one-month horizon—we find that $\widehat\cL$ can be large. Consequently, Proposition \ref{param-learning} poses a significant challenge for models driven by macroeconomic shocks, since such models often struggle to generate a sufficiently high $\Var_t[\widetilde{\cM}_{t+1}]$. As we argue, this puzzle can be partially resolved by accounting for high-dimensional parameter learning.

The key caveat is the assumption that agents know $f_t$. The classic rational-expectations approach in modern dynamic stochastic general equilibrium models typically presumes that agents know all parameters of the data-generating process. This assumption has recently faced growing criticism (\citealp{han2021deepham}; \citealp{moll2024trouble}; \citealp{moll2025mean}; \citealp{dou2025ai}) because these models often contain a very large number of parameters, many of which are difficult to estimate, leading to extremely slow learning. See, for example, \cite{collin2016parameter} and \cite{farmer2024learning}. Hence, it is natural to assume that, just like econometricians, economic agents face the same limits-to-learning in richly parametrized environments.

Consider now an econometrician who has access to the same number $T$ of observations as the economic agent in our equilibrium model. Alternatively, we may assume that the econometrician evaluates the SDF on a purely out-of-sample basis: Even with access to $\bar T>T$ return realizations ex-post, only $t\le T$ are used for time-$T$ estimation. Let $\nu_*(R)dR$ be the true (unobservable and unknown, neither to the econometrician nor to economic agents) distribution of returns, and $\nu_T(R)dR$ the agent's posterior distribution given the observations $t\le T$. We also denote by $I_{T+1}$ the agent's IMRS (e.g., the ratio of marginal utilities in a \cite{lucas1978asset} model). Then, the agent's inter-temporal optimality condition gives the {\it subjective} pricing equation
\begin{equation}
\int \underbrace{\nu_T(R_{T+1})}_{subjective\ probabilities}\,\underbrace{E[I_{T+1}|R_{T+1}]}_{IMRS}\,R_{T+1}dR_{T+1}\ =\ 0\,. 
\end{equation}
To get the {\it objective} pricing equation, we need to define the SDF under the {\it objective probability distribution}: 
\begin{equation}\label{decomp}
\widetilde{\cM}_{T+1}\ =\ \underbrace{\ell_T(R_{T+1})}_{objective\ likelihood}\,\underbrace{E[I_{T+1}|R_{T+1}]}_{IMRS}\,,\ \qquad\ \ell_T(R_{T+1})\ =\ \underbrace{\frac{\nu_T(R_{T+1})}{\nu_*(R_{T+1})}}_{unobservable}\,. 
\end{equation}
Mechanically, $\widetilde{\cM}_{T+1}$ satisfies \eqref{objective}. 
However, when $P$ is large enough, $\nu_*(R_{T+1})$ can be extremely difficult to estimate. As a result, even with a fully specified and calibrated equilibrium model (for example, \cite{cogley2008market, johannes2016learning, collin2016parameter}), the true SDF $\widetilde{\cM}_{T+1}$ exists but remains inaccessible to the econometrician. Substituting \eqref{decomp} into \eqref{hjb-m}, we arrive at the following result. 

\begin{proposition}\label{gap}  Consider an equilibrium asset pricing model with an agent who needs to learn the true parameters of the model based on past $T$ observations and ends up with a posterior $\nu_T(R_{T+1})$. Then, this agent's SDF $\widetilde{\cM}_{T+1}$ is given by \eqref{decomp} and, hence,  
\begin{equation}\label{hjb-m-learning}
\frac{\bigr(\Var[\ell_T(R_{T+1})E[I_{T+1}|R_{T+1}]]\bigr)^{1/2}}{E[\ell_T(R_{T+1})E[I_{T+1}|R_{T+1}]]}\ \ge\ \lim\sup \frac{R^2_{OOS}(\hat f)+\widehat\cL}{1-R_{OOS}^2(\hat f)}
\end{equation}
\end{proposition}
Proposition \ref{gap} implies that the bound \eqref{hjb-m} is not just about ``macro shocks being not volatile enough," it is also about ``the objective likelihood being too volatile." In the language of \cite{chen2024measuring}, this unobservable objective likelihood constitutes a form of ``dark matter", explaining the gap between the volatility of the true SDF and the volatility of IMRS. LLG emphasizes that economic agents face exactly the same limits-to-learning as academic econometricians and, as such,  allows us to establish a lower bound for the size of this dark matter. For example, it can be used to measure the presence (and size) of fat tails in $\nu_T(R_{T+1}),$ which naturally lead to large $\Var[\ell_T(R_{T+1})E[I_{T+1}|R_{T+1}]],$ as well as test 
(rational or behavioral) belief formation models such as those in \cite{farmer2024learning}. 

\subsection{Parameter Learning And Equilibrium Excess Volatility} 

One of the most important empirical regularities in financial markets is excess volatility \citep{shiller1981excess}: The fact that prices move too much to be justifiable by the volatility of fundamentals. The common explanation proposed in the literature relies either on irrationality (over-reaction, \cite{deBondtThaler1985}) or on complex preferences \citep{bansalYaron2004}. In this section, we argue that excess volatility can emerge in a simple equilibrium model with {\it risk-neutral} agents due to limits-to-learning. 

Consider a simple model where stocks live for one period and pay dividends 
\begin{equation}
y_{t+1}\ =\ \beta'S_t+\eps_{t+1}\,.
\end{equation}
We assume that agents do not know the true value of $\beta$ and have a rational, Gaussian prior about it, $\beta\sim N(0, \dfrac{\gs_\beta^2}{P} I_{P\times P})\,.$ In this case, standard calculations (see Appendix \ref{bayes}) imply that the agents' posterior mean of $\beta$ after observing $(S,y)$ is given by $\hat\beta_T(cz),$ where $z= \dfrac{\gs_\eps^2}{\gs_\beta^2}$ and 
\begin{equation}
\hat\beta_T(cz)\ =\ (cz I+\hat\Psi)^{-1}S'y\,.
\end{equation}
If all agents are risk neutral and share this common prior, prices are given by 
\begin{equation}
Q_T\ =\ E_T[y_{T+1}]\ =\ \hat\beta_T(cz)'S_T\,. 
\end{equation}
We can now characterize price variance and its relationship to fundamental variance. 

\begin{proposition}\label{prop:excess} Suppose that $E[S]=0,\ E[SS']=\Psi$ and that $S_t$ are i.i.d. across $t.$ We have 
\begin{equation}
\begin{aligned}
&\Var[y_{T+1}]\ =\ \beta'\Psi\beta\ +\ \gs_\eps^2\\
&\Var_T[Q_{T+1}]\ =\ \hat\beta_T(cz)'\Psi \hat\beta_T(cz)
\end{aligned}
\end{equation}
In the limit as $T,P\to\infty,\ \dfrac{P}{T}\to c,$ we have 
\begin{equation}
\Var_T[Q_{T+1}]\ \ge\ \gs_\eps^2\widehat\cL,
\end{equation}
where $\widehat\cL=\dfrac1T\tr \Bigr(\Psi\hat\Psi(zcI+\hat\Psi)^{-2}\Bigr)\,$ is the LLG from \eqref{old-llg}, computed with the true $\Psi$. 
\end{proposition}

Proposition \ref{prop:excess} offers an elegant economic interpretation of LLG: It is the excess volatility generated due to the rational over-reaction of economic agents to noise while learning from high-dimensional data. When $P>T,$ $\hat\Psi$ has at least $P-T$ zero eigenvalues and, as a result, $\widehat\cL$ typically explodes when $z$ is small enough. By Proposition \ref{prop:excess}, this always leads to excess volatility. This offers a novel potential solution to the \cite{shiller1981excess} puzzle: Prices move too much because economic agents face complexity and limits-to-learning. 

Computing $\widehat\cL$ in Proposition \ref{prop:excess} formally requires the knowledge of the unobservable, true covariance matrix $\Psi$. However, Random Matrix Theory (RMT) can be used to compute it under more stringent conditions on $S_t.$ 

\begin{proposition}\label{prop:theoldllg} Suppose that $S_t=\Psi^{1/2}X_t$ where $X_t=(X_{i,t})$ has i.i.d. coordinates $X_{i,t}$ with $E[X_{i,t}]=0,\ E[X_{i,t}^2]=1.$ Then, $\widehat\cL-\widetilde\cL$ converges to zero almost surely, where 
\begin{equation}\label{final-form}
\widetilde\cL\ =\ \frac{\dfrac1T\tr \bigl((zcI+SS'/T)^{-2} \bigr)}{\Bigr(\dfrac1T\tr \bigl((zcI+SS'/T)^{-1} \bigr)\Bigr)^2}-1\,. 
\end{equation}
\end{proposition}
Thus, $\widetilde\cL$ is the Herfindahl index of the eigenvalues of the matrix $(zcI+\dfrac{SS'}{T})^{-1}\in \R^{T\times T}.$ How can this Herfindahl index matter if the law of large numbers (LLN) implies that $\dfrac{SS'}{T}\to 0$?  The reason behind this is statistical complexity. While the off-diagonal elements of $\dfrac{SS'}{T}$ (i.e., cross-time signal covariances) are indeed small and deviate only slightly from zero, they aggregate a large number $P$ of signals: For $t_1\not=t_2,$ we have  $\dfrac1TS_{t_1}'S_{t_2}=\dfrac1T\sum_{i=1}^P S_{i,t_1}S_{i,t_2}=O(\dfrac1TP^{1/2})$ by the central limit theorem. Thus, while in the classical regime these elements decay like $\dfrac1T$, in the complex regime, when $c= \dfrac{P}{T}>0,$ they decay like $T^{-1/2}.$  As a result, an econometrician observes systematic deviations from LLN and the eigenvalues of $\dfrac{SS'}{T}$ converge to non-zero limits when the statistical complexity $c=\dfrac{P}{T}>0.$ It is this form of spurious signal correlations across time periods that generates the LLG.

\section{Empirics} 

In this section, we provide extensive empirical and simulated evidence illustrating the power of Theorems \ref{llg1} and \ref{main-th-main-text} to recover (a lower bound for) the true amount of predictability. As such, LLG can be used to extract important dynamic information about the economy. 

While our analysis applies to any linear-in-$y$ estimator, including nonlinear ones (Proposition \ref{prop:yang}), in this section we focus on ridge regression with random nonlinear features as the prototypical application of our theory. We proceed as follows: 

\begin{itemize}
\item Given the data, $X_t,$ we follow \cite{kelly2024virtue, kelly2025understanding} and generate $P=20000$ random features 
\begin{equation}\label{def:random-feat}
S_{k,t}\ =\ g(W_k'X_{t}),\ g\in \{\text{tanh}, \text{ReLu}\}
\end{equation}
from original data $X_t\in \R^d$, where $W_k$ are sampled i.i.d.~from $\cN(0,I_{d\times d}).$ The non-linearity $g(x)$ is commonly referred to as the {\it activation function.} We study both $\tanh(x)$ and $\text{ReLu}(x)=\max(x,0)$ because of their distinct features: $\tanh(x)$ flattens out for large $x$ and, hence, is less sensitive to outliers and tends to focus on the bulk of the distribution. By contrast, $\text{ReLu}(x)$ grows indefinitely at $\infty$ and is therefore able to capture tail dependencies better.  

In our analysis, we always report the effect of statistical complexity on model performance by varying $P_1=100,\cdots,20000$ and running the ridge regression using the first $P_1$ of the random features \eqref{def:random-feat}. Following \cite{kelly2024virtue}, we refer to the curve showing model performance as a function of $c=P_1/T$ as a virtue of complexity (VoC) curve. 

\item Given the signals $S$ and labels, $y,$ we compute $\hat\beta(z)$ from \eqref{4}. Since $\widehat\cL(z)$ in \eqref{old-llg} is monotone decreasing in $z,$ we pick a relatively small $z$ that ``scales" with the size of the signals. We set 
\begin{equation}\label{zref}
z\ =\ z_{ref}\,\dfrac{1}{P}\tr(S'S),\ with\ z_{ref}\ =\ 0.01\,.
\end{equation}

\item We use Theorem \ref{main-th-main-text} to compute $R^2_{OOS},$ its lower bound \eqref{llg3-first}, as well as the one-sided 95\% asymptotic confidence band for $R^2_*$, 
\begin{equation}\label{low-conf}
\left[\frac{R^2_{OOS}(\hat f)+\widehat\cL(z)}{1+\widehat\cL(z)}\ -\ 1.65T^{-1/2}  \hat\gs_{R^2}\,,\,1\right]\,.
\end{equation}
\end{itemize}

\subsection{Data}

We consider the classic \cite{welch2008comprehensive} dataset with 14 different time series at monthly frequency covering the period from 1930-01-01 to 2020-11-01:  
\begin{itemize}
\item {\bf Group one: Equity Valuation and Market.} Excess returns (\texttt{retx}) are the excess returns on the {CRSP value-weighted index}. The Dividend Price Ratio (\texttt{dp}) is the difference between the log of dividends and the log of prices. The Dividend Yield (\texttt{dy}) is the difference between the log of dividends and the log of lagged prices.
\footnote{We exclude \texttt{dy} from our set of targets but use it in the set of signals because $dy_{t+1}=\log(d_{t+1}/p_t)$ mechanically co-moves with other time$-t$ variables involving $p_t$ and \texttt{retx.}}  
The Earnings Price Ratio (\texttt{ep}) is the difference between the log of earnings and the log of prices.\footnote{Earnings are 12-month moving sums of earnings on the S\&P 500 index.} The Dividend Payout Ratio (\texttt{de}) is the difference between the log of dividends and the log of earnings. The Book-to-Market Ratio (\texttt{bm}) is the ratio of book value to market value for the Dow Jones Industrial Average. The Net Equity Expansion (\texttt{ntis}) is the ratio of 12-month moving sums of net issues by NYSE-listed stocks divided by the total end-of-year market capitalization of NYSE stocks.

\item {\bf Group 2: Term structure, Credit, Inflation.} Treasury Bills (\texttt{tbl}) is the Treasury-bill rate. Long Term Yield (\texttt{lty}) is the U.S. Yield On Long-Term United States Bonds from the NBER’s Macrohistory database. Long Term Rate of Returns (\texttt{ltr}) is from the same source. The Term Spread (\texttt{tms}) is the difference between the long-term yield on government bonds and the Treasury bill. The Default Yield Spread (\texttt{dfy}) is the difference between BAA and AAA-rated corporate bond yields. The Default Return Spread (\texttt{dfr}) is the difference between long-term corporate bond and long-term government bond returns. Inflation (\texttt{infl}) is the Consumer Price Index (All Urban Consumers) from the Bureau of Labor Statistics. 
\end{itemize}

Our data sample spans almost 90 years, and non-stationarity considerations become essential. Furthermore, the variables from \cite{welch2008comprehensive} may not satisfy the technical conditions of Theorem \ref{llg1}, requiring that the residuals $\eps$ are uncorrelated and homoskedastic. To deal with these considerations, we pre-process the \cite{welch2008comprehensive} data using the following procedure. 

\begin{procedure}[Construction of the Processed Series]\label{proc:construction}
\begin{itemize}
\item Demean and standardize each $X_{i,t}$ using its rolling 36-month mean and standard deviation, lagged by one month;\footnote{In Python syntax, we perform the transformation
\[
X\to (X-X.rolling(36).mean().shift(1))/X.rolling(36).std().shift(1).
\]
}
\item Clip it at $[-3,3]$;

\item Compute the AR(1) residuals of the resulting normalized variable using an expanding window to estimate the autocorrelation coefficient. 
\end{itemize}
\end{procedure}

We use $X_t$ to denote the \cite{welch2008comprehensive} data (including the excess returns) processed according to Procedure \ref{proc:construction}.  By construction, these data have zero autocorrelation and are approximately homoskedastic.\footnote{We have performed extensive simulations with $y_t$ exhibiting diverse forms of autocorrelation and stochastic volatility. These simulations indicate that Theorems \ref{llg1} and \ref{main-th-main-text} continue to hold as long as the above-listed transformations are applied to the underlying predicted variable.}

\subsection{Semi-Synthetic Simulations}

We start with a semi-synthetic simulation where we use the 13 \cite{welch2008comprehensive} variables and excess returns as $X_t,$ and simulate 
\begin{equation}\label{semi-synth}
y_{t+1}^{synthetic}\ =\ f_t\ +\ \eps_{t+1},\ f_t\ =\ \gamma \, g(X_t'W)\,,
\end{equation}
where $\varepsilon_{t+1} \sim \cN(0,1)$ and $W \in \mathbb{R}^d$ is drawn from $\cN(0,I)$, and $g\in \{\text{tanh},\ \text{ReLu}\}.$ We also study a pure-noise benchmark, corresponding to the limiting semi-synthetic case with $\gamma = 0$, in which $y_{t+1}$ follows a GARCH(1,1) process. Specifically, we generate a zero-mean GARCH(1,1) sequence $\{y^{GARCH(1,1)}_t\}_{t=1}^T$ following the procedure outlined in Appendix \ref{app:garch}. Since, formally, our theory does not apply to GARCH residuals, we test the efficiency of Procedure \ref{proc:construction} and also study $y^{GARCH(1,1)}_{t,standardized}$ obtained from $y^{GARCH(1,1)}_t$ using the first two steps of Procedure \ref{proc:construction}. 

The convenience of a semi-synthetic simulation is that we can estimate the infeasible $R^2_*,$
\begin{equation}\label{semi-r2}
R^2_*\ \approx\ \frac{E[f_t^2]}{E[y_{t+1}^2]}
\end{equation}
and, hence, we can test our theory without sacrificing the highly complex nature of the real data. Varying $\gamma$ in \eqref{semi-synth} allows us to control $R^2_*$ and, thus, test the efficiency of our lower bound for various degrees of predictability. 

Our in-sample/out-of-sample split is set at January 1990. This period corresponds to the early phase of large-scale market electrification—characterized by the adoption of electronic trading platforms, faster information dissemination, and increased automation—which represents a structural break in how financial markets process information.

As is explained above, we construct $P=20000$ random features $S_t$ using \eqref{def:random-feat}, run regression with the first $P_1$ features, and report $R^2_{OOS}$ as well as the lower bounds \eqref{llg3-first} and \eqref{low-conf} as a function of complexity $c=P_1/T,$ where $T$ is the number of in-sample observations. 

Figures \ref{fig:semi1}-\ref{fig:semi2} clearly demonstrate that our theory holds very well: First, although the lower bound \eqref{llg3-first} does cross the $R^2_*$ sometimes (e.g., for $R^2_*=0$), the lower confidence bound \eqref{low-conf} is always below $R^2_*$, emphasizing the importance of the Central Limit Theorem correction to LLG.  Second, the realized $R^2_{OOS}$ decays very quickly with complexity $c.$ An econometrician might interpret this as evidence against statistical complexity. However, once we correct for LLG, \eqref{llg3-first} increases with $c$ (or stays flat) for a majority of plots, emphasizing a novel form of Virtue of Complexity: The decay in $R^2_{OOS}$ is caused by accumulating estimation errors with growing $P$ (models with more parameters are more difficult to estimate). 

Figures \ref{fig:semi1}-\ref{fig:semi2} clearly show that approximating the ground truth requires complexity, and running a model with a larger $P$ makes the lower bound \eqref{llg3-first} converge to the infeasible $R^2_*.$ Third, although GARCH residuals violate the hypothesis of Theorems \ref{llg1} and \ref{main-th-main-text}, Figures \ref{fig:semi1}-\ref{fig:semi2} show the theory holds both before and after the data is standardized using Procedure \ref{proc:construction}. 

\subsection{Real Data}

We now use LLG to investigate the predictability of the \cite{welch2008comprehensive} variables transformed according to Procedure \ref{proc:construction}. For each $i=1,\cdots,13,$ we set $y_{i,t+1}=X_{i,t+1}$ and use $X_t$ as the signals. We then feed $X_t$ into the $P=20000$ random features in \eqref{def:random-feat}. Figures \ref{fig:g1-tanh}–\ref{fig:g2-relu} summarize the results. Several patterns stand out clearly. First, $R^2_*$ rises sharply for many target variables, displaying a strong virtue of complexity. Second, the LLG implied by $\text{ReLu}$ features differs substantially from that implied by $\text{tanh}$ features, illustrating how slightly different linear models can yield significantly different lower bounds. 

For group 1 (Figures \ref{fig:g1-tanh} and \ref{fig:g1-relu}), the largest gains arise for the dividend-to-price \texttt{dp}, earnings-to-price \texttt{ep}, and dividends-to-earnings \texttt{de} ratios, whose lower bounds reach roughly 50\% (and up to 70\% for \texttt{de} under the $tanh$ activation). We also find substantial gains for stock volatility \texttt{svar} and the book-to-market ratio \texttt{bm}, where the bounds reach 20–25\%. Because we are predicting AR(1) residuals (Procedure \ref{proc:construction}), identifying a model capable of achieving such a high population $R^2$ would yield meaningful economic gains for investors and provide deeper insights into the underlying macroeconomic dynamics.

Surprisingly, the $\text{ReLu}$-based LLG implies a lower bound of about 20\% for predicting U.S. market excess returns—around 10–20 times higher than a typical realized $R^2_{OOS}$ at a monthly horizon documented, e.g., in \cite{welch2008comprehensive, kelly2024virtue, kelly2025understanding}. This finding casts the common belief that market returns are essentially unpredictable in a new light: the LLG-based bound suggests that returns may be highly predictable, but that the underlying nonlinear function is extremely hard to learn due to limits-to-learning. This interpretation aligns with \cite{kelly2024virtue, kelly2025understanding}, who argue that the true predictive relationship between returns and the \cite{welch2008comprehensive} variables is highly complex.

For group 2 (Figures \ref{fig:g2-tanh} and \ref{fig:g2-relu}), the effects are smaller but still economically meaningful: our theory implies an $R^2_*$ of at least 10\% for T-bills \texttt{tbl} and the long-term rate \texttt{ltr}; and at least 25\% for the default yield spread \texttt{dfy} and the long-term yield \texttt{lty}.

In every target except \texttt{de},\footnote{\texttt{de} is special because it is a ``purely fundamental'' variable that does not involve market prices or returns.} the strong predictability implied by the LLG-based lower bound coexists with a negative realized $R^2_{OOS}.$ These findings match the semi-synthetic simulations in Figures \ref{fig:semi1}–\ref{fig:semi2}, where a negative $R^2_{OOS}$ often coexists with significant underlying true (population) predictability. Thus, although the simple random-feature ridge model fails to uncover predictability in Figures \ref{fig:g1-tanh}–\ref{fig:g2-relu}, some other model could achieve a realized $R^2_{OOS}$ close to (or even above) the LLG-implied lower bound.

Identifying such a model directly may require nontrivial effort. Here, we consider two classes of related models: Ridge regression with a different ridge penalty and a model we refer to as ``recursive ridge," which uses a simple feature selection and construction algorithm similar to that of \cite{yan2017fundamental, chen2023high, li2025machine} before running the ridge regression. We describe this algorithm in Appendix \ref{app:rec-ridge}. 

We begin by noting that, under Theorem \ref{llg1}, a ridge regression \eqref{4} estimated with a different penalty or a different feature set (e.g., activation $\text{tanh}$ versus $\text{ReLu}$) constitutes a distinct ML model. There is no theoretical result tying the out-of-sample performance of such models to the performance of the reference model used to construct the LLG. For example, Figures \ref{fig:g1-tanh}–\ref{fig:g2-relu} rely on $z_{ref}=0.01$ in \eqref{zref}, and there is no ex-ante basis for expecting the corresponding LLG-based lower bound to be informative about the $R^2_{OOS}$ of ridge regressions estimated with other penalties, even when the set of features is held fixed. This disconnect is even sharper for the recursive ridge model in Appendix \ref{app:rec-ridge}, which uses a different feature set and is nonlinear due to its variable-selection step. The distinction matters: the construction of the LLG relies on linearity, but the bound itself applies to any forecasting model.

For each variable $y_{i,t+1}$, we estimate ridge regressions with random features \eqref{def:random-feat} using a grid of $z_{ref}\in{0.01,0.1,1.,10.}$, and estimate recursive ridge models using the same grid. All specifications are trained on the 1933-01 to 1989-12 sample and evaluated out of sample from 1990-01 to 2024-12.

Our analysis focuses on two questions. First, which models deliver economically meaningful out-of-sample performance? Second, is the realized $R^2_{OOS}$ systematically related to the LLG-based lower bound? To assess this, we report, for each model class—ridge with two types of random features \eqref{def:random-feat} on $P=100,\ldots,20000$, and recursive ridge—the highest $R^2_{OOS}$ attained within that class. Although these best-in-class values inevitably reflect model selection, they serve as a benchmark for the predictive content extractable by flexible (nonlinear) ML methods.

Table \ref{tab:lb_ridge_and_recursive} reports the lower bounds \eqref{low-conf}, maximized over $P=100,\ldots,20000$, for $\text{tanh}$ and $\text{ReLu}$ activations,\footnote{The two LLGs should be interpreted as distinct signals. A combined lower bound incorporating their covariance is feasible but outside the scope of this analysis.} along with the corresponding best $R^2_{OOS}$ for each model class.

The results indicate substantial out-of-sample predictability. Both ridge and recursive ridge models forecasting \texttt{dfy}, \texttt{de}, \texttt{tbl}, and \texttt{ep} achieve sizable $R^2_{OOS}$, broadly consistent with the LLG-based bounds. For instance, for \texttt{dfy} and \texttt{ep}, the recursive ridge yields $R^2_{OOS}\approx 15\%$, compared with an LLG bound of roughly 25\% from the $\text{tanh}$ features. For \texttt{tbl}, recursive ridge achieves $R^2_{OOS}\approx 10\%$, matching the 10\% bound. For \texttt{de}, the recursive ridge attains $R^2_{OOS}\approx 42\%$ versus a 67\% bound.

The cases of \texttt{ep}, \texttt{tbl}, \texttt{tms}, and \texttt{dfr} are especially informative: the nonlinear recursive ridge consistently outperforms the linear ridge by a large margin and, in some instances, approaches the LLG-based benchmark. This pattern highlights a central implication of our theoretical results: {\it although LLG is built from a linear reference model, it captures information relevant for the performance frontier of nonlinear models.}

In contrast, for \texttt{retx}, \texttt{dp}, \texttt{bm}, \texttt{svar}, and \texttt{lty}, even the recursive ridge falls far short of the lower bound. This may indicate that other functional forms are required to approximate the true predictive structure, or that sampling variation imposes a hard limit on feasible predictive accuracy.

Table \ref{tab:correlation_lb_dnn} documents a strong positive association between the LLG-based lower bounds and the realized out-of-sample performance of both classes of models. The $\text{tanh}$-based bound performs particularly well, exhibiting an 80\% correlation with the best nonlinear $R^2_{OOS}$. This is noteworthy given that the bound is computed from a ridge model with a small penalty $z_{ref}=0.01$ \eqref{zref} that itself performs poorly out of sample. Yet, Theorem \ref{llg1} successfully extracts information about the underlying degree of predictability.

Overall, the evidence suggests that LLG provides a powerful, model-agnostic diagnostic for detecting genuine structure in predictive regressions, even in long samples marked by structural breaks. It offers researchers a principled tool for identifying promising forecasting relationships and guiding subsequent theoretical and empirical inquiry.

\section{Conclusion}

This paper demonstrates that the apparent weakness of predictability in financial data is often a statistical illusion. In high-dimensional environments, even state-of-the-art ML estimators face intrinsic limits: finite samples prevent them from recovering the true signal, causing conventional out-of-sample metrics to systematically understate population predictability. We formalize this observation by deriving the Limits-to-Learning Gap (LLG)—a universal, data-driven lower bound on the discrepancy between empirical and population fit. The LLG provides a sharp, model-agnostic correction to measured $R^2$, yields confidence bounds for population predictability, and quantitatively identifies when weak OOS performance reflects a lack of signal versus the inability of an estimator to learn it.

Empirically, we find that LLGs are large across a broad set of classical financial predictors, implying that true $R^2$ values are often many times higher than their raw estimates. These results overturn the conventional view that return predictability is negligible. They further show that nonlinear or more expressive models can, in several cases, recover economically meaningful predictive structure precisely when the LLG indicates that such structure must exist.

Beyond empirical applications, the LLG has implications for asset-pricing theory. By refining the \citet{hansen1991implications} bounds, our framework quantifies the role of parameter uncertainty in shaping equilibrium outcomes. In a simple general-equilibrium model, we show that high LLGs naturally generate excess volatility and slow learning, offering a rational benchmark for phenomena often attributed to behavioral biases or misperceptions. More broadly, the LLG sheds new light on why ML scaling laws exhibit sharply diminishing returns: when the limits-to-learning gap closes only slowly, additional data or computation cannot eliminate residual estimation error.

Taken together, these findings suggest a reorientation of empirical practice. Instead of evaluating predictability solely through realized OOS performance, researchers should assess whether the observed performance is informative about population fit at all. The LLG offers precisely this diagnostic. It clarifies when predictability is truly absent, when existing ML architectures are inadequate, and when economic signals are fundamentally obscured by finite-sample noise. As economic data continue to grow in dimensionality and complexity, incorporating limits-to-learning into empirical design, model evaluation, and asset-pricing theory will be essential. The framework introduced here provides a tractable and broadly applicable foundation for doing so.

\clearpage 

\begin{figure}
\centering
\includegraphics[width=1\linewidth]{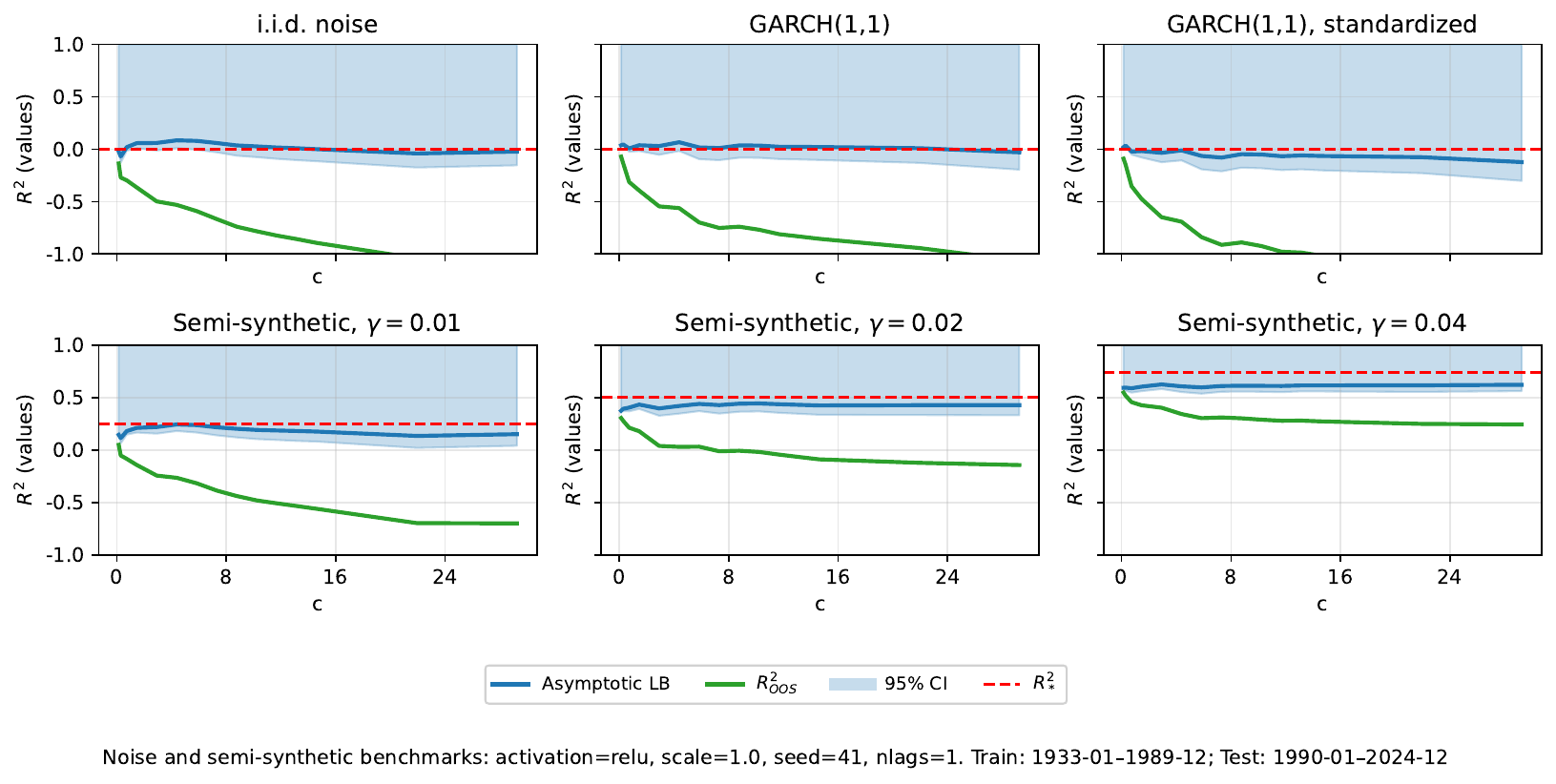}
\caption{Semi-synthetic simulation \eqref{semi-synth}, with activation=$\text{ReLu}$. In-sample period is 1933-01 to 1989-12; OOS period is 1990-01 to 2024-12. i.i.d. noise has $y_{t+1}=\eps_{t+1}\sim \cN(0,1)$. GARCH(1,1) has $y_{t+1}=\eps_{t+1}$ being a GARCH(1,1) noise defined in Appendix \ref{app:garch}. Asymptotic Lower Bound is given by \eqref{llg3-first}. The shaded region is the one-sided confidence band for $R^2_*.$ The lower bound of the shaded region is \eqref{low-conf}. Horizontal axis is statistical complexity $c=P_1/T,$ where $P_1$ is the number of random features \eqref{def:random-feat}, increasing from $P_1=100$ to $P_1=20000.$ 
$R^2_*$ is computed in \eqref{semi-r2}. Values of $R^2_{OOS}<-1$ are not shown. $\gamma$ values are selected to achieve $R^2_*$ of $0,\ 0.25,\ 0.5,\ 0.75$, respectively.}
\label{fig:semi1}
\end{figure}

\begin{figure}
\centering
\includegraphics[width=1\linewidth]{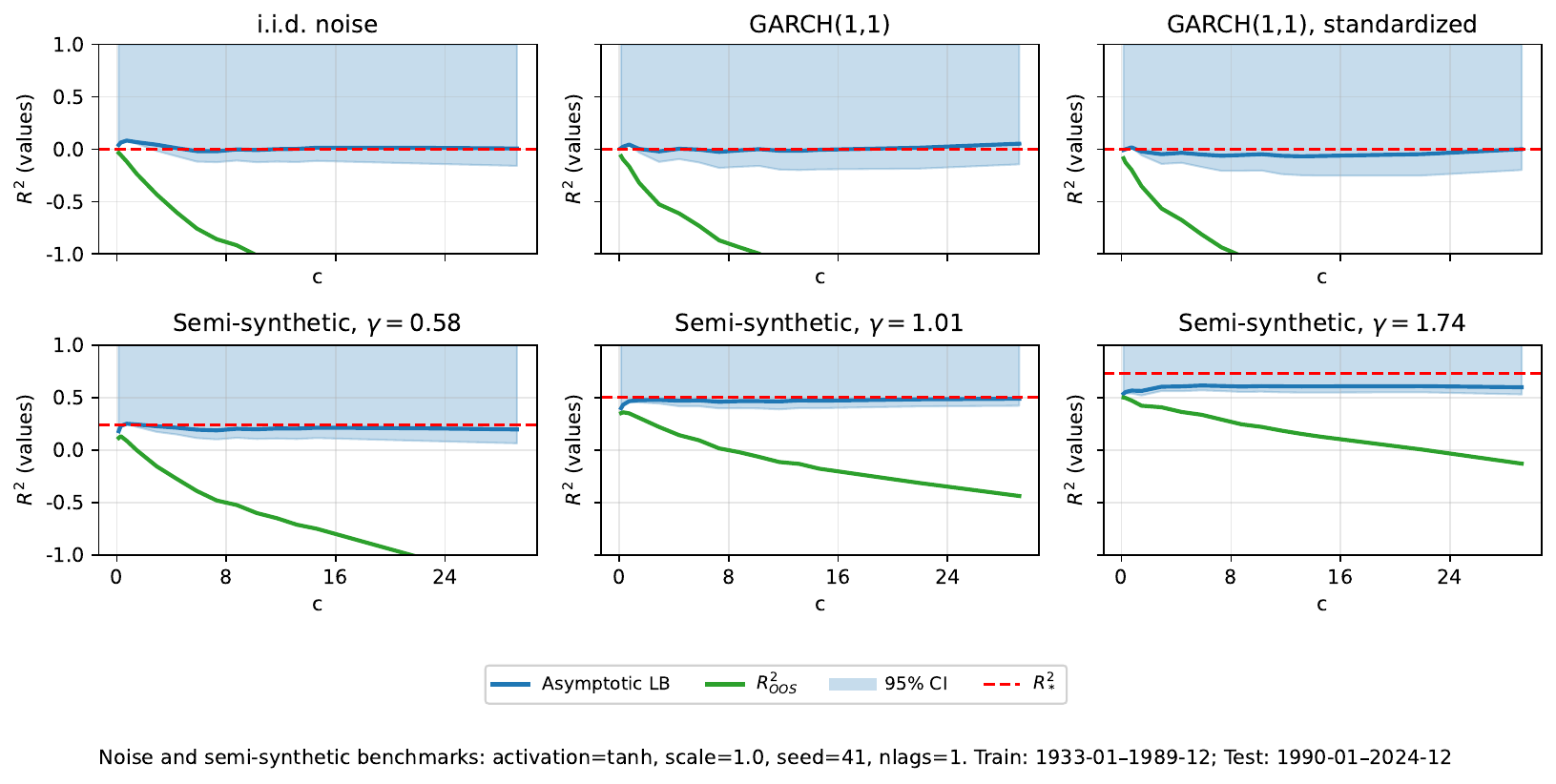}
\caption{Semi-synthetic simulation \eqref{semi-synth}, with activation=$\text{tanh}$. In-sample period is 1933-01 to 1989-12; OOS period is 1990-01 to 2024-12. i.i.d. noise has $y_{t+1}=\eps_{t+1}\sim \cN(0,1)$. GARCH(1,1) has $y_{t+1}=\eps_{t+1}$ being a GARCH(1,1) noise defined in Appendix \ref{app:garch}. Asymptotic Lower Bound is given by \eqref{llg3-first}. The shaded region is the one-sided confidence band for $R^2_*.$ The lower bound of the shaded region is \eqref{low-conf}. Horizontal axis is statistical complexity $c=P_1/T,$ where $P_1$ is the number of random features \eqref{def:random-feat}, increasing from $P_1=100$ to $P_1=20000.$ 
$R^2_*$ is computed in \eqref{semi-r2}. Values of $R^2_{OOS}<-1$ are not shown. $\gamma$ values are selected to achieve $R^2_*$ of $0,\ 0.25,\ 0.5,\ 0.75$, respectively.}
\label{fig:semi2}
\end{figure}

\begin{figure}
\centering
\includegraphics[width=1.\linewidth]{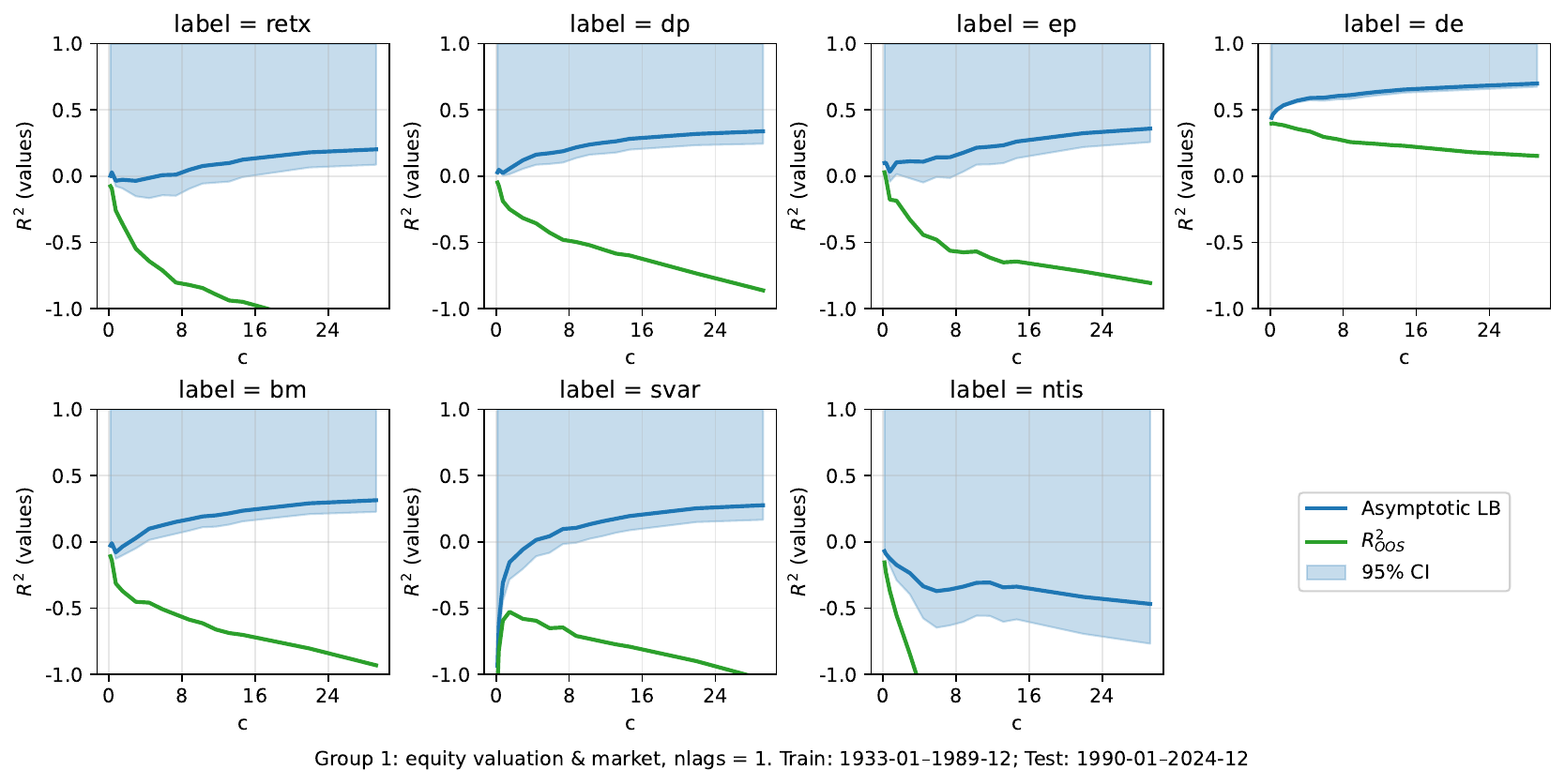}
\caption{Predicting \cite{welch2008comprehensive} variables from {\bf Group one} processed according to Procedure \ref{proc:construction}. Signals are the \cite{welch2008comprehensive} 14 variables and excess returns.  In-sample period is 1933-01 to 1989-12; OOS period is 1990-01 to 2024-12. Asymptotic Lower Bound is given by \eqref{llg3-first}. The shaded region is the one-sided confidence band for $R^2_*.$ The lower bound of the shaded region is \eqref{low-conf}. Horizontal axis is statistical complexity $c=P_1/T,$ where $P_1$ is the number of random features \eqref{def:random-feat} with {\bf activation=}$\text{tanh}$, increasing from $P_1=100$ to $P_1=20000.$ Values of $R^2_{OOS}<-1$ are not shown.}
\label{fig:g1-tanh}
\end{figure}

\begin{figure}
\centering
\includegraphics[width=1.\linewidth]{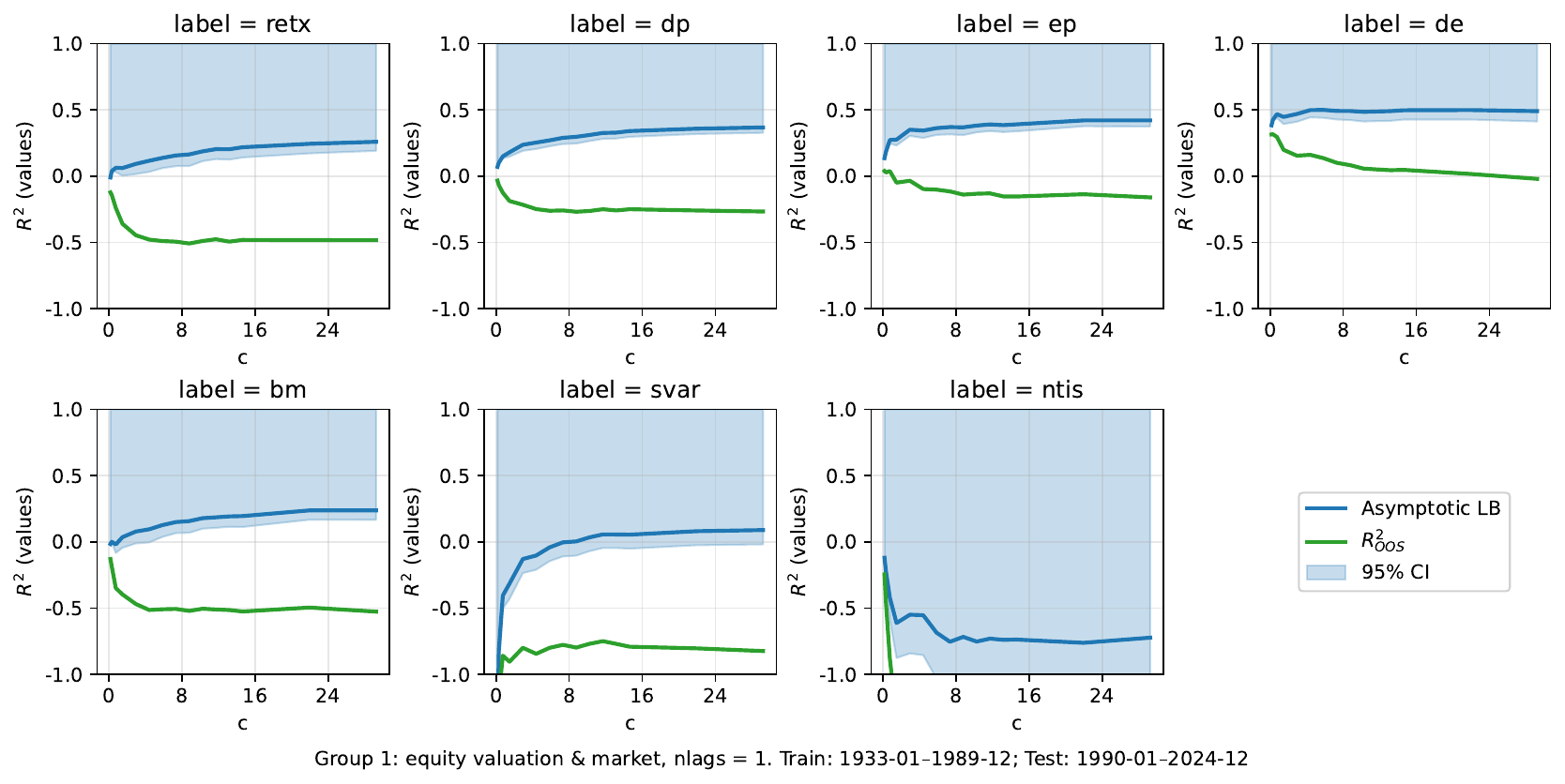}
\caption{Predicting \cite{welch2008comprehensive} variables from {\bf Group one} processed according to Procedure \ref{proc:construction}.  Signals are the \cite{welch2008comprehensive} 14 variables and excess returns.  In-sample period is 1933-01 to 1989-12; OOS period is 1990-01 to 2024-12. Asymptotic Lower Bound is given by \eqref{llg3-first}. The shaded region is the one-sided confidence band for $R^2_*.$ The lower bound of the shaded region is \eqref{low-conf}. Horizontal axis is statistical complexity $c=P_1/T,$ where $P_1$ is the number of random features \eqref{def:random-feat} with {\bf activation=}$\text{ReLu}$, increasing from $P_1=100$ to $P_1=20000.$ Values of $R^2_{OOS}<-1$ are not shown.}
\label{fig:g1-relu}
\end{figure}

\begin{figure}
\centering
\includegraphics[width=1.\linewidth]{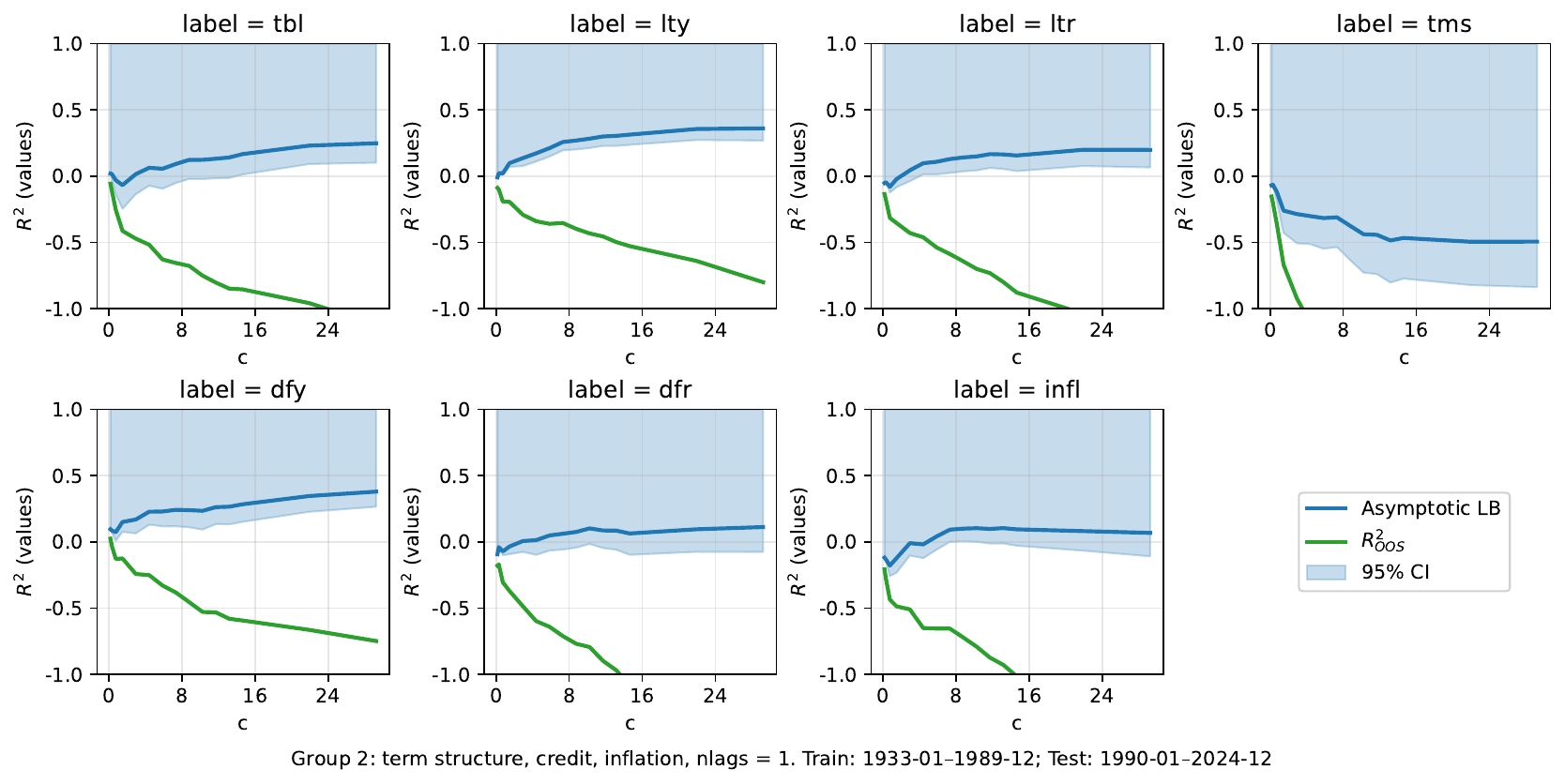}
\caption{Predicting \cite{welch2008comprehensive} variables from {\bf Group two} processed according to Procedure \ref{proc:construction}.  Signals are the \cite{welch2008comprehensive} 14 variables and excess returns.  In-sample period is 1933-01 to 1989-12; OOS period is 1990-01 to 2024-12. Asymptotic Lower Bound is given by \eqref{llg3-first}. The shaded region is the one-sided confidence band for $R^2_*.$ The lower bound of the shaded region is \eqref{low-conf}. Horizontal axis is statistical complexity $c=P_1/T,$ where $P_1$ is the number of random features \eqref{def:random-feat} with {\bf activation=}$\text{tanh}$, increasing from $P_1=100$ to $P_1=20000.$ Values of $R^2_{OOS}<-1$ are not shown.}
\label{fig:g2-tanh}
\end{figure}

\begin{figure}
\centering
\includegraphics[width=1.\linewidth]{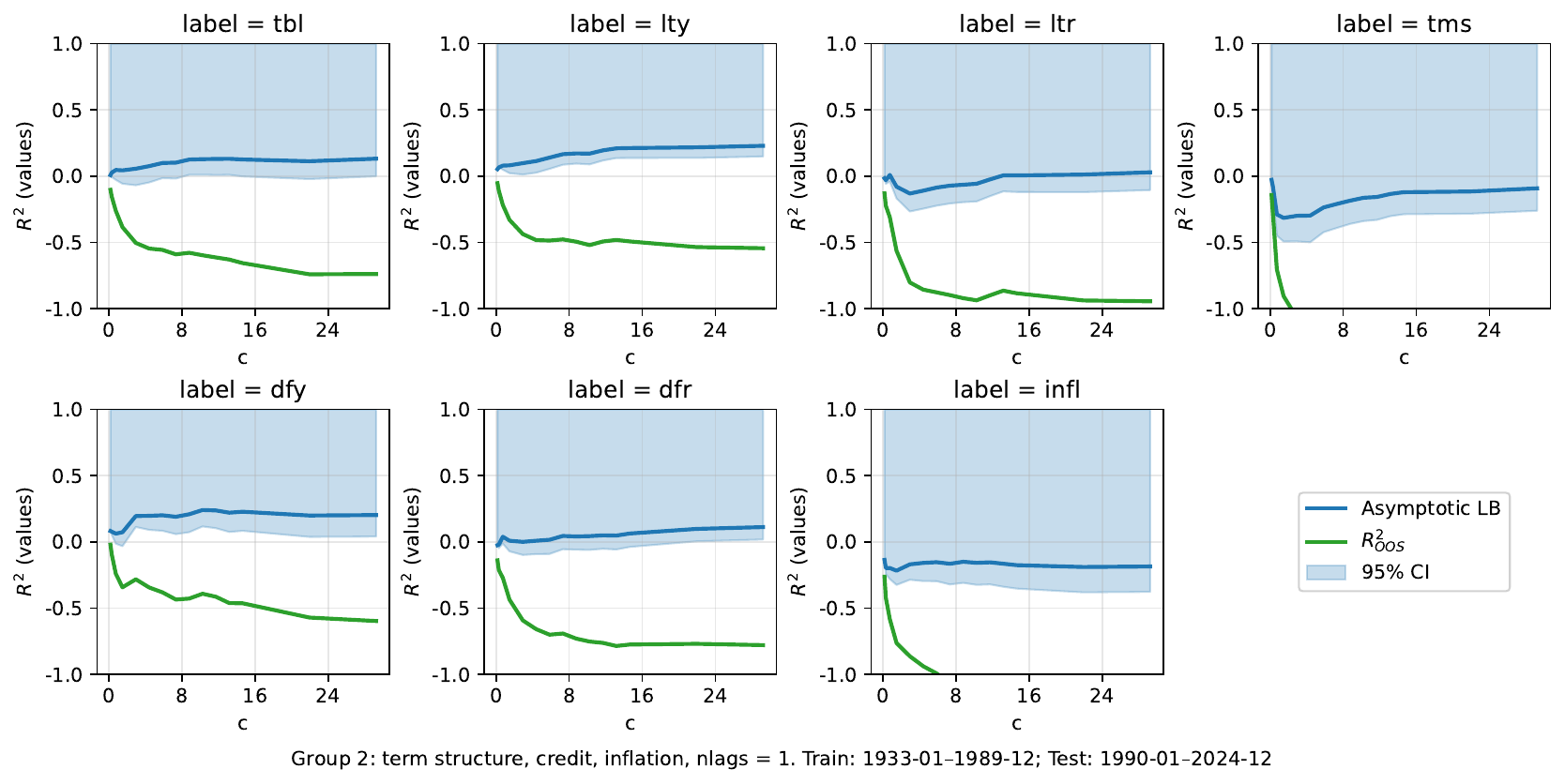}
\caption{Predicting \cite{welch2008comprehensive} variables from {\bf Group two} processed according to Procedure \ref{proc:construction}.  Signals are the \cite{welch2008comprehensive} 14 variables and excess returns.  In-sample period is 1933-01 to 1989-12; OOS period is 1990-01 to 2024-12. Asymptotic Lower Bound is given by \eqref{llg3-first}. The shaded region is the one-sided confidence band for $R^2_*.$ The lower bound of the shaded region is \eqref{low-conf}. Horizontal axis is statistical complexity $c=P_1/T,$ where $P_1$ is the number of random features \eqref{def:random-feat} with {\bf activation=}$\text{ReLu}$, increasing from $P_1=100$ to $P_1=20000.$ Values of $R^2_{OOS}<-1$ are not shown.}
\label{fig:g2-relu}
\end{figure}

\begin{table}[htbp]
\centering
\caption{95\% probability lower bound \eqref{low-conf} vs $R^2_{OOS}$ from best benchmarks}
\label{tab:lb_ridge_and_recursive}
\begin{tabular}{lcccc}
\toprule
& {tanh} 
& {ReLu} 
& \multicolumn{2}{c}{{Best}} \\
\cmidrule(lr){2-2}\cmidrule(lr){3-3}\cmidrule(lr){4-5}
{Label} 
& 95\% prob.\ bound 
& 95\% prob.\ bound 
& Ridge
& Recursive Ridge \\
\midrule
\texttt{retx} &  9\% & 19\% &  1\% &  2\% \\
\texttt{dp}   & 24\% & 33\% &   2\% &  2\% \\
\texttt{ep}   & 26\% & 38\% &   10\% & 15\% \\
\texttt{de}   & 67\% & 41\% &  42\% & 42\% \\
\texttt{bm}   & 23\% & 17\% &  0\% &  0\% \\
\texttt{svar} & 17\% &  0\% & 0\% &  0\% \\
\texttt{ntis} &  0\% &  0\% &  0\% &  1\% \\
\texttt{tbl}  & 10\% &  0\% &   5\% &  10\% \\
\texttt{lty}  & 27\% & 15\% &   1\% &  1\% \\
\texttt{ltr}  &  7\% &  0\% &   1\% &  1\% \\
\texttt{tms}  &  0\% &  0\% &   5\% &  8\% \\
\texttt{dfy}  & 27\% &  4\% &  13\% & 14\% \\
\texttt{dfr}  &  0\% &  2\% &  0\% &  3\% \\
\texttt{infl} &  0\% &  0\% & 0\% &  0\% \\
\bottomrule
\end{tabular}
\end{table}

\begin{table}[htbp]
\centering
\caption{Correlation matrix of 95\% probability lower bound \eqref{low-conf} and $R^2_{OOS}$ from best benchmarks}
\label{tab:correlation_lb_dnn}
\begin{tabular}{llcccc}
\toprule
&  & \multicolumn{1}{c}{tanh} & \multicolumn{1}{c}{ReLu}& \multicolumn{2}{c}{Best} \\
\cmidrule(lr){3-3} \cmidrule(lr){4-4} \cmidrule(lr){5-6}
&  & 95\% prob. bound & 95\% prob. bound & Ridge & Recursive Ridge\\
\midrule
\multirow{1}{*}{tanh}
& 95\% prob.\ bound   & 1.00 &        &        &        \\
\multirow{1}{*}{ReLu}
& 95\% prob.\ bound   & 0.76 & 1.00   &        &        \\
\multirow{2}{*}{Best}
& Ridge                     & 0.67 & 0.59   & 1.00   &        \\
& Recursive Ridge           & 0.80 & 0.57   & 0.86   & 1.00   \\
\bottomrule
\end{tabular}
\end{table}

\clearpage

\bibliographystyle{aer}
\bibliography{OLS}

\appendix 

\numberwithin{equation}{section}
\numberwithin{theorem}{section}
\numberwithin{lemma}{section}

\renewcommand{\thesubsection}{\thesection.\arabic{subsection}}

\renewcommand{\theequation}{\thesection.\arabic{equation}}
\renewcommand{\thetheorem}{\thesection.\arabic{theorem}}
\renewcommand{\theproposition}{\thesection.\arabic{proposition}}
\renewcommand{\thelemma}{\thesection.\arabic{lemma}}

\newpage

\section{Preliminaries on Random Matrix Theory }
\subsection{Concentration of Quadratic Forms}

Our key assumption in the main text is that the residuals, $\eps_t,$ are i.i.d. Hence, for any matrix $A$ independent of $\eps,$ we have 
\begin{equation}
\begin{aligned}
\frac1T\eps'A \eps\ &=\ \frac1T\sum_{t_1,t_2}\eps_{t_1}\eps_{t_2} A_{t_1,t_2}\\ 
&=\ \underbrace{\frac1T\sum_{t=1}^T \eps_{t}^2A_{t,t}}_{\approx \frac1T\gs_{\eps}^2\sum_{t=1}^T A_{t,t}\ because\ E[\eps_{t}^2]=\gs_{\eps}^2} +\ \underbrace{\frac1T\sum_{t_1\not=t_2}\eps_{t_1}\eps_{t_2}A_{t_1,t_2}}_{\approx 0\ because\ E[\eps_{t_1}\eps_{t_2}]=0}\\
&\approx\ \frac1T\gs_{\eps}^2\sum_{i=1}^N A_{t,t}\\
&=\ \frac1T\gs_{\eps}^2\tr(A)\,. 
\end{aligned}
\end{equation}
Many proofs in this paper are based on the classic ``concentration of quadratic forms" lemma that makes the above argument rigorous. 

\begin{lemma}\label{lem-quad} Suppose that $u=(u_i)_{i=1}^P$ where $u_i$ are i.i.d., with $E[u_i]=0,\ E[u_i^2]=\gs^2,\ E[u_i^4]<\infty.$ Suppose also $A_P$ is a sequence of symmetric random matrices that is independent of $u$. Then, 
\begin{equation}
\begin{aligned}
E[(\frac1P u'A_P u)^2-(P^{-1}\gs^2\tr(A_P))^2]\ =\ & E[(\frac1P u'A_P u-(P^{-1}\gs^2\tr(A_P)))^2]\\
 \le\ & (E[u_i^4]-\gs^4)P^{-2}E[\tr(A_P^2)]
\end{aligned}
\end{equation}
In $E[u_i^4]=3\gs^2,$ then this identity is exact: 
\begin{equation}
\begin{aligned}
E[(\frac1P u'A_P u)^2-(P^{-1}\gs^2\tr(A_P))^2]\ =\ & E[(\frac1P u'A_P u-(P^{-1}\gs^2\tr(A_P)))^2]\\
 =\ & 2\gs^4 P^{-2}E[\tr(A_P^2)]\,.
\end{aligned}
\end{equation}
If $\lim P^{-2}E[\tr(A_P^2)]=0,$ then  
\begin{equation}
\frac1P u'A_P u\ -\ P^{-1}\gs^2\tr(A_P)\ \to\ 0
\end{equation}
in $L_2$ and, hence, in probability. 
\end{lemma}

\subsection{Spectral Concentration for Sub-Gaussian Designs}

In this section we collect a concentration result for empirical feature
covariance matrices generated by high-dimensional sub-Gaussian designs.
This lemma is used repeatedly in later proofs to control higher-order
spectral terms.

A real-valued random variable $Z$ is called sub-Gaussian if there exists a
finite constant $K>0$ such that $\mathbb{E}\bigl[e^{Z^2/K^2}\bigr] \le 2.$
We then write $\|Z\|_{\psi_2}\le K$. A random vector $x\in\mathbb{R}^P$ is called
sub-Gaussian with parameter $K$ if every one-dimensional projection is
sub-Gaussian, that is, $\sup_{u\in\mathbb{S}^{P-1}} \|\langle x,u\rangle\|_{\psi_2} \le K.$ This implies in particular that all coordinates of $x$ are sub-Gaussian with
parameter bounded by a constant multiple of $K$.

Let $T,P\in\mathbb{N}$ and let $S\in\mathbb{R}^{T\times P}$ denote a random matrix whose rows $s_t'\in\mathbb{R}^P$ (the feature vectors at time $t$)
satisfy the following assumptions. We define the empirical feature covariance as $\hat{\Psi} = \dfrac{1}{T}S'S \in \mathbb{R}^{P\times P}.$
\begin{itemize}
  \item[(A1)] The rows $(s_t)_{t=1}^T$ are independent.
  \item[(A2)] Each row is isotropic: $\mathbb{E}[s_t]=0$ and
  $\mathbb{E}[s_t s_t'] = I_P$.
  \item[(A3)] Each $s_t$ is sub-Gaussian with parameter $K$ in the above sense.
\end{itemize}
\begin{lemma}[Spectral concentration for sub-Gaussian designs]\label{lem:SubG-spectral}
Under {\rm(A1)--(A3)}, there exist constants $a>0$ and $C_\Psi>0$, depending only on $K$,
such that for all $t\ge 0$,
\begin{equation}
\mathbb{P}\Bigl(\bigl\|\hat{\Psi}-I_P\bigr\| >
C_\Psi\Bigl(\sqrt{P/T} + t/\sqrt{T}\Bigr)\Bigr)
\;\le\; 2 e^{-a t^2}.
\label{eq:SubG-spectral}
\end{equation}
Consequently, for any fixed integer $k\ge 1$, there exists $C_k<\infty$
depending only on $k$ and $K$ such that
\begin{equation}
\mathbb{E}\Bigl[\bigl\|S'S\bigr\|^{k}\Bigr]
\;\le\; C_k T^k,
\label{eq:SubG-moment}
\end{equation}
and hence
\begin{equation}
\mathbb{E}\bigl[\|\hat\Psi\|^k\bigr]
= \mathbb{E}\Bigl[\Bigl\|\frac{1}{T}S'S\Bigr\|^k\Bigr]
\;\le\; C_k.
\end{equation}
\end{lemma}

\begin{proof}[Proof of Lemma \ref{lem:SubG-spectral}] By (A1)--(A3), the rows $s_t\in\mathbb{R}^P$ are independent, mean-zero,
isotropic, sub-Gaussian with $\|s_t\|_{\psi_2}\le K$. Since $S$ is a $T\times P$ matrix with independent isotropic sub-Gaussian rows,
by Theorem~4.6.1 of \cite{vershynin2018high}, there exist constant $C_\Psi>0$, depending only on $K$, such that for all $t\ge 0$,
\begin{equation}\label{eq:versh-46-rows}
\mathbb{P}\!\left(
\left\| \frac{1}{T} S'S - I_P \right\|
> C_\Psi \max\!\left( \delta,\delta^2 \right)
\right)
\;\le\; 2 e^{- t^2},
\qquad
\delta = \sqrt{\frac{P}{T}} + \frac{t}{\sqrt{T}} .
\end{equation}
Our first goal is to convert \eqref{eq:versh-46-rows} into a deviation
inequality stated directly in the level~$s$.  Observe that if $\Bigl\|\dfrac1T S'S - I_P\Bigr\|>s,$ then necessarily $s>C_\Psi\max(\delta,\delta^2)$.
For $s>C_\Psi$, this forces $\delta\ge 1$, hence $\delta = \sqrt{\dfrac{s}{C_\Psi}}$ and $    t
=
\sqrt{T}\Bigl(
\sqrt{\dfrac{u}{C_\Psi}}
-
\sqrt{\dfrac{P}{T}}
\Bigr). $ Since $\|SS'\|\ge 0$, we
may use the standard representation
\begin{equation}
\mathbb{E}[Z^k] = \int_0^\infty k s^{k-1}\,\mathbb{P}(Z>s)\,ds,
\qquad Z\ge 0.    
\end{equation}
Let $Z = \|XX'\|$. For $0\le s\le C_\Psi T$ (with $C_\Psi$ depending only on $K$), the
trivial bound $\mathbb{P}(Z>s)\le 1$ implies that
\begin{equation}\label{eq: part1}
\int_0^{C_\Psi T} k s^{k-1}\,\mathbb{P}(Z>s)\,ds
\;\le\; \int_0^{C_\Psi T} k s^{k-1}\,ds
\;=\; C^k_\Psi T^k,    
\end{equation}
which is of order $T^k$. For $s>C_\Psi T$, the above concentration bound gives
\begin{equation}
\mathbb{P}(Z>s)
\;=\; \mathbb{P}\Bigl(\Bigl\|\frac{1}{T}SS'\Bigr\|>s/T\Bigr)
\;\le\; 2\exp\bigl[-T\,(s/T-C_\Psi)^2\bigr]
\;=\; 2\exp\bigl[-T^{-1}(s-C_\Psi T)^2\bigr].    
\end{equation}
Hence the tail contribution satisfies
\begin{equation}
\int_{C_\Psi T}^\infty k s^{k-1}\,\mathbb{P}(Z>s)\,ds
\;\le\; 2k\int_{C_\Psi T}^\infty s^{k-1} \exp\bigl[-T^{-1}(s-C_\Psi T)^2\bigr]\,ds.    
\end{equation}
With the change of variables $u = (s-C_\Psi T)/\sqrt{T}$, so that
$s = C_\Psi T + \sqrt{T}\,u$ and $ds=\sqrt{T}\,du$, we obtain
\begin{equation}\label{eq: part2}
\begin{aligned}
\int_{C_\Psi T}^\infty s^{k-1} \exp\bigl[-T^{-1}(s-C_\Psi P)^2\bigr]\,ds
=\ & \sqrt{T}\int_0^\infty (C_\Psi T+\sqrt{T}\,u)^{k-1} e^{- u^2}\,du\\
=\ & T^{k - \frac12} \int_0^\infty (C_\Psi+ \frac{u}{\sqrt{T}})^{k-1} e^{- u^2}\,du\\
\le\ & C^*_\Psi T^{k - \frac12}
\end{aligned}
\end{equation}
with the constant $C^*_\Psi$ depending
only on $k$ and $K$. The inequality comes from the fact that Gaussian decay term shrinks much faster than any polynomial grows. Combining the two parts of the integral \eqref{eq: part1} and \eqref{eq: part2}, we arrive at
\begin{equation}
\mathbb{E}\bigl[\|SS'\|^{k}\bigr]
\;\le\; C_k T^k,  \qquad  \mathbb{E}\bigl[\|\hat\Psi\|^k\bigr]
= \mathbb{E}\Bigl[\Bigl\|\frac{1}{T}S'S\Bigr\|^k\Bigr]
\;\le\; C_k.
\end{equation}
for some finite constant $C_k$ depending only on $k$ and $K$. The proof of Lemma \ref{lem:SubG-spectral} is complete.
\end{proof}

\subsection{Main RMT Asymptotic Results}

We will need the Stieltjes transform and its derivative:
\begin{equation}
\begin{aligned}
&m(-z)\ =\ \lim_{P\to\infty} P^{-1}\tr((zI+\Psi_P)^{-1})\\
&m'_\Psi(-z)\ =\ \lim_{P\to\infty} P^{-1}\tr((zI+\Psi_P)^{-2})
\end{aligned}
\end{equation}
and 
\begin{equation}
\begin{aligned}
&\hat m(-z)\ =\ P^{-1}\tr((zI+\hat\Psi)^{-1})\\
&\hat m'(-z)\ =\ P^{-1}\tr((zI+\hat\Psi)^{-2})
\end{aligned}
\end{equation}
Let also $m(-z;c)$ be the unique, positive solution to the fixed point equation\footnote{See, for example, \cite{chernov2025test} for a proof that this equation indeed has a unique solution.}
\begin{equation}\label{fixed-point-master-app}
m(-z;c)\ =\ \frac1{1-c+cz m(-z;c)}m_{\Psi}\left(
\frac{-z}{1-c+cz m(-z;c)}
\right)\,. 
\end{equation}

We will now need the following result from \cite{kelly2024virtue}. 

\begin{proposition}\label{prop:exptrace} We have 
\begin{equation}\label{xi-def}
\lim_{T\to\infty}\frac{1}{T}\tr ((zI+\hat\Psi)^{-1}\Psi)\ \to\ \xi(z;c)\,
\end{equation}
almost surely 
and 
\begin{equation}\label{xi-def-q}
\lim_{T\to\infty}\frac1{T}S_{T+1}'(zI+\hat\Psi)^{-1}S_{T+1}\ \to\ \xi(z;c)\,
\end{equation}
in probability, where 
\begin{equation}
\xi(z;c) = \frac{1-zm(-z;c)}{c^{-1}-1+zm(-z;c)}.
\end{equation}
Similarly, 
\begin{equation}\label{xi-def1}
\lim_{T\to\infty}\frac1{T}\tr ((zI+\hat\Psi)^{-2}\Psi)\ \to\ -\xi'(z;c)\,
\end{equation}
almost surely, where 
\begin{equation}
\begin{aligned}
\xi'(z;c) &= \frac{d}{dz}\left(\frac{1-zm(-z;c)}{c^{-1}-1+zm(-z;c)}\right)\\
&=\ \frac{d}{dz}\left(-1+\frac{1}{1-c+czm(-z;c)}\right)\\
&=\ -\ \frac{c\Bigr(m(-z;c)-zm'(-z;c)\Bigr)}{\Bigr(1-c+czm(-z;c)\Bigr)^2}
\end{aligned}
\end{equation}
\end{proposition}

We will also define 
\begin{equation}\label{z_*}
Z_*(z;c)\ =\ \frac{z}{1-c+cz m(-z;c)}\ =\ z(1+\xi(z;c)),\ Z_*'=1+\xi+z\xi'\,. 
\end{equation}
to be the implicit shrinkage and 
\begin{equation}
\begin{aligned}
&\um(z;c)\ =\ 1/Z_*(z;c)\ =\ z^{-1}(1-c+cz m(-z;c))\\
&\tum(z;c)\ =\ 1/Z_*(z;c)\ =\ z^{-1}(1-c+cz \hat m(-z;c))\,.
\end{aligned}
\end{equation}

Then, the results of \cite{bai1996effect, li2020adaptable, liu2015marvcenko, knowles2017anisotropic, hastie2019surprises} imply that the following is true:

\begin{theorem}\label{thm:concentration} Suppose that $S_{t}\ =\ \Psi^{1/2} X_t,$ where $X_t=(X_{i,t})$ where $E[X_{i,t}]=0,\ E[X_{i,t}^2]=1,\ E[X_{i,t}^4]<\infty$ are independent and identically distributed, and $\Psi=E[S_tS_t']$ is a positive semi-definite, uniformly bounded signal covariance matrix. Then, in the limit as $P,T\to\infty,P/T\to c,$ for any uniformly bounded sequence of non-random matrices $A_P$, we have 
\begin{equation}
P^{-1}z\tr(A_P (zI+\underbrace{\hat\Psi}_{random})^{-1})\ -P^{-1}Z_*\tr(A_P(Z_*I+\underbrace{\Psi}_{deterministic})^{-1})\ \to\ 0
\end{equation}
almost surely. 
Similarly, for any sequence of uniformly bounded vectors $\beta,$ we have 
\begin{equation}
z\beta' (zI+\underbrace{\hat\Psi}_{random})^{-1}\beta\ -Z_*\beta'(Z_*I+\underbrace{\Psi}_{deterministic})^{-1}\beta\ \to\ 0
\end{equation}
almost surely. 
Furthermore, the same result holds for $\tilde\Psi=\hat\Psi-\bar S\bar S',$ where $\bar S=T^{-1}\sum_t S_t:$
\begin{equation}
\begin{aligned}
& P^{-1}z\tr(A_P(zI+{\tilde\Psi})^{-1})\ - P^{-1}Z_* \tr(A_P(Z_*I+{\Psi})^{-1})\ \to\ 0\\
&z\beta'(zI+{\tilde\Psi})^{-1}\beta\ -\ Z_*\beta'(Z_*I+{\Psi})^{-1}\beta\ \to\ 0\,
\end{aligned}
\end{equation}
almost surely. 
\end{theorem}
Informally, we can rewrite this theorem as 
\begin{equation}
\begin{aligned}
&z(zI+{\hat\Psi})^{-1}\ \approx\ Z_*(Z_*I+{\Psi})^{-1}\\
&z(zI+{\tilde\Psi})^{-1}\ \approx\ Z_*(Z_*I+{\Psi})^{-1}\,. 
\end{aligned}
\end{equation}

\subsection{Auxiliary CLT results}
First, we will need the following 

\begin{lemma}\label{ext-sriniv} Suppose that $E[\eps_t^3]=0$ and $E[\eps_t^4]=3.$ Let $a_T$ be a sequence of bounded vectors vectors and $A_T$ a sequence of bounded random matrices, all of them being independent of $\eps_t.$ Then, $(a_T'\eps\|a_T\|^{-1},\ T^{-1/2}(\eps'A_T\eps-\gs_\eps^2 \tr(A_T))/\sqrt{2\frac1T\gs_\eps^4 \tr(A_T^2)})$ converges to a $\cN(0,I)$ distribution. That is, these two random variables are asymptotically independent standard Normals. 
\end{lemma}

\begin{proof}[Proof of Lemma \ref{ext-sriniv}] The proof of Lemma \ref{ext-sriniv} is completely analogous to that of the main result in \cite{srivastava2009test} and follows closely the classical argument from \cite{de1987central}, plus the Cramer-Wold device (\cite{Billingsley1995}). Without loss of generality, we normalize $\gs_\eps^2=1.$ Let $\tilde a_T=a_T/\|a_T\|,\ \tilde A_T=A_T/\sqrt{2\frac1T \tr(A_T^2)}.$ 

With Cramer-Wold, we build for any vector $(q_1,q_2)$
\begin{equation}
S(T) =\ q_1 \tilde a_T'\eps\ +\ q_2 T^{-1/2}(\eps'\tilde A_T\eps-\tr(\tilde A_T))
\end{equation}
As in \cite{de1987central}, we define 
\begin{equation}
S(t)\ =\ \sum_{\tau\le t} \Bigr(q_1 \tilde a(\tau)\eps_\tau\ +\ q_2T^{-1/2}(\eps_\tau^2-1)\tilde A_{\tau,\tau}+q_2T^{-1/2}2\sum_{\tau_1<\tau}\eps_{\tau_1}\eps_\tau \tilde A_{\tau_1,\tau} \Bigr)
\end{equation}
By direct calculation, $S(t)$ is a Martingale and then, standard arguments based on the Lindeberg CLT implies asymptotic normality. 
\end{proof}

We will also use the following multi-variate extension. 

\begin{lemma}\label{ext-sriniv-multivar} Suppose that $E[\eps_t^3]=0$ and $E[\eps_t^4]=3.$ Let $a_{k,T}$ be a sequence of bounded vectors vectors and $A_{k,T}$ a sequence of bounded, symmetric random matrices, all of them being independent of $\eps_t;\ k=1,\cdots,K.$ Then, $((a_{k,T}'\eps\|a_{k,T}\|^{-1},\ T^{-1/2}(\eps'A_{k,T}\eps-\gs_\eps^2 \tr(A_{k,T}))/\sqrt{2\frac1T\gs_\eps^4 \tr(A_{k,T}^2)}))_{k=1}^K$ converges to a $\cN(0,\Sigma(\gs_\eps^2))$ distribution. The structure of the covariance matrix $\Sigma(\gs_\eps)$ is as follows: any of the linear components, $a_{k,T}'\eps\|a_{k,T}\|^{-1},$ has zero covariance with any quadratic component; by contrast, quadratic terms are correlated with 
\begin{equation}
\Cov(\eps'A\eps,\ \eps'B\eps)\ =\ 2\tr(AB)
\end{equation}
while linear terms give 
\begin{equation}
\Cov(\eps'a,\ \eps'b)\ =\ a'b\,. 
\end{equation}
\end{lemma}

\begin{proof}[Proof of Lemma \ref{ext-sriniv-multivar}] The joint normality follows by the same Cramer-Wold argument. 

The only claim to prove is the covariance formula. We have 
\begin{equation}
\Var[\eps'(A+B)\eps]\ =\ 2\tr((A+B)^2)\ =\ 2\tr(A^2+B^2+2AB)
\end{equation}
but, by the variance decomposition 
\begin{equation}
\Var[\eps'(A+B)\eps]\ =\ \Var[\eps'A\eps]+\Var[\eps'B\eps]+2\Cov(\eps'A\eps,\eps'B\eps)
\end{equation}
Comparing, we get the required identity. 
\end{proof}

\section{Proof of Proposition \ref{dec} and Theorem \ref{llg1}}

The goal of this section is to prove Theorem \ref{llg1}. We also proof the special case of the ridge regression as a separate
corollary.

\begin{proof}[Proof of Proposition \ref{dec}] Recall that for each OOS period $t$,
\begin{equation}
\hat f_t\ =\ \hat f^s_t\ +\ \hat f^\eps_t\ =\ \cK(S_t,S)f\ +\ \cK(S_t,S)\eps\,. 
\end{equation}
Hence, period by period we have
\begin{equation}
\begin{aligned}
(y_{t+1}-\hat f_t)^2\ =\ &  \Bigr[f_t+\eps_{t+1}-( \hat f_{t}^s + \hat f_{t}^\eps) \Bigr]^2 \\
=\ &  \Bigr(f_t-\hat f_{t}^s -\hat f_{t}^\eps+\eps_{t+1} \Bigr)^2 \\
=\ &  (f_t-\hat f_{t}^s )^2 +\eps_{t+1}^2+ (\hat f_{t}^{\eps})^2 -2(f_t-\hat f_{t}^s)(\hat f_{t}^{\eps}-\eps_{t+1}) -2\eps_{t+1}'\hat f_{t}^{\eps}
\end{aligned}
\end{equation}
Averaging over $t = T, \dots, T+T_{\text{OOS}}-1$ yields
\begin{equation}
\begin{aligned}
& MSE_{OOS}(\hat f)\\
=\ & \frac1{T_{OOS}}\sum_{t=T}^{T+T_{OOS}-1}(y_{t+1}-\hat f_t)^2\\
=\ & E_{OOS}[\eps^2] + \underbrace{E_{\text{OOS}}\bigl[(f_t - \hat f_{t}^s)^2\bigr]}_{\widehat \cB} + \underbrace{E_{\text{OOS}}\bigl[(\hat f_{t}^s)^2\bigr]}_{\widehat \cV} -\underbrace{2E_{OOS}\bigl[(f-\hat f^s)(\hat f^\eps-\eps)+\eps \hat f^\eps\bigr]}_{\widehat \cI}
\end{aligned}
\end{equation}
The proof of Proposition \ref{dec} is complete.

\end{proof}

\begin{proof}[Proof of Theorem~\ref{llg1}]
We have 
\begin{equation}\label{three-i}
\begin{aligned}
\frac12\widehat\cI\ =\ & -E_{OOS}[(f-\hat f^s)\hat f_\eps]\ +\ E_{OOS}[(f-\hat f^s)\eps]\ -\ E_{OOS}[\eps\hat f_\eps]\\
=\ & \widehat\cI_1+\widehat\cI_2+\widehat\cI_3\,. 
\end{aligned}
\end{equation}
We will use the inequality 
\begin{equation}
E[(\eps'A\eps)^2]\ - \gs_\eps^2 E[(\tr(A))^2]\ \le\  CE[\tr(A^2)]
\end{equation}
where $C$ is a constant (see Lemma \ref{lem-quad}). Let $\eps=\binom{\eps_{OOS}}{\eps_{IS}}$ where we use the obvious notations $\eps_{OOS},\eps_{IS}$ to denote subsets of indices. 
Then, 
\begin{equation}
\begin{aligned}
&\hat\cI_3\ =\ \frac1{T_{OOS}}\eps_{OOS}'\widehat{\cK}\eps_{IS}
\end{aligned}
\end{equation}
Then, by the independence of $\eps,$ we get  
\begin{equation}\label{i3sq}
\begin{aligned}
E[\hat\cI_3^2]\ =\ &\frac1{T_{OOS}^2}E[(\eps_{OOS}'\widehat{\cK}\eps_{IS})^2]\\
=\  &\frac1{T_{OOS}^2}E[\eps_{OOS}'\widehat{\cK}\eps_{IS}\eps_{IS}'\widehat{\cK}'\eps_{OOS}]\\
=\ &\frac1{T_{OOS}^2}\gs_\eps^2E[\tr(\widehat{\cK}\eps_{IS}\eps_{IS}'\widehat{\cK}')]\\
=\ & \frac1{T_{OOS}^2}\gs_\eps^4 E[\tr(\widehat{\cK}\widehat{\cK}')]
\end{aligned}
\end{equation}

by assumption. Next, we will use the identity 
\begin{equation}
E[(a'\eps)^2]\ =\ E[\|a\|^2]
\end{equation}
for any random vector $a$ independent of $\eps$. Then, 
\begin{equation}\label{i12sq}
\begin{aligned}
&E[\widehat\cI_1^2]\ =\ \frac1{T_{OOS}^2 } \gs_\eps^2 E\Bigr[ E_{OOS}[\|(f-\hat f^s)\widehat\cK\|^2] \Bigr]\\
&E[\widehat\cI_2^2]\ =\ \frac1{T_{OOS}^2} \gs_\eps^2 E\Bigr[  E_{OOS}[\|(f-\hat f^s)\|^2]  \Bigr] \,,
\end{aligned}
\end{equation}
and the claim follows from the hypotheses of the theorem. Since $\widehat\cB(z),\widehat\cV(z)\ge0$, this implies
\begin{equation}\label{eq:mse-lower}
MSE_{OOS}(\hat f)\ \ge\ E_{OOS}[\eps^2]\ +\ \widehat\cV\ +\ o(1)\ =\ \gs_\eps^2\ +\ \widehat\cV\ +\ o(1),
\end{equation}
where the last equality follows from the definition of the OOS expectation. Furthermore, 
\begin{equation}
\widehat\cV\ =\  \,\eps'\,\Big(\dfrac{1}{T_{OOS}}\widehat\cK'\widehat\cK \Big)\,\eps\,. 
\end{equation}
and, hence, 
\begin{equation}
\widehat\cV\ -\ \gs_\eps^2\,\widehat\cL\ \to\ 0
\end{equation}
in $L_2$ and in probability by Lemma \ref{lem-quad}. 
Substituting this expression into~\eqref{eq:mse-lower} yields
\begin{equation}
MSE_{OOS}(\hat f)\ \ge\ \gs_\eps^2\Bigl(1+\widehat\cL\Bigr)\ +\ o(1),    
\end{equation}
and dividing both sides by $1+\widehat\cL$ and taking $\liminf_{T,T_{OOS}\to\infty}$ (appealing to the continuous mapping theorem) gives the desired inequality
\begin{equation}
\liminf\ \frac{MSE(\hat f)}{1+\widehat\cL}\ \ge\ \gs_\eps^2.  
\end{equation}
When $f_t=0$, the bias term $\widehat\cB$ vanishes, and the bound becomes tight. In that case,
\[
MSE_{OOS}(\hat f)\ =\ \gs_\eps^2+\widehat\cV+o(1)\ =\ \gs_\eps^2 \Bigr(1+\widehat\cL \Bigr)+o(1),
\]
so that the inequality holds with equality asymptotically. The proof of Theorem~\ref{llg1} is complete.
\end{proof}

\begin{corollary}\label{dec_ridge} 
For the ridge estimator, we have 
\begin{equation}\label{eq:risk3_expanded-main-n}
\begin{aligned}
&MSE_{OOS}(\hat\beta)\ =\ E_{OOS}[\eps^2]\ +\ \widehat\cB(z)\ -\ \widehat\cI(z)\ +\ \widehat\cV(z)\,,
\end{aligned}
\end{equation}
where 
\begin{equation}\label{eq:risk3_expanded1-main-n}
\begin{aligned}
&\widehat\cB(z)\ =\ \underbrace{z^2\,\beta'(zI+\hat{\Psi})^{-1}\hat \Psi_{OOS} (zI+\hat{\Psi})^{-1}\beta}_{bias}\ \ge\ 0\\
&\widehat\cV(z)\ =\ \underbrace{\frac{1}{T^2}\,\eps' S (zI+\hat{\Psi})^{-1}\hat\Psi_{OOS} (zI+\hat{\Psi})^{-1} S'\eps}_{variance}\ \ge\ 0,
\end{aligned}
\end{equation}
while 
\begin{equation}
\widehat\cI(z)\ =\ \underbrace{\frac{2z}{T}\,\beta'(zI+\hat{\Psi})^{-1}\hat\Psi_{OOS} (zI+\hat{\Psi})^{-1} S'\eps+2 E_{OOS}[\eps'S(\beta-\hat\beta)]}_{interaction}
\end{equation} 
\end{corollary}

\begin{proof}[Proof of Corollary \ref{dec_ridge}] 
Recall that ridge estimator $\hat \beta(z)$ 
\begin{equation}
\begin{aligned}
\hat \beta(z)\ &=\ (zI\ +\  \hat\Psi)^{-1} T^{-1} S'd\,,
\end{aligned}
\end{equation} implies the decomposition
\begin{equation}\label{noise-dec}
\begin{aligned}
&\hat\beta(z)\ =\ \beta\ -\ \underbrace{z\,(zI+\hat{\Psi})^{-1}\beta}_{bias}\ +\ \underbrace{T^{-1}(zI+\hat{\Psi})^{-1}S'\eps}_{noise}\,. 
\end{aligned}
\end{equation}   
We have the realized out-of-sample MSE 
\begin{equation}\label{def-mse1_ridge}
\begin{aligned}
MSE_{OOS}(\hat\beta)\ =\ & E_{OOS}[(d-\hat\beta'S)^2]\\
=\ & E_{OOS}[(\beta'S-\hat\beta'S + \eps)^2]\\
=\ & (\beta-\hat{\beta})' \Psi_{OOS} (\beta-\hat{\beta}) +2 E_{OOS}[\eps'S(\beta-\hat\beta)] +\ E_{OOS}[\eps^2]\,     
\end{aligned}
\end{equation}
Plugging the decomposition into the first term, we obtain
\begin{equation}
\begin{aligned}
&(\beta-\hat{\beta})' \Psi_{OOS} (\beta-\hat{\beta})\\
=\ & \Bigr(z\,(zI+\hat{\Psi})^{-1}\beta - T^{-1}(zI+\hat{\Psi})^{-1}S'\eps\Bigr)' \hat\Psi_{OOS} \Bigr(z\,(zI+\hat{\Psi})^{-1}\beta - T^{-1}(zI+\hat{\Psi})^{-1}S'\eps\Bigr)\\
=\ & z^2\,\beta'(zI+\hat{\Psi})^{-1}\hat\Psi_{OOS} (zI+\hat{\Psi})^{-1}\beta\ +\ \frac{1}{T^2}\,\eps' S (zI+\hat{\Psi})^{-1}\hat\Psi_{OOS} (zI+\hat{\Psi})^{-1} S'\eps \\
&- \frac{2z}{T}\,\beta'(zI+\hat{\Psi})^{-1}\hat\Psi_{OOS} (zI+\hat{\Psi})^{-1} S'\eps
\end{aligned}
\end{equation}
The proof of the Corollary \ref{dec_ridge} is complete.
\end{proof}

\begin{corollary}\label{ridge_MSE_LLG}
Suppose that 
\begin{equation}
E[\bigl\|E_T[\hat\Psi_{OOS}^2]\bigr\|]= o(\min(T_{OOS},T))
\quad\text{as } T_{OOS}\to\infty.    
\end{equation}
In the limit as $T,T_{OOS}\to\infty$, the MSE from \eqref{def-mse1_ridge} satisfies
\begin{equation}\label{llg2_ridge}
\liminf \frac{MSE_{OOS}(\hat\beta)}{1+\widehat\cL(z)}\ \ge\ 
\underbrace{\sigma_\eps^2}_{\text{infeasible}},
\end{equation}
in probability, where 
\begin{equation}\label{old-llg_ridge}
\widehat\cL(z)\ =\ \frac{1}{T}\tr \Bigl(\hat\Psi_{OOS}\hat\Psi (zI+\hat\Psi)^{-2} \Bigr)
\end{equation}
is the Limits-To-Learning Gap (LLG). This bound turns into an identity for $\beta=0$.
\end{corollary}
\begin{proof}[Proof of Corollary \ref{ridge_MSE_LLG}]
We have 
\begin{equation}\label{eq: 1/2I}
\frac12\widehat\cI(z)\ =\ \frac{z}{T}\,\beta'(zI+\hat{\Psi})^{-1}\hat\Psi_{OOS} (zI+\hat{\Psi})^{-1} S'\eps+ E_{OOS}[\eps'S(\beta-\hat\beta)]
\end{equation}
Note that
\begin{equation}
\begin{aligned}
& E_T \Bigr[ \bigl(\frac{z}{T}\,\beta'(zI+\hat{\Psi})^{-1}\hat\Psi_{OOS} (zI+\hat{\Psi})^{-1} S'\eps \bigr)^2 \Bigr]\\
=\ & \frac{\gs_\eps^2 z^2}{T^2}  \beta'(zI+\hat{\Psi})^{-1}\hat\Psi_{OOS} (zI+\hat{\Psi})^{-1} S'S (zI+\hat{\Psi})^{-1}\hat\Psi_{OOS}(zI+\hat{\Psi})^{-1}\beta    \\
=\ & \frac{\gs_\eps^2 z^2}{T} \beta'(zI+\hat{\Psi})^{-1}\hat\Psi_{OOS} (zI+\hat{\Psi})^{-1}\hat \Psi (zI+\hat{\Psi})^{-1}\hat\Psi_{OOS}(zI+\hat{\Psi})^{-1}\beta  
\end{aligned}
\end{equation}
Using the inequality $x'ABA x \le \|B\|\,x'A^2x$ for any symmetric $B$ and any vector $x$, we obtain 
\begin{equation}
\hat\Psi_{OOS} (zI+\hat\Psi)^{-1} \hat\Psi (zI+\hat\Psi)^{-1} \hat\Psi_{OOS}
\ \preceq\ 
\|(zI+\hat\Psi)^{-1}\hat\Psi (zI+\hat\Psi)^{-1}\|\,\hat\Psi_{OOS}^2.
\end{equation}
Since $\|(zI+\hat\Psi)^{-1}\hat\Psi\|\le 1$ and $\|(zI+\hat\Psi)^{-1}\|\le 1/z$, it follows
\begin{equation}
\begin{aligned}
& E_T \Bigr[  \bigl(\frac{z}{T}\,\beta'(zI+\hat{\Psi})^{-1}\hat\Psi_{OOS} (zI+\hat{\Psi})^{-1} S'\eps \bigr)^2   \Bigr]\\
\leq\ & \frac{\gs_\eps^2 z^2}{T} \frac{1}{z} \beta'(zI+\hat\Psi)^{-1}\,\hat\Psi_{OOS}^2\,(zI+\hat\Psi)^{-1}\beta\\
\leq\ & \frac{\gs_\eps^2 \|(zI+\hat\Psi)^{-1}\beta\|^2 }{zT}  \| E_T[\hat\Psi_{OOS}^2] \|
\end{aligned}
\end{equation}
By the law of iterated expectations and the assumption
$\|E_T[\hat\Psi_{OOS}^2]\|=o(\min(T_{OOS},T))$ as $T_{OOS}\to\infty$, we have
\begin{equation}
E \Bigr[ \bigl(\frac{z}{T}\,\beta'(zI+\hat{\Psi})^{-1}\hat\Psi_{OOS} (zI+\hat{\Psi})^{-1} S'\eps \bigr)^2 \Bigr] \leq\ \frac{\gs_\eps^2 \|\beta\|^2 }{zT} E [  \| E_T[\hat\Psi_{OOS}^2] \|  ] = o(1)
\end{equation}
 Thus, the first term in \eqref{eq: 1/2I} is negligible in $L^2$. We next bound the square of the out-of-sample term in \eqref{eq: 1/2I}. Recall that
\begin{equation}
\begin{aligned}
E_{OOS}[\eps'S(\beta-\hat\beta)] & = \frac{1}{T_{OOS}}\eps_{OOS}'S_{OOS}\Bigr(z\,(zI+\hat{\Psi})^{-1}\beta - T^{-1}(zI+\hat{\Psi})^{-1}S'\eps\Bigr)
\end{aligned}
\end{equation}
We have 
\begin{equation}\label{eq: E_OOS_I}
\begin{aligned}
&  \Bigr(   E_{OOS}[\eps'S(\beta-\hat\beta)]   \Bigr) ^2     \\
=\ & \frac{1}{T^2_{OOS}}  \eps_{OOS}'S_{OOS}\Bigr(z\,(zI+\hat{\Psi})^{-1}\beta - T^{-1}(zI+\hat{\Psi})^{-1}S'\eps\Bigr)      \Bigr(z\,(zI+\hat{\Psi})^{-1}\beta - T^{-1}(zI+\hat{\Psi})^{-1}S'\eps\Bigr)'S'_{OOS}\eps_{OOS} \\
=\ & \frac{z^2}{T^2_{OOS}}\eps_{OOS}'S_{OOS} (zI+\hat{\Psi})^{-1}\beta \beta'(zI+\hat{\Psi})^{-1}S'_{OOS}\eps_{OOS}\\
& -  \frac{2z}{T T^2_{OOS}}  \eps_{OOS}'S_{OOS}(zI+\hat{\Psi})^{-1}\beta \eps' S (zI+\hat{\Psi})^{-1} S'_{OOS}\eps_{OOS}\\
& + \frac{1}{T^2 T^2_{OOS}} \eps_{OOS}'S_{OOS}(zI+\hat{\Psi})^{-1}S'\eps \eps' S (zI+\hat{\Psi})^{-1}S'_{OOS}\eps_{OOS}
\end{aligned}
\end{equation}
Note that, for any unit vector $h$, we have
\begin{equation}
\begin{aligned}
\|E_T[\hat\Psi_{OOS}]h\|^2
\le\ & E_T[\|\hat\Psi_{OOS}h\|^2]\\
=\ & E_T[h'\hat\Psi_{OOS}^2h]\\
=\ & h'E_T[\hat\Psi_{OOS}^2]h \\
\le\ &
\|E_T[\hat\Psi_{OOS}^2]\|,    
\end{aligned}
\end{equation}
so that, taking supremum over all unit $h$,
\begin{equation}
    \|E_T[\hat\Psi_{OOS}]\|\le \|E_T[\hat\Psi_{OOS}^2]\|^{\tfrac12}
\end{equation}
Taking conditional expectations of the first term in \eqref{eq: E_OOS_I} on the training sample, we obtain
\begin{equation}
\begin{aligned}
&E_T \Bigr[ \frac{z^2}{T^2_{OOS}}\eps_{OOS}'S_{OOS} (zI+\hat{\Psi})^{-1}\beta \beta'(zI+\hat{\Psi})^{-1}S'_{OOS}\eps_{OOS}    \Bigr]\\
=\ & \frac{z^2\sigma_\eps^2}{T_{OOS}^2}\,
\beta'(zI+\hat{\Psi})^{-1}\,E_T[S_{OOS}'S_{OOS}]\,(zI+\hat{\Psi})^{-1}\beta\\
=\ & \frac{z^2\sigma_\eps^2}{T_{OOS}}\,
\beta'(zI+\hat{\Psi})^{-1}\,E_T[\hat \Psi_{OOS}]\,(zI+\hat{\Psi})^{-1}\beta\\
\leq\ & \frac{z^2\sigma_\eps^2}{T_{OOS}}\,
\|(zI+\hat{\Psi})^{-1}\|^2\|\beta\|^2\,\|E_T[\hat\Psi_{OOS}]\| \\
\leq\ & \frac{\sigma_\eps^2 \|\beta\|^2}{T_{OOS}}\|E_T[\hat\Psi^2_{OOS}]\|^{\tfrac{1}{2}} \\
=\ & o(\frac{1}{\sqrt{T_{OOS}}})
\end{aligned}
\end{equation}
Then taking unconditional expectation, we have
\begin{equation}
E \Bigr[ \frac{z^2}{T^2_{OOS}}\eps_{OOS}'S_{OOS} (zI+\hat{\Psi})^{-1}\beta \beta'(zI+\hat{\Psi})^{-1}S'_{OOS}\eps_{OOS}    \Bigr] = o(1)
\end{equation}
Taking conditional expectations of the second term in \eqref{eq: E_OOS_I} on the training sample, by the independence of $\eps$, we obtain 
\begin{equation}
\begin{aligned}
& E_T \Bigr[ \frac{2z}{T T^2_{OOS}}  \eps_{OOS}'S_{OOS}(zI+\hat{\Psi})^{-1}\beta \eps' S (zI+\hat{\Psi})^{-1} S'_{OOS}\eps_{OOS} \Bigr]\\
=\ & \frac{2z \gs^2_{\eps}}{T T^2_{OOS}} E_T\Bigr[\tr(S_{OOS}(zI+\hat{\Psi})^{-1}\beta \eps' S (zI+\hat{\Psi})^{-1} S'_{OOS}  ) \Bigr]\\
=\ & \frac{2z \gs^2_{\eps}}{T T_{OOS}} E_T\Bigr[\eps' S (zI+\hat{\Psi})^{-1}\hat \Psi_{OOS}(zI+\hat{\Psi})^{-1}\beta  \Bigr]\\
=\ & 0
\end{aligned}
\end{equation}
Thus, we have $E \Bigr[ \dfrac{2z}{T T^2_{OOS}}  \eps_{OOS}'S_{OOS}(zI+\hat{\Psi})^{-1}\beta \eps' S (zI+\hat{\Psi})^{-1} S'_{OOS}\eps_{OOS} \Bigr] = 0$.
Finally, for the third term in \eqref{eq: E_OOS_I}, we obtain the conditional expectations
\begin{equation}
\begin{aligned}
& E_T \Bigr[ \frac{1}{T^2 T^2_{OOS}} \eps_{OOS}'S_{OOS}(zI+\hat{\Psi})^{-1}S'\eps \eps' S (zI+\hat{\Psi})^{-1}S'_{OOS}\eps_{OOS} \Bigr]    \\
=\ & \frac{\gs^2_{\eps}}{T^2 T^2_{OOS}} E_T  \Bigr[ \tr(S_{OOS}(zI+\hat{\Psi})^{-1}S'\eps \eps' S (zI+\hat{\Psi})^{-1}S'_{OOS})    \Bigr]\\
=\ & \frac{\gs^2_{\eps}}{T^2 T_{OOS}}  E_T  \Bigr[ \eps' S (zI+\hat{\Psi})^{-1} \hat \Psi_{OOS} (zI+\hat{\Psi})^{-1} S'\eps \Bigr] \\
=\ & \frac{\gs^4_{\eps}}{T^2 T_{OOS}} E_T  \Bigr[ \tr( S (zI+\hat{\Psi})^{-1} \hat \Psi_{OOS} (zI+\hat{\Psi})^{-1} S' ) \Bigr] \\
=\ & \frac{\gs^4_{\eps}}{T T_{OOS}} E_T  \Bigr[ \tr(  (zI+\hat{\Psi})^{-1} \hat \Psi_{OOS} (zI+\hat{\Psi})^{-1} \hat\Psi ) \Bigr]\\
\leq\ & \frac{\gs^4_{\eps}}{z T T_{OOS}}  E_T  [ \tr( \hat\Psi_{OOS} )   ] \\
\leq\ & \frac{\gs^4_{\eps}}{z T T_{OOS}} (E_T  [ \tr( \hat\Psi_{OOS} )^2   ])^{\tfrac12}\\
\leq\ & \frac{\gs^4_{\eps}}{z T T_{OOS}} P^{\tfrac12} (E_T  [ \tr( \hat\Psi_{OOS}^2 )   ])^{\tfrac12}\\
\leq\ & \frac{\gs^4_{\eps}}{z T T_{OOS}} P^{\tfrac12} ( P \|E_T[\hat\Psi_{OOS}^2] \| )^{\tfrac12}\\
=\ & \frac{\gs^4_{\eps}P}{z T T_{OOS}}\|E_T[\hat\Psi_{OOS}^2] \|^{\tfrac12}\\
=\ & o(\frac{P}{T}\frac{1}{\sqrt{T_{OOS}}}) 
\end{aligned}
\end{equation}
Then taking unconditional expectation, we have
\begin{equation}
E \Bigr[ \frac{1}{T^2 T^2_{OOS}} \eps_{OOS}'S_{OOS}(zI+\hat{\Psi})^{-1}S'\eps \eps' S (zI+\hat{\Psi})^{-1}S'_{OOS}\eps_{OOS} \Bigr] = o(1)
\end{equation}
Thus the third term in \eqref{eq: 1/2I} is negligible in $L^2$. Since $\widehat\cB(z),\widehat\cV(z)\ge0$, this implies
\begin{equation}\label{eq:mse-lower_ridge}
MSE_{OOS}(\hat \beta)\ \ge\ E_{OOS}[\eps^2]\ +\ \widehat\cV(z) +\ o(1)\ =\ \gs_\eps^2\ +\ \widehat\cV(z) +\ o(1),
\end{equation}
Recall that
\begin{equation}
\widehat\cV(z)\ =\ \frac{1}{T^2}\,\eps' S (zI+\hat{\Psi})^{-1}\hat\Psi_{OOS} (zI+\hat{\Psi})^{-1} S'\eps.
\end{equation}
Note that we have
\begin{equation}
\begin{aligned}
&\frac{1}{T}\gs^2_{\eps}\tr \Bigr(\frac{1}{T} S (zI+\hat{\Psi})^{-1}\hat\Psi_{OOS} (zI+\hat{\Psi})^{-1} S' \Bigr)\\
=\ & \frac{1}{T^2}\gs^2_{\eps}\tr \Bigr(S'S (zI+\hat{\Psi})^{-1}\hat\Psi_{OOS} (zI+\hat{\Psi})^{-1}  \Bigr)\\
=\ & \frac{1}{T}\gs^2_{\eps}\tr \Bigr(\hat \Psi (zI+\hat{\Psi})^{-1}\hat\Psi_{OOS} (zI+\hat{\Psi})^{-1}  \Bigr)\\
=\ & \frac{1}{T} \gs^2_{\eps} \tr \Bigl(\hat\Psi_{OOS}\hat\Psi (zI+\hat\Psi)^{-2} \Bigr)\\
=\ & \gs^2_{\eps} \widehat\cL(z).
\end{aligned}
\end{equation}
Lemma \ref{lem-quad} implies that
\begin{equation}
\widehat\cV(z) -\ \gs_\eps^2\,\widehat\cL(z) \to\ 0
\end{equation}
in $L_2$ and in probability. Substituting this expression into~\eqref{eq:mse-lower_ridge} yields
\begin{equation}
MSE_{OOS}(\hat \beta)\ \ge\ \gs_\eps^2\Bigl(1+\widehat\cL(z)\Bigr)\ +\ o(1),    
\end{equation}
and dividing both sides by $1+\widehat\cL(z)$ and taking $\liminf_{T,T_{OOS}\to\infty}$ (appealing to the continuous mapping theorem) gives
\begin{equation}
\liminf\ \frac{MSE(\hat \beta)}{1+\widehat\cL(z)}\ \ge\ \gs_\eps^2.  
\end{equation}
When $\beta=0$, the bias term $\widehat\cB(z)$ vanishes, and the bound becomes tight. In that case,
\[
MSE_{OOS}(\hat \beta)\ =\ \gs_\eps^2+\widehat\cV(z)+o(1)\ =\ \gs_\eps^2 \Bigr(1+\widehat\cL(z) \Bigr)+o(1),
\]
so that the inequality holds with equality asymptotically. The proof of Corollary \ref{ridge_MSE_LLG} is complete.
\end{proof}

\section{CLT}

\subsection{CLT For MSE}

The following Lemma is a direct consequence of Proposition \ref{dec} and Lemma \ref{ext-sriniv-multivar}.

\begin{lemma}\label{lem:MSE-CLT} We have 
\begin{equation}
T^{1/2}\dfrac{\dfrac{MSE_{OOS}(\hat f)-\widehat\cB}{1+\widehat\cL}-\gs_\eps^2}{\gs_{MSE}^2}\ \to\ \cN(0,1)
\end{equation}
in distribution, where 
\begin{equation}\label{eq:sigma2_MSE}
\begin{aligned}
&\gs_{MSE}^2\ =\ \frac{2\dfrac{T}{T_{OOS}}\gs_\eps^4+\gs_\eps^4\gs_V^2+\gs_\eps^2 \gs_I^2+\gs_\eps^2\gs_{I,OOS}^2}{(1+\widehat\cL)^2},
\end{aligned}
\end{equation}
where 
\begin{equation}\label{defn:sigmav_sim}
\begin{aligned}
\gs_V^2\ =\ &2TT_{OOS}^{-2}\tr ((\widehat\cK'\widehat\cK)^2)\\
\gs_I^2\ =\ & 4T T_{OOS}^{-2}\|(f-\hat f^s)\widehat\cK\|^2\\
\gs_{I,OOS}^2\ =\ &  4\frac{T}{T_{OOS}} \Big(\widehat\cB+\ \gs_\eps^2 \widehat\cL \Big)
\end{aligned}
\end{equation}
\end{lemma}

\begin{proof}[Proof of Lemma \ref{lem:MSE-CLT}] By Proposition \ref{dec}, 
\begin{equation}\label{eq:risk3_expanded-main-a}
\begin{aligned}
&MSE_{OOS}(\hat f)-\widehat\cB\ =\ E_{OOS}[\eps^2]\ +\ \widehat\cI\ +\ \widehat\cV\,,
\end{aligned}
\end{equation}
where 
\begin{equation}\label{eq:risk3_expanded1-main-a}
\begin{aligned}
&\widehat\cV\ =\ T_{OOS}^{-1}\eps'\widehat\cK'\widehat\cK \eps\,. 
\end{aligned}
\end{equation}
By Lemma \ref{ext-sriniv-multivar}, these three terms satisfy 
\begin{equation}
\begin{aligned}
T_{OOS}^{1/2}(E_{OOS}[\eps^2]-\gs_\eps^2)\ \sim\ & \cN(0,2\gs_\eps^4)\\
\frac{T^{1/2}(\widehat\cV-\gs_\eps^2\widehat\cL)}{\gs_V\gs_\eps^2}\ \sim\ & \cN(0,1)
\end{aligned}
\end{equation}
where, under the made assumptions of $E[\eps_t^4]=3,$
\begin{equation}
\begin{aligned}
\gs_V^2\ =\  2TT_{OOS}^{-2}\tr ((\widehat\cK'\widehat\cK)^2)\,.
\end{aligned}
\end{equation}
Similarly, under the made assumptions of $E[\eps_t^3]=0,\ E[\eps_t^4]=3,$ we have that the three terms \eqref{three-i} are jointly Gaussian,
\begin{equation}
\widehat\cI\ =\ 2\hat\cI_1\ +\ (2\hat\cI_2+2\hat\cI_3)
\end{equation}
and asymptotically uncorrelated, so that, by \eqref{i12sq}-\eqref{i3sq}, 
\begin{equation}
\begin{aligned}
E_\eps[\widehat\cI^2]\ =\ & 4 E_\eps[\widehat\cI_1^2]\ +\ 4 E_\eps[\widehat\cI_2^2]\ +\ 4 E_\eps[\widehat\cI_3^2]\\
=\ & 4\gs_\eps^2 \frac1{T_{OOS}^2}\|(f-\hat f^s)\widehat{\cK}\|^2 +\ 4\gs_\eps^2 \frac1{T_{OOS}^2}\|(f-\hat f^s)\|^2 +\ 4\gs_\eps^4 \frac1{T_{OOS}^2}\tr(\widehat{\cK}\widehat{\cK}')\\
=\ & 4\gs_\eps^2 \frac1{T_{OOS}^2}\|(f-\hat f^s)\widehat{\cK}\|^2 +\ 4\gs_\eps^2 \frac1{T_{OOS}}\widehat\cB +\ 4\gs_\eps^4 \frac1{T_{OOS}}\widehat\cL\,. 
\end{aligned}
\end{equation}
Thus, Lemma \ref{ext-sriniv-multivar} implies 
\begin{equation}
\frac{T^{1/2}\widehat\cI(z)}{(\gs_\eps^2\gs_I^2+\gs_\eps^2\gs_{I,OOS}^2)^{1/2}}\ \to\ \cN(0,1),
\end{equation}
The proof of Lemma \ref{lem:MSE-CLT} is complete. 
\end{proof}

\begin{corollary}\label{lem:MSE-CLT-ridge} We have 
\begin{equation}
T^{1/2}\dfrac{\dfrac{MSE_{OOS}(\hat f)-\widehat\cB(z)}{1+\widehat\cL}-\gs_\eps^2}{\gs_{MSE}^2}\ \to\ \cN(0,1)
\end{equation}
in distribution, where 
\begin{equation}\label{eq:sigma2_MSE2}
\begin{aligned}
&\gs_{MSE}^2\ =\ \frac{2\dfrac{T}{T_{OOS}}\gs_\eps^4+\gs_\eps^4\gs_V^2+\gs_\eps^2 \gs_I^2+\gs_\eps^2\gs_{I,OOS}^2}{(1+\widehat\cL)^2},
\end{aligned}
\end{equation}
where 
\begin{equation}\label{defn:sigmav}
\begin{aligned}
\gs_V^2\ =\ &2\frac{1}{T}\tr\Bigr((zI+\hat{\Psi})^{-1}\hat\Psi_{OOS} (zI+\hat{\Psi})^{-1}\hat\Psi(zI+\hat{\Psi})^{-1}\hat\Psi_{OOS} (zI+\hat{\Psi})^{-1}\hat\Psi\Bigr)\\
\gs_I^2\ =\ & 4z^2\beta'(zI+\hat{\Psi})^{-1}\hat\Psi_{OOS} (zI+\hat{\Psi})^{-1}\hat\Psi (zI+\hat{\Psi})^{-1}\hat\Psi_{OOS} (zI+\hat{\Psi})^{-1}\beta\\
\gs_{I,OOS}^2\ =\ &  4\frac{T}{T_{OOS}} \Big(\widehat\cB(z)+\ \gs_\eps^2 \widehat\cL \Big)
\end{aligned}
\end{equation}
\end{corollary}

\begin{proof}[Proof of Corollary \ref{lem:MSE-CLT-ridge}] By Proposition \ref{dec}, 
\begin{equation}\label{eq:risk3_expanded-main-a-1}
\begin{aligned}
&MSE_{OOS}(\hat f)-\widehat\cB(z)\ =\ E_{OOS}[\eps^2]\ -\ \widehat\cI(z)\ +\ \widehat\cV(z)\,,
\end{aligned}
\end{equation}
where 
\begin{equation}\label{eq:risk3_expanded1-main-a-1}
\begin{aligned}
&\widehat\cB(z)\ =\ {z^2\,\beta'(zI+\hat{\Psi})^{-1}\hat\Psi_{OOS} (zI+\hat{\Psi})^{-1}\beta}\ \ge\ 0\\
&\widehat\cV(z)\ =\ {\frac{1}{T^2}\,\eps' S (zI+\hat{\Psi})^{-1}\hat\Psi_{OOS} (zI+\hat{\Psi})^{-1} S'\eps}\ \ge\ 0,
\end{aligned}
\end{equation}
while 
\begin{equation}\label{eq:risk3_expanded2-main-a}
\widehat\cI(z)\ =\ {\frac{2z}{T}\,\beta'(zI+\hat{\Psi})^{-1}\hat\Psi_{OOS} (zI+\hat{\Psi})^{-1} S'\eps+2 E_{OOS}[\eps'S(\beta-\hat\beta)]}
\end{equation}
By Lemma \ref{ext-sriniv-multivar}, these three terms satisfy 
\begin{equation}
\begin{aligned}
T_{OOS}^{1/2}(E_{OOS}[\eps^2]-\gs_\eps^2)\ \sim\ & \cN(0,2\gs_\eps^4)\\
\frac{T^{1/2}(\widehat\cV(z)-\gs_\eps^2\widehat\cL)}{\gs_V\gs_\eps^2}\ \sim\ & \cN(0,1)
\end{aligned}
\end{equation}
where 
\begin{equation}
\begin{aligned}
\gs_V^2\ =\  &2\frac{1}{T^3}\tr\Bigr((S(zI+\hat{\Psi})^{-1}\hat\Psi_{OOS} (zI+\hat{\Psi})^{-1}S')^2\Bigr)\\
=\ & 2\frac{1}{T^3}\tr\Bigr(S(zI+\hat{\Psi})^{-1}\hat\Psi_{OOS} (zI+\hat{\Psi})^{-1}S'S(zI+\hat{\Psi})^{-1}\hat\Psi_{OOS} (zI+\hat{\Psi})^{-1}S'\Bigr)\\
=\ & 2\frac{1}{T^3}\tr\Bigr((zI+\hat{\Psi})^{-1}\hat\Psi_{OOS} (zI+\hat{\Psi})^{-1}S'S(zI+\hat{\Psi})^{-1}\hat\Psi_{OOS} (zI+\hat{\Psi})^{-1}S'S\Bigr)\\
=\ & 2\frac{1}{T}\tr\Bigr((zI+\hat{\Psi})^{-1}\hat\Psi_{OOS} (zI+\hat{\Psi})^{-1}\hat\Psi(zI+\hat{\Psi})^{-1}\hat\Psi_{OOS} (zI+\hat{\Psi})^{-1}\hat\Psi\Bigr)\,, 
\end{aligned}
\end{equation}
and 
\begin{equation}
\frac{T^{1/2}\widehat\cI(z)}{(\gs_\eps^2\gs_I^2+\gs_\eps^2\gs_{I,OOS}^2)^{1/2}}\ \to\ \cN(0,1),
\end{equation}
where 
\begin{equation}
\begin{aligned}
&\gs_I^2\ =\ 4z^2\beta'(zI+\hat{\Psi})^{-1}\hat\Psi_{OOS} (zI+\hat{\Psi})^{-1}\hat\Psi (zI+\hat{\Psi})^{-1}\hat\Psi_{OOS} (zI+\hat{\Psi})^{-1}\beta
\end{aligned}
\end{equation}
and 
\begin{equation}
\begin{aligned}
&\gs_{I,OOS}^2\\
=\  &   4 \frac{T}{T_{OOS}}  (\beta-\hat\beta)'\hat\Psi_{OOS} (\beta-\hat\beta)\\
=\ & 4\frac{T}{T_{OOS}} \Bigr(z^2\beta'(zI+\hat{\Psi})^{-1}\hat\Psi_{OOS} (zI+\hat{\Psi})^{-1}\beta +\ \gs_\eps^2 \frac{1}{T}\tr((zI+\hat{\Psi})^{-1}\hat\Psi_{OOS} (zI+\hat{\Psi})^{-1}\hat\Psi)\Bigr)\ +\ O(T^{-1/2})\\
=\ & 4 \frac{T}{T_{OOS}} \Big(\widehat\cB(z)+\ \gs_\eps^2 \widehat\cL \Big)\ +\ O(T^{-1/2})\
\end{aligned}
\end{equation}
where we have used the decomposition 
\begin{equation}\label{noise-dec21}
\beta\ -\hat\beta\ =\ \ \underbrace{z\,(zI+\hat{\Psi})^{-1}\beta}_{bias}\ -\ \underbrace{(zI+\hat{\Psi})^{-1}\frac{1}{T}S'\eps}_{noise}\,. 
\end{equation}
Furthermore, the three Normals are asymptotically independent, conditional on $S.$ We have 
\begin{equation}
\begin{aligned}
&T^{1/2} \Bigr(\frac{MSE_{OOS}(\hat f)-\hat B(z)}{1+\widehat\cL}-\gs_\eps^2 \Bigr) =\ T^{1/2}\frac{MSE_{OOS}(\hat f)-\hat B(z)-\gs_\eps^2-\widehat\cL \gs_\eps^2}{1+\widehat\cL}\,.
\end{aligned}
\end{equation}
The claim follows from the continuous mapping theorem. 
\end{proof}

\section{Proof of Theorem \ref{main-th-main-text}} \label{theBig}

\begin{theorem}[Probabilistic Lower Bound for $R^2_*$]\label{main-th-r} Suppose that $E[\eps_t^3]=0,\ E[\eps_t^4]=3.$ Let 
\begin{equation}\label{MSE-OOS-R2-1}
\begin{aligned}
&R^2_{OOS}(\hat f)\ =\ 1\ -\ \frac{MSE_{OOS}(\hat f)}{E_{OOS}[y^2]}
\end{aligned}
\end{equation}
be the realized OOS $R^2$. Then, 
\begin{equation}\label{llg3}
\frac{R^2_{OOS}(\hat f)+\widehat\cL(z)}{1+\widehat\cL(z)}\,,
\end{equation}
is a $T^{1/2}$-consistent upper bound for $\tilde R^2_*$ in the following sense: The event $\frac{R^2_{OOS}(\hat f)+\widehat\cL(z)}{1+\widehat\cL(z)}>\tilde R_*^2$ occurs with vanishing probability: 
\begin{equation}
\lim\sup_{T,T_{OOS}\to\infty} \Prob\left(\frac{T_{OOS}^{1/2}\left(\tilde R_*^2-\dfrac{R^2_{OOS}(\hat f)+\widehat\cL(z)}{1+\widehat\cL(z)}\right)}{\hat\gs_{R^2}}<\ga\right)\ \le\ \Phi(\ga)
\end{equation}
for any $\ga\le 0,$ where $\Phi(\cdot)$ is the c.d.f. of the standard normal distribution. Here,
\begin{equation}
\begin{aligned}
&\hat\gs_{R^2}\ =\ \frac{\dfrac{1}{MSE(0)^2} \gS_{1,1}+\Bigr(\dfrac{\widehat{MSE}}{MSE(0)^2}\Bigr)^2\gS_{2,2} -2\dfrac{\widehat{MSE}}{MSE(0)^3} \gS_{1,2}}{(1+\widehat\cL)^2}\\
&\gS_{1,1}\ =\ \frac{T_{OOS}}{T}\gs^2_{MSE}(1+\widehat\cL)^2\\
&\gS_{1,2}\ =\ 2\gs_\eps^4+4\gs_\eps^2 E_{OOS}[f (f-\hat f)]\\
&\gS_{2,2}\ =\ 2\gs_\eps^4+4\gs_\eps^2 E_{OOS}[f^2]\\
&\widehat{MSE}\ =\ \widehat\cB\ +\ \gs_\eps^2(1+\widehat\cL)\\
&\gs_{MSE}^2\ =\ \frac{2\dfrac{T}{T_{OOS}}\gs_\eps^4+\gs_\eps^4\gs_V^2+\gs_\eps^2 \gs_I^2+\gs_\eps^2\gs_{I,OOS}^2}{(1+\widehat\cL)^2}\\
\gs_V^2\ =\ &2TT_{OOS}^{-2}\tr ((\widehat\cK'\widehat\cK)^2)\\
\gs_I^2\ =\ & 4T T_{OOS}^{-2}\|(f-\hat f^s)\widehat\cK\|^2\\
\gs_{I,OOS}^2\ =\ &  4\frac{T}{T_{OOS}} \Big(\widehat\cB+\ \gs_\eps^2 \widehat\cL \Big)
\end{aligned}
\end{equation}
\end{theorem}

\begin{proof}[Proof of Theorem \ref{main-th-r-ridge}] We have 
\begin{equation}
R^2_{OOS}\ =\ 1\ -\ \frac{MSE_{OOS}(\hat f)}{MSE_{OOS}(0)}\,.
\end{equation}
As above, we suppose that $E[\eps_t^3]=0,\ E[\eps_t^4]=3.$ We have 
\begin{equation}
\begin{aligned}
E_{OOS}[y^2]\ =\ & E_{OOS}[\eps^2]\ +\ 2E_{OOS}[\eps f]\ +\ E_{OOS}[f^2]
\end{aligned}
\end{equation}
Now, asymptotic normality and asymptotic independence of all the terms follow by the same argument as in the proof of Lemma \ref{ext-sriniv}. All we need to do is compute $\Sigma_{OOS}^2.$ We have 
\begin{equation}
T_{OOS}^{1/2}(E_{OOS}[\eps^2]-\gs_\eps^2)\ \to\ \cN(0, 2\gs_\eps^4)
\end{equation}
and 
\begin{equation}
\frac{T_{OOS}^{1/2} 2E_{OOS}[\eps f]}{(4 T_{OOS}^{-1}\|f\|^2)^{1/2}}\ \to\ \cN(0,1)\,,
\end{equation}
Thus, 
\begin{equation}
\binom{MSE_{OOS}(\hat f)}{E_{OOS}[y^2]}\ =\ \binom{\widehat{MSE}}{MSE(0)}\ +\ \frac{1}{T_{OOS}} \cN(0,\Sigma_{Joint})
\end{equation}
where 
\begin{equation}
\Sigma_{Joint}\ =\ \begin{pmatrix}
\gS_{1,1}&\gS_{1,2}\\
\gS_{1,2}&\gS_{2,2}
\end{pmatrix}
\end{equation}
where
\begin{equation}
\begin{aligned}
&\gS_{1,1}\ =\ \frac{T_{OOS}}{T}\gs^2_{MSE}(1+\widehat\cL)^2\\
&\gS_{1,2}\ =\ 2\gs_\eps^4+4\gs_\eps^2 E_{OOS}[f (f-\hat f)]\\
&\gS_{2,2}\ =\ 2\gs_\eps^4+4\gs_\eps^2 E_{OOS}[f^2]\,. 
\end{aligned}
\end{equation}
Here, 
\begin{equation}
\begin{aligned}
&E_{OOS}[f(f-\hat f)]\ =\ E_{OOS}[f^2]- E_{OOS}[f\hat f]
\end{aligned}
\end{equation}
and $E_{OOS}[f\hat f]$ admits a pivotal estimator 
\begin{equation}\label{betahatbeta}
E_{OOS}[d \hat f]\ =\ \frac{1}{T_{OOS}}(f_{OOS}+\eps_{OOS})'\hat f_{OOS}\ \approx\ E_{OOS}[f\hat f]\,.  
\end{equation}
Thus, 
\begin{equation}
\begin{aligned}
\frac{MSE_{OOS}(\hat f)}{E_{OOS}[y^2]}
\approx\ & \frac{ \dfrac{\widehat{MSE}}{MSE(0)}+T_{OOS}^{-1/2}\dfrac{Error_1}{MSE(0)} }{1+T_{OOS}^{-1/2} \dfrac{Error_2}{MSE(0)}}\\
\approx\ & \Bigr(\dfrac{\widehat{MSE}}{MSE(0)}+T_{OOS}^{-1/2}\dfrac{Error_1}{MSE(0)}\Bigr)\Bigr(1-T_{OOS}^{-1/2}\dfrac{Error_2}{MSE(0)}\Bigr)\\
\approx\ & \dfrac{\widehat{MSE}}{MSE(0)} +\ T_{OOS}^{-1/2}\Bigr(\dfrac{Error_1}{MSE(0)}- \frac{\widehat{MSE}}{MSE(0)^2}Error_2 \Bigr)
\end{aligned}
\end{equation}
Thus, 
\begin{equation}
\begin{aligned}
T_{OOS}\Var\Bigr[\frac{MSE_{OOS}}{E_{OOS}[y^2]}\Bigr]
=\ & \frac{1}{MSE(0)^2} \gS_{1,1}+\Bigr(\frac{\widehat{MSE}}{MSE(0)^2}\Bigr)^2\gS_{2,2} -2\frac{\widehat{MSE}}{MSE(0)^3} \gS_{1,2}\,.
\end{aligned}
\end{equation}
Thus, we get 
\begin{equation}
\begin{aligned}
\frac{R^2_{OOS}+\widehat\cL}{1+\widehat\cL}
=\ & \frac{1-\dfrac{MSE_{OOS}}{E_{OOS}[y^2]}+\widehat\cL}{1+\widehat\cL}\\
=\ & \frac{1-\dfrac{\widehat{MSE}}{MSE(0)} -\ T_{OOS}^{-1/2}\Bigr(\dfrac{Error_1}{MSE(0)}- \dfrac{\widehat{MSE}}{MSE(0)^2}Error_2 \Bigr)+\widehat\cL}{1+\widehat\cL}\\
\le\ & \frac{1- \dfrac{\widehat{MSE}-\widehat\cB}{MSE(0)} -\ T_{OOS}^{-1/2} \Bigr(\dfrac{Error_1}{MSE(0)}- \dfrac{\widehat{MSE}}{MSE(0)^2}Error_2 \Bigr)+\widehat\cL}{1+\widehat\cL}\\
=\ & \frac{\dfrac{E_{OOS}[f_t^2]-\gs_\eps^2\widehat\cL}{E_{OOS}[f_t^2]+\gs_\eps^2}+\widehat\cL}{1+\widehat\cL}-\frac{T_{OOS}^{-1/2} \Bigr(\dfrac{Error_1}{MSE(0)}- \dfrac{\widehat{MSE}}{MSE(0)^2}Error_2 \Bigr)}{1+\widehat\cL}\\
=\ & R^2_*\ -\ \frac{T_{OOS}^{-1/2} \Bigr(\dfrac{Error_1}{MSE(0)}- \dfrac{\widehat{MSE}}{MSE(0)^2}Error_2 \Bigr)}{1+\widehat\cL}
\end{aligned}
\end{equation}
The proof of Theorem \ref{main-th-main-text} is complete.
\end{proof}

For the reader's convenience, we state the special case of the ridge regression as a separate proposition.

\begin{corollary}[Probabilistic Lower Bound for $R^2_*$ for a Ridge Regression]\label{main-th-r-ridge} Suppose that $E[\eps_t^3]=0,\ E[\eps_t^4]=3.$ Then, 
\begin{equation}\label{llg3-1}
\frac{R^2_{OOS}(\hat f)+\widehat\cL(z)}{1+\widehat\cL(z)}\,,
\end{equation}
is a $T^{1/2}$-consistent upper bound for $\tilde R^2_*$ in the following sense: The event $\frac{R^2_{OOS}(\hat f)+\widehat\cL(z)}{1+\widehat\cL(z)}>\tilde R_*^2$ occurs with vanishing probability: 
\begin{equation}
\lim\sup_{T,T_{OOS}\to\infty} \Prob\left(\frac{T_{OOS}^{1/2}\left(\tilde R_*^2-\dfrac{R^2_{OOS}(\hat f)+\widehat\cL(z)}{1+\widehat\cL(z)}\right)}{\hat\gs_{R^2}}<\ga\right)\ \le\ \Phi(\ga)
\end{equation}
for any $\ga\le 0,$ where $\Phi(\cdot)$ is the c.d.f. of the standard normal distribution. Here,
\begin{equation}
\begin{aligned}
&\hat\gs_{R^2}\ =\ \frac{\dfrac{1}{MSE(0)^2} \gS_{1,1}+\Bigr(\dfrac{\widehat{MSE}}{MSE(0)^2}\Bigr)^2\gS_{2,2} -2\dfrac{\widehat{MSE}}{MSE(0)^3} \gS_{1,2}}{(1+\widehat\cL)^2}\\
&\gS_{1,1}\ =\ \frac{T}{T_{OOS}}\gs^2_{MSE}(1+\widehat\cL)^2\\
&\gS_{1,2}\ =\ 2\gs_\eps^4+4\gs_\eps^2 \beta'\hat\Psi_{OOS}(\beta-\hat\beta)\\
&\gS_{2,2}\ =\ 2\gs_\eps^4+4\gs_\eps^2 \beta'\Psi_{OOS}\beta\\
&\widehat{MSE}\ =\ \widehat\cB\ +\ \gs_\eps^2(1+\widehat\cL)\\
&\gs_{MSE}^2\ =\ \frac{2\dfrac{T}{T_{OOS}}\gs_\eps^4+\gs_\eps^4\gs_V^2+\gs_\eps^2 \gs_I^2+\gs_\eps^2\gs_{I,OOS}^2}{(1+\widehat\cL)^2}\\
&\gs_V^2\ =\ 2 \frac{1}{T}\tr\Bigr((zI+\hat{\Psi})^{-1}\hat\Psi_{OOS} (zI+\hat{\Psi})^{-1}\hat\Psi(zI+\hat{\Psi})^{-1}\hat\Psi_{OOS} (zI+\hat{\Psi})^{-1}\hat\Psi\Bigr)\\
&\gs_I^2\ =\ 4z^2\beta'(zI+\hat{\Psi})^{-1}\hat\Psi_{OOS} (zI+\hat{\Psi})^{-1}\hat\Psi (zI+\hat{\Psi})^{-1}\hat\Psi_{OOS} (zI+\hat{\Psi})^{-1}\beta\\
&\gs_{I,OOS}^2\ =\  4\frac{T}{T_{OOS}} \Bigr(\widehat\cB(z)+\ \gs_\eps^2 \widehat\cL \Bigr)
\end{aligned}
\end{equation}

\end{corollary}

\begin{proof}[Proof of Corollary \ref{main-th-r-ridge}] We have 
\begin{equation}
R^2_{OOS}\ =\ 1\ -\ \frac{MSE_{OOS}(\hat f)}{MSE_{OOS}(0)}\,.
\end{equation}
As above, we suppose that $E[\eps_t^3]=0,\ E[\eps_t^4]=3.$ We have 
\begin{equation}
\begin{aligned}
E_{OOS}[y^2]\ =\ & E_{OOS}[\eps^2]\ +\ 2E_{OOS}[\eps'S\beta]\ +\ \beta'\hat\Psi_{OOS}\beta\\
MSE_{OOS}(\hat f)\ =\ & E_{OOS}[\eps^2]\ +\ 2E_{OOS}[\eps'S(\beta-\hat\beta)]\\
&+\ \widehat\cV(z)\ +\ \widehat\cB(z)\ +\ \frac{2z}{T}\,\beta'(zI+\hat{\Psi})^{-1}\hat\Psi_{OOS} (zI+\hat{\Psi})^{-1} S'\eps\,. 
\end{aligned}
\end{equation}
Now, asymptotic normality and asymptotic independence of all the terms follow by the same argument as in the proof of Lemma \ref{ext-sriniv}. All we need to do is compute $\Sigma_{OOS}^2.$ We have 
\begin{equation}
T_{OOS}^{1/2}(E_{OOS}[\eps^2]-\gs_\eps^2)\ \to\ \cN(0, 2\gs_\eps^4)
\end{equation}
and 
\begin{equation}
\frac{T_{OOS}^{1/2} 2E_{OOS}[\eps'S\beta]}{(4\beta'\hat\Psi_{OOS}\beta)^{1/2}}\ \to\ \cN(0,1)\,,
\end{equation}
Thus, 
\begin{equation}
\binom{MSE_{OOS}(\hat f)}{E_{OOS}[y^2]}\ =\ \binom{\widehat{MSE}}{MSE(0)}\ +\ \frac{1}{T_{OOS}} \cN(0,\Sigma_{Joint})
\end{equation}
where 
\begin{equation}
\Sigma_{Joint}\ =\ \begin{pmatrix}
\gS_{1,1}&\gS_{1,2}\\
\gS_{1,2}&\gS_{2,2}
\end{pmatrix}
\end{equation}
where
\begin{equation}
\begin{aligned}
&\gS_{1,1}\ =\ \frac{T_{OOS}}{T}\gs^2_{MSE}(1+\widehat\cL)^2\\
&\gS_{1,2}\ =\ 2\gs_\eps^4+4\gs_\eps^2 \beta'\hat\Psi_{OOS}(\beta-\hat\beta)\\
&\gS_{2,2}\ =\ 2\gs_\eps^4+4\gs_\eps^2 \beta'\Psi_{OOS}\beta\,. 
\end{aligned}
\end{equation}
Here, 
\begin{equation}
\begin{aligned}
&\beta'\hat\Psi_{OOS}(\beta-\hat\beta)\ \approx\ \beta'\hat\Psi_{OOS}\beta -\beta'\hat\Psi_{OOS}\hat\beta
\end{aligned}
\end{equation}
and $\beta'\hat\Psi_{OOS}\hat\beta$ admits a pivotal estimator 
\begin{equation}\label{betahatbeta1}
E_{OOS}[y'S]\hat\beta\ =\ \frac{1}{T_{OOS}}(S_{OOS}\beta+\eps_{OOS})'S_{OOS}\hat\beta\ \approx\ \beta'\Psi \hat\beta\,. 
\end{equation}
Thus, 
\begin{equation}
\begin{aligned}
\frac{MSE_{OOS}}{E_{OOS}[y^2]}
\approx\ & \frac{ \dfrac{\widehat{MSE}}{MSE(0)}+T_{OOS}^{-1/2}\dfrac{Error_1}{MSE(0)} }{1+T_{OOS}^{-1/2} \dfrac{Error_2}{MSE(0)}}\\
\approx\ & \Bigr(\dfrac{\widehat{MSE}}{MSE(0)}+T_{OOS}^{-1/2}\dfrac{Error_1}{MSE(0)}\Bigr)\Bigr(1-T_{OOS}^{-1/2}\dfrac{Error_2}{MSE(0)}\Bigr)\\
\approx\ & \dfrac{\widehat{MSE}}{MSE(0)} +\ T_{OOS}^{-1/2}\Bigr(\dfrac{Error_1}{MSE(0)}- \frac{\widehat{MSE}}{MSE(0)^2}Error_2 \Bigr)
\end{aligned}
\end{equation}
Thus, 
\begin{equation}
\begin{aligned}
T_{OOS}\Var\Bigr[\frac{MSE_{OOS}}{E_{OOS}[y^2]}\Bigr]
=\ & \frac{1}{MSE(0)^2} \gS_{1,1}+\Bigr(\frac{\widehat{MSE}}{MSE(0)^2}\Bigr)^2\gS_{2,2} -2\frac{\widehat{MSE}}{MSE(0)^3} \gS_{1,2}\,.
\end{aligned}
\end{equation}
Thus, we get 
\begin{equation}
\begin{aligned}
\frac{R^2_{OOS}+\widehat\cL}{1+\widehat\cL}
=\ & \frac{1-\dfrac{MSE_{OOS}}{E_{OOS}[y^2]}+\widehat\cL}{1+\widehat\cL}\\
=\ & \frac{1-\dfrac{\widehat{MSE}}{MSE(0)} -\ T_{OOS}^{-1/2}\Bigr(\dfrac{Error_1}{MSE(0)}- \dfrac{\widehat{MSE}}{MSE(0)^2}Error_2 \Bigr)+\widehat\cL}{1+\widehat\cL}\\
\le\ & \frac{1- \dfrac{\widehat{MSE}-\widehat\cB}{MSE(0)} -\ T_{OOS}^{-1/2} \Bigr(\dfrac{Error_1}{MSE(0)}- \dfrac{\widehat{MSE}}{MSE(0)^2}Error_2 \Bigr)+\widehat\cL}{1+\widehat\cL}\\
=\ & \frac{\dfrac{\beta'\Psi_{OOS}\beta-\gs_\eps^2\widehat\cL}{\beta'\Psi_{OOS}\beta+\gs_\eps^2}+\widehat\cL}{1+\widehat\cL}-\frac{T_{OOS}^{-1/2} \Bigr(\dfrac{Error_1}{MSE(0)}- \dfrac{\widehat{MSE}}{MSE(0)^2}Error_2 \Bigr)}{1+\widehat\cL}\\
=\ & R^2_*\ -\ \frac{T_{OOS}^{-1/2} \Bigr(\dfrac{Error_1}{MSE(0)}- \dfrac{\widehat{MSE}}{MSE(0)^2}Error_2 \Bigr)}{1+\widehat\cL}
\end{aligned}
\end{equation}
The proof of Theorem \ref{main-th-main-text} is complete.
\end{proof}

\subsection{Pivotal Bounds for the Variance}

\begin{proposition}\label{super-short-estim}
Let 
\begin{equation}
\begin{aligned}
&\hat q_{OOS}\ =\ \widehat\cK' (\widehat\cK y-y_{OOS})\,.
\end{aligned}
\end{equation}
Suppose that 
\begin{equation}\label{ass-vs}
TT_{OOS}^{-2}E[(f_{OOS}-\hat f^s)'\widehat\cK \widehat\cK'\bigl(\widehat\cK\widehat\cK' + I\bigr)\widehat\cK \widehat\cK' (f_{OOS}-\hat f^s)]\ \to\ 0
\end{equation}
as $T,T_{OOS}\to\infty.$ Then, 
\begin{equation}
\begin{aligned}
\frac14 \gs_I^2 =\  T T_{OOS}^{-2}\|(f-\hat f^s)\widehat\cK\|^2 \approx\  T T_{OOS}^{-2}\|\hat q_{OOS}\|^2\ -\ \gs_\eps^2\Big(\frac12 \gs_V^2\ +\ \frac{T}{T_{OOS}}\widehat\cL\Big)\,. 
\end{aligned}
\end{equation}
Thus, 
\begin{equation}
\begin{aligned}
\gs_{MSE}^2\ = \frac{2\dfrac{T}{T_{OOS}}\gs_\eps^4-\gs_\eps^4\gs_V^2+4\gs_\eps^2 T T_{OOS}^{-2}\|\hat q_{OOS}\|^2\ + 4 \gs_\eps^2 \dfrac{T}{T_{OOS}} \widehat\cB }{(1+\widehat\cL)^2}
\end{aligned}
\end{equation}
\end{proposition}

\begin{proof}[Proof of Proposition \ref{super-short-estim}] 
Recall that
\begin{equation}
y = f + \varepsilon, 
\qquad
y_{OOS} = f_{OOS} + \varepsilon_{OOS}.
\end{equation}
Then, 
\begin{equation}
    \begin{aligned}
       \widehat\cK y - y_{OOS}
&= \widehat\cK(f+\varepsilon) - (f_{OOS} + \varepsilon_{OOS}) \\
&= (\widehat\cK f - f_{OOS}) + (\widehat\cK \eps - \eps_{OOS}) \\
&= (\widehat\cK f - f_{OOS}) + \hat\eps_{OOS},  
    \end{aligned}
\end{equation}
and we have defined
\begin{equation}
\hat\eps_{OOS} \equiv \widehat\cK\eps - \eps_{OOS}.    
\end{equation}
By definition of $\hat q_{OOS}$,
\begin{equation}\label{eq:qOOS_decomp}
\begin{aligned}
\hat q_{OOS}
&= \widehat\cK'(\widehat\cK y - y_{OOS}) \\
&= \widehat\cK'\bigl[(\widehat\cK f - f_{OOS}) + \hat\eps_{OOS}\bigr] \\
&= \widehat\cK'(\widehat\cK f - f_{OOS}) + \widehat\cK'\hat\eps_{OOS} \\
&= -\,\widehat\cK'(f_{OOS} - \widehat\cK f) + \tilde\eps_{OOS} \\
&= -\,\widehat\cK'(f_{OOS} - \hat f^s) + \tilde\eps_{OOS},
\end{aligned}     
\end{equation}
where we have defined
\begin{equation}
    \tilde\eps_{OOS} \equiv \widehat\cK'\hat\eps_{OOS}.
\end{equation}
Taking quadratic forms in \eqref{eq:qOOS_decomp} yields
\begin{equation}
\begin{aligned}\label{eq:qOOS_quad}
 T T_{OOS}^{-2} \|\hat q_{OOS}\|^2
&=  T T_{OOS}^{-2} 
\Big\|
-\,\widehat\cK'(f_{OOS} - \hat f^s) + \tilde\eps_{OOS}
\Big\|^2 \\
&=  T T_{OOS}^{-2}
\Big\|
\widehat\cK'(f_{OOS}-\hat f^s)
\Big\|^2
+  T T_{OOS}^{-2} \|\tilde\eps_{OOS}\|^2
-  2 T T_{OOS}^{-2}\langle \widehat\cK'(f_{OOS}-\hat f^s),\,\tilde\eps_{OOS}\rangle.
\end{aligned}
\end{equation}
We first show that the cross term in \eqref{eq:qOOS_quad} is asymptotically negligible. Using
$\tilde\eps_{OOS}=\widehat\cK'\hat\eps_{OOS}$ and the fact that
$E[\hat\eps_{OOS}\mid S]=0$, we have
\begin{equation}
E\!\left[
\big\langle \widehat\cK'(f_{OOS}-\hat f^s),\,\tilde\eps_{OOS}\big\rangle
\;\middle|\; S
\right]
=
\big\langle \widehat\cK'(f_{OOS}-\hat f^s),\,E[\tilde\eps_{OOS}\mid S]\big\rangle
= 0.
\end{equation}
Therefore, by conditional variance,
\begin{equation}
E\!\left[
\big(
2 T T_{OOS}^{-2}\,
\langle \widehat\cK'(f_{OOS}-\hat f^s),\,\tilde\eps_{OOS}\rangle
\big)^2
\,\middle|\, S
\right]
=
4 T^2 T_{OOS}^{-4}\,
\hat v(S),    
\end{equation}
where
\begin{equation}
\hat v(S)
:=
\big\langle 
\widehat\cK'(f_{OOS}-\hat f^s),\,
E[\tilde\eps_{OOS}\tilde\eps_{OOS}'\mid S]\,
\widehat\cK'(f_{OOS}-\hat f^s)
\big\rangle.
\end{equation}
Since
\begin{equation}
E[\hat\eps_{OOS}\hat\eps_{OOS}'\mid S]
=
\sigma_\eps^2\bigl(\widehat\cK\widehat\cK' + I\bigr),
\qquad
\tilde\eps_{OOS}
=\widehat\cK'\hat\eps_{OOS},
\end{equation}
we obtain
\begin{equation}
E[\tilde\varepsilon_{OOS}\tilde\eps_{OOS}'\mid S]
=
\sigma_\eps^2\,
\widehat\cK'\bigl(\widehat\cK\widehat\cK' + I\bigr)\widehat\cK.
\end{equation}
Hence
\begin{equation}
\begin{aligned}
\hat v(S)\ =\ \gs_\eps^2 (f_{OOS}-\hat f^s)'\widehat\cK \widehat\cK'\bigl(\widehat\cK\widehat\cK' + I\bigr)\widehat\cK \widehat\cK' (f_{OOS}-\hat f^s)
\end{aligned}
\end{equation}
By \eqref{ass-vs}, the cross term in \eqref{eq:qOOS_quad} converges to zero in $L_2$. Now, we have 
\begin{equation}\label{eq:qOOS_quad_approx}
 T T_{OOS}^{-2} \|\hat q_{OOS}\|^2 \approx T T_{OOS}^{-2}
\Big\|
\widehat\cK'(f_{OOS}-\hat f^s)
\Big\|^2
+  T T_{OOS}^{-2} \|\tilde\eps_{OOS}\|^2.   
\end{equation}
By the concentration of quadratic forms (Lemma \ref{lem-quad}), 
\begin{equation}
\begin{aligned}
 T T_{OOS}^{-2}\,\|\tilde\eps_{OOS}\|^2
=\ &  T T_{OOS}^{-2}\,\hat\eps_{OOS}'\,\widehat\cK\widehat\cK'\,\hat\eps_{OOS}\\
=\ &  T T_{OOS}^{-2}\,\tr\Bigl[\widehat\cK\widehat\cK'\,\hat\eps_{OOS}\hat\eps_{OOS}'\Bigr]\\
\approx\ &  T T_{OOS}^{-2}\,\tr\Bigl[\widehat\cK\widehat\cK'\,
E\bigl[\hat\eps_{OOS}\hat\eps_{OOS}'\mid S\bigr]\Bigr]\\
=\ & T T_{OOS}^{-2}\,\tr\Bigl[\widehat\cK\widehat\cK'\,
\gs_\eps^2\bigl(\widehat\cK\widehat\cK' + I\bigr)\Bigr]\\
=\ & \gs_\eps^2\,\underbrace{T T_{OOS}^{-2}\,\tr\Bigl[(\widehat\cK\widehat\cK')^2\Bigr]}_{=\frac12\,\gs_V^2\ by\  \eqref{defn:sigmav_sim}}
\;+\; \gs_\eps^2\,T T_{OOS}^{-2}\,\tr\Bigl[\widehat\cK\widehat\cK'\Bigr]\\
=\ & \gs_\eps^2\,\frac12\,\gs_V^2
\;+\; \gs_\eps^2\,\frac{T}{T_{OOS}}\,\widehat\cL\\
=\ & \gs_\eps^2\Bigl(\frac12\,\gs_V^2\;+\;\frac{T}{T_{OOS}}\,\widehat\cL\Bigr)\,.
\end{aligned}
\end{equation}
Thus, by \eqref{defn:sigmav_sim} and \eqref{eq:qOOS_quad_approx}, 
\begin{equation}
\begin{aligned}
\frac14 \gs_I^2 =\  T T_{OOS}^{-2}\|(f_{OOS}-\hat f^s)\widehat\cK\|^2 \approx\  T T_{OOS}^{-2}\|\hat q_{OOS}\|^2\ -\ \gs_\eps^2\Big(\frac12 \gs_V^2\ +\ \frac{T}{T_{OOS}}\widehat\cL\Big)\,. 
\end{aligned}
\end{equation}
Plugging this expression into \eqref{eq:sigma2_MSE}, we have
\begin{equation}
\begin{aligned}
\gs_{MSE}^2\ =\ & \frac{2\dfrac{T}{T_{OOS}}\gs_\eps^4+\gs_\eps^4\gs_V^2+\gs_\eps^2 \gs_I^2+\gs_\eps^2\gs_{I,OOS}^2}{(1+\widehat\cL)^2}\\
=\ & \frac{2\dfrac{T}{T_{OOS}}\gs_\eps^4+\gs_\eps^4\gs_V^2+4\gs_\eps^2 \Bigr[ T T_{OOS}^{-2}\|\hat q_{OOS}\|^2\ -\ \gs_\eps^2\Big(\dfrac12 \gs_V^2\ +\ \dfrac{T}{T_{OOS}}\widehat\cL\Big)\Bigr]+4\gs_\eps^2\dfrac{T}{T_{OOS}} \Big(\widehat\cB+\ \gs_\eps^2 \widehat\cL \Big)}{(1+\widehat\cL)^2}\\
=\ & \frac{2\dfrac{T}{T_{OOS}}\gs_\eps^4-\gs_\eps^4\gs_V^2+4\gs_\eps^2 T T_{OOS}^{-2}\|\hat q_{OOS}\|^2\ + 4 \gs_\eps^2 \dfrac{T}{T_{OOS}} \widehat\cB }{(1+\widehat\cL)^2}
\end{aligned}
\end{equation}
The proof of Proposition \ref{super-short-estim} is complete. 
\end{proof}

For the reader's convenience, we state the special case of the ridge regression as a separate proposition.

\begin{proposition}\label{super-short-estim-ridge}
Let 
\begin{equation}
\begin{aligned}
&\hat q_{OOS}\ =\ S(zI+\hat\Psi)^{-1}\frac{1}{T_{OOS}}S_{OOS}' \Bigr(S_{OOS}(zI+\hat\Psi)^{-1}\frac{1}{T}S'y-y_{OOS} \Bigr)\,.
\end{aligned}
\end{equation}
Then, 
\begin{equation}
\begin{aligned}
\frac14 \gs_I^2 =\ & z^2 \beta ' (zI + \hat \Psi)^{-1} \Psi_{OOS} (zI + \hat \Psi)^{-1} \hat\Psi (zI + \hat \Psi)^{-1}\Psi_{OOS}(zI + \hat \Psi)^{-1}\beta\\
\approx\ &  \frac{1}{T}\|\hat q_{OOS}\|^2\ -\ \gs_\eps^2\Big(\frac12 \gs_V^2\ +\ \frac{T}{T_{OOS}}\widehat\cL\Big)\,. 
\end{aligned}
\end{equation}
Thus, 
\begin{equation}
\begin{aligned}
\gs_{MSE}^2\ =\ \frac{2\dfrac{T}{T_{OOS}}\gs_\eps^4-\gs_\eps^4\gs_V^2+4\gs_\eps^2 \dfrac{1}{T}\|\hat q_{OOS}\|^2\ + 4 \gs_\eps^2 \dfrac{T}{T_{OOS}} \widehat\cB(z) }{(1+\widehat\cL)^2}
\end{aligned}
\end{equation}
\end{proposition}

\begin{proof}[Proof of Proposition \ref{super-short-estim-ridge}]
Now, 
\begin{equation}
\begin{aligned}
&S_{OOS}(zI+\hat\Psi)^{-1} \frac{1}{T}S'y-y_{OOS}\\
=\ & S_{OOS}(zI+\hat\Psi)^{-1} \frac{1}{T}S'(S\beta+\eps)-(S_{OOS}\beta+\eps_{OOS})\\
=\ & S_{OOS}(zI+\hat\Psi)^{-1}\hat\Psi\beta-S_{OOS}\beta +\ S_{OOS}(zI+\hat\Psi)^{-1} \frac{1}{T}S'\eps-\eps_{OOS}\\
=\ & -zS_{OOS}(zI+\hat\Psi)^{-1}\beta\ +\ \hat\eps_{OOS}
\end{aligned}
\end{equation}
where we have defined 
\begin{equation}
\hat\eps_{OOS}\ =\ S_{OOS}(zI+\hat\Psi)^{-1} \frac{1}{T}S'\eps-\eps_{OOS}
\end{equation}
and, hence, 
\begin{equation}
\begin{aligned}
\hat q_{OOS}\ =\ & S(zI+\hat\Psi)^{-1} \frac{1}{T_{OOS}} S_{OOS}' \Bigr(S_{OOS}(zI+\hat\Psi)^{-1} \frac{1}{T}S'y-y_{OOS} \Bigr)\\
=\ & S(zI+\hat\Psi)^{-1}\frac{1}{T_{OOS}}S_{OOS}' \Bigr(-zS_{OOS}(zI+\hat\Psi)^{-1}\beta\ +\ \hat\eps_{OOS} \Bigr)\\
=\ & -zS(zI+\hat\Psi)^{-1}\hat\Psi_{OOS}(zI+\hat\Psi)^{-1}\beta+\tilde\eps_{OOS},
\end{aligned}
\end{equation}
where 
\begin{equation}\label{eq: tild_eps_OOS}
\tilde\eps_{OOS}\ =\ S(zI+\hat\Psi)^{-1}\frac{1}{T_{OOS}}S_{OOS}'\hat\eps_{OOS},
\end{equation}
so that 
\begin{equation}\label{qoos-norm}
\begin{aligned}
& \frac{1}{T} \hat q_{OOS}'\hat q_{OOS}\\
=\ &  \frac{1}{T}  \Bigr(-zS(zI+\hat\Psi)^{-1}\hat\Psi_{OOS}(zI+\hat\Psi)^{-1}\beta+\tilde\eps_{OOS} \Bigr)' \Bigr(-zS(zI+\hat\Psi)^{-1}\hat\Psi_{OOS}(zI+\hat\Psi)^{-1}\beta+\tilde\eps_{OOS} \Bigr)\\
=\ & z^2 \beta ' (zI + \hat \Psi)^{-1} \hat\Psi_{OOS} (zI + \hat \Psi)^{-1} \hat\Psi (zI + \hat \Psi)^{-1} \hat\Psi_{OOS}(zI + \hat \Psi)^{-1}\beta +\ \frac{1}{T} \|\tilde\eps_{OOS}\|^2\ \\
& -\ \frac{2z}{T}\,\beta' (zI+\hat\Psi)^{-1}\hat\Psi_{OOS}(zI+\hat\Psi)^{-1} S'\tilde\eps_{OOS}\,. 
\end{aligned}
\end{equation}
We now show that the cross term in \eqref{qoos-norm}
\begin{equation}
-\frac{2z}{T}\,\beta' (zI+\hat\Psi)^{-1}\hat\Psi_{OOS}(zI+\hat\Psi)^{-1} S'\tilde\eps_{OOS}
\end{equation}
is asymptotically negligible. Recall that
\begin{equation}
\hat\Psi = \frac{1}{T}S'S,
\qquad
\hat\Psi_{OOS} = \frac{1}{T_{OOS}}S_{OOS}' S_{OOS}.
\end{equation}

Using the fact that $\tilde\eps_{OOS}\ =\ S(zI+\hat\Psi)^{-1}\dfrac{1}{T_{OOS}}S_{OOS}'\hat\eps_{OOS}$ and 
$E[\hat\eps_{OOS}\mid S,S_{OOS}]=0$, we have 
\begin{equation}
E\!\left[
-\frac{2z}{T}\,
\beta' (zI+\hat\Psi)^{-1}\hat\Psi_{OOS}(zI+\hat\Psi)^{-1} S'\tilde\eps_{OOS}\middle|\; S,S_{OOS}
\right]
= 0.
\end{equation}
Its conditional second moment is
\begin{equation}\label{eq:I^2_ridge}
\begin{aligned}
& E\Big[
\Big(-\frac{2z}{T}\,
\beta' (zI+\hat\Psi)^{-1}\hat\Psi_{OOS}(zI+\hat\Psi)^{-1} S'\tilde\eps_{OOS}
\Big)^2
\;\Big|\; S,S_{OOS}\Big] \\
=\ &
\frac{4z^2}{T^2}\,
E\Big[
\big(
\beta' (zI+\hat\Psi)^{-1}\hat\Psi_{OOS}(zI+\hat\Psi)^{-1} S'\tilde\eps_{OOS}
\big)^2
\;\Big|\; S,S_{OOS}\Big].    
\end{aligned}
\end{equation}
Since
\begin{equation}
\begin{aligned}
E[\hat\eps_{OOS}\hat\eps_{OOS}'\mid  S,S_{OOS}]
=\ & \gs_\eps^2\Bigr(
S_{OOS}(zI+\hat\Psi)^{-1} \frac{1}{T^2}  S'S(zI+\hat\Psi)^{-1}S_{OOS}'+I
\Bigr)\\
=\ & \gs_\eps^2\Bigr(
\frac{1}{T}S_{OOS}(zI+\hat\Psi)^{-1}  \hat\Psi(zI+\hat\Psi)^{-1}S_{OOS}'+I
\Bigr)
\end{aligned}
\end{equation} and by \eqref{eq: tild_eps_OOS}, we obtain
\begin{equation}
\begin{aligned}
& E[\tilde\eps_{OOS}\tilde\eps_{OOS}'\mid S,S_{OOS}] \\
=\ & \frac{\gs_\eps^2}{T^2_{OOS}} S(zI+\hat\Psi)^{-1}S_{OOS}'     \Bigr(
\frac{1}{T}S_{OOS}(zI+\hat\Psi)^{-1}  \hat\Psi(zI+\hat\Psi)^{-1}S_{OOS}'+I
\Bigr)  S_{OOS}(zI+\hat\Psi)^{-1}S' \\
=\ & \frac{\gs_\eps^2}{T T^2_{OOS}} S(zI+\hat\Psi)^{-1}S_{OOS}'S_{OOS}(zI+\hat\Psi)^{-1} \hat\Psi(zI+\hat\Psi)^{-1}S_{OOS}'S_{OOS}(zI+\hat\Psi)^{-1}S'  \\
& + \frac{\gs_\eps^2}{T^2_{OOS}} S(zI+\hat\Psi)^{-1}S_{OOS}'  S_{OOS}(zI+\hat\Psi)^{-1}S'  \\
=\ & \frac{\gs_\eps^2}{T } S(zI+\hat\Psi)^{-1}\hat\Psi_{OOS}(zI+\hat\Psi)^{-1} \hat\Psi(zI+\hat\Psi)^{-1}\hat\Psi_{OOS}(zI+\hat\Psi)^{-1}S' \\
&+ \frac{\gs_\eps^2}{T_{OOS}} S(zI+\hat\Psi)^{-1}\hat\Psi_{OOS}(zI+\hat\Psi)^{-1}S'
\end{aligned}
\end{equation}
For simplicity, let $R:= (zI+\hat\Psi)^{-1}$. Thus,
\begin{equation}
\begin{aligned}
& E\Big[
\big(
\beta' (zI+\hat\Psi)^{-1}\hat\Psi_{OOS}(zI+\hat\Psi)^{-1} S'\tilde\eps_{OOS}
\big)^2
\;\Big|\; S,S_{OOS}\Big] \\
=\ &
\beta' R \hat\Psi_{OOS}R
S'\,
E[\tilde\eps_{OOS}\tilde\eps_{OOS}'\mid S,S_{OOS}]\,
S
R\hat\Psi_{OOS}R\beta \\
=\ & \beta' R\hat\Psi_{OOS}R
S'\, \frac{\gs_\eps^2}{T } SR\hat\Psi_{OOS}R \hat\Psi R \hat\Psi_{OOS} R S' S
R \hat\Psi_{OOS} R \beta  + \beta' R \hat\Psi_{OOS} R
S'\, \frac{\gs_\eps^2}{T_{OOS}} R \hat\Psi_{OOS} R S' S
R \hat\Psi_{OOS} R \beta \\
=\ & \gs_\eps^2 T\,\beta' M_1 \beta
\;+\; \gs_\eps^2 \frac{T^2}{T_{OOS}}\,\beta' M_2 \beta,
\end{aligned}
\end{equation}
where 
\begin{equation}
\begin{aligned}
M_1 &= R\hat\Psi_{OOS}R\hat\Psi R \hat\Psi_{OOS} R\hat\Psi R \hat\Psi_{OOS} R\hat\Psi R\hat\Psi_{OOS}R,\\
M_2 &= R\hat\Psi_{OOS}R\hat\Psi R \hat\Psi_{OOS} R\hat\Psi R\hat\Psi_{OOS}R.   
\end{aligned}
\end{equation}
We now bound $\beta'M_1\beta$ and $\beta'M_2\beta$ explicitly.  
Since $M_1$ and $M_2$ are symmetric, for any vector $\beta$, we have
\begin{equation}
E[\beta' M_i \beta] \;\le\; \|\beta\|^2\,E[\|M_i\|],
\qquad i=1,2.
\end{equation}
Note that, if $\lambda_i$ are the eigenvalues of $\hat\Psi$, then the eigenvalues
of $\hat\Psi R$ are $\dfrac{\lambda_i}{z+\lambda_i} \in[0,1)$, so $\|R\hat\Psi\|\le\;1.$ From the structure of $M_1$ and $M_2$ we then obtain
\begin{equation}\label{eq:M1M2-op-bd}
\begin{aligned}
\|M_1\|
&\;\le\; \|R\hat\Psi_{OOS}\|
\|R\hat\Psi\|\|R\hat\Psi_{OOS}\|\|R\hat\Psi\|
\|R\hat\Psi_{OOS}\| \|R\hat\Psi\| \|R\hat\Psi_{OOS}R\| \le \frac{\|\hat\Psi_{OOS}\|^4}{z^5},\\
\|M_2\|
&\;\le\; \|R\hat\Psi_{OOS}\|
\|R\hat\Psi\|\|R\hat\Psi_{OOS}\|\|R\hat\Psi\|
 \|R\hat\Psi_{OOS}R\| \le  \frac{\|\hat\Psi_{OOS}\|^3}{z^4}.    
\end{aligned}
\end{equation}

Under our sub-Gaussian assumptions on the rows of $S_{OOS}$, Lemma~\ref{lem:SubG-spectral}
applied to $S_{OOS}$ with $k = 3, 4,$ implies that 
\begin{equation}
E [\|\hat\Psi_{OOS}\|^3 ] \le C_3 ,
\qquad E [\|\hat\Psi_{OOS}\|^4 ] \le C_4 
\end{equation}
where $C_3$ and $C_4$ depend only on the sub-Gaussian parameter $K$ but not on $T_{OOS}$ or $P$. Hence,
\begin{equation}\label{eq:betaM1M2}
E[\beta'M_1\beta] \;\le\; \frac{C_4}{z^5}\|\beta\|^2,
\qquad
E[\beta'M_2\beta] \;\le\; \frac{C_3}{z^4}\|\beta\|^2.
\end{equation}
Plugging these bounds into the expression for the conditional second moment
and then taking expectations, we obtain
\begin{equation}
E\Big[
\Big(-\frac{2z}{T}\,
\beta' (zI+\hat\Psi)^{-1}\hat\Psi_{OOS}(zI+\hat\Psi)^{-1} S'\tilde\eps_{OOS}
\Big)^2
\Big] \le\;
4\gs_\eps^2 \|\beta\|^2
\Bigg[
\frac{1}{T}\,\frac{C_4}{z^3}
+ \frac{1}{T_{OOS}}\,\frac{C_3}{z^2}
\Bigg].
\end{equation}
In particular, if $T\to\infty$ and $T_{OOS}\to\infty$, the right-hand side converges
to zero, so the cross term in \eqref{qoos-norm} converges to zero in $L_2$. Now, we have
\begin{equation}\label{qoos-norm_approx}
\frac{1}{T} \hat q_{OOS}'\hat q_{OOS}
\approx\  z^2 \beta ' (zI + \hat \Psi)^{-1} \hat\Psi_{OOS} (zI + \hat \Psi)^{-1} \hat\Psi (zI + \hat \Psi)^{-1} \hat\Psi_{OOS}(zI + \hat \Psi)^{-1}\beta +\ \frac{1}{T} \|\tilde\eps_{OOS}\|^2 .
\end{equation}
Here, by the concentration of quadratic forms (Lemma \ref{lem-quad}), 
\begin{equation}
\begin{aligned}
& \frac{1}{T}  \|\tilde\eps_{OOS}\|^2\\
=\ &  \frac{1}{T_{OOS}^2}   \hat\eps_{OOS}'S_{OOS}(zI+\hat\Psi)^{-1} \frac{1}{T}  S'S(zI+\hat\Psi)^{-1}S_{OOS}'\hat\eps_{OOS}\\
=\ &  \frac{1}{T_{OOS}^2}   \tr [S_{OOS}(zI+\hat\Psi)^{-1} \frac{1}{T}  S'S(zI+\hat\Psi)^{-1}S_{OOS}'\hat\eps_{OOS}\hat\eps_{OOS}']\\
\approx\ &  \frac{1}{T_{OOS}^2}   \tr \Bigr[S_{OOS}(zI+\hat\Psi)^{-1} \frac{1}{T}  S'S(zI+\hat\Psi)^{-1}S_{OOS}'
\gs_\eps^2\Bigr(
S_{OOS}(zI+\hat\Psi)^{-1} \frac{1}{T^2}  S'S(zI+\hat\Psi)^{-1}S_{OOS}'+I
\Bigr)
\Bigr]\\
=\ & \gs_\eps^2  \frac{1}{T_{OOS}^2}   \tr  \Bigr[S_{OOS}(zI+\hat\Psi)^{-1} \frac{1}{T}  S'S(zI+\hat\Psi)^{-1}S_{OOS}'
\Bigr(
S_{OOS}(zI+\hat\Psi)^{-1} \frac{1}{T^2}  S'S(zI+\hat\Psi)^{-1}S_{OOS}'
\Bigr)
\Bigr]\\
&+\ \gs_\eps^2  \frac{1}{T_{OOS}^2}   \tr [S_{OOS}(zI+\hat\Psi)^{-1} \frac{1}{T}  S'S(zI+\hat\Psi)^{-1}S_{OOS}'
]\\
=\ & \gs_\eps^2 \underbrace{ \frac{1}{T}  \tr [(zI+\hat\Psi)^{-1}\hat\Psi(zI+\hat\Psi)^{-1}\hat\Psi_{OOS}(zI+\hat\Psi)^{-1}\hat\Psi(zI+\hat\Psi)^{-1}\hat\Psi_{OOS}
]}_{=0.5\gs_V^2\ by\ \eqref{defn:sigmav}}\\
&+\ \gs_\eps^2  \frac{1}{T_{OOS}}   \tr [(zI+\hat\Psi)^{-1}\hat\Psi(zI+\hat\Psi)^{-1}\hat\Psi_{OOS}
]\\
=\ & \gs_\eps^2\Big(\frac12 \gs_V^2\ +\ \frac{T}{T_{OOS}}\widehat\cL\Big)\,. 
\end{aligned}
\end{equation}
Thus, by \eqref{defn:sigmav} and \eqref{qoos-norm}, 
\begin{equation}\label{qoos-norm1}
\begin{aligned}
\frac14 \sigma^2_I =\ & z^2 \beta ' (zI + \hat \Psi)^{-1} \hat\Psi_{OOS} (zI + \hat \Psi)^{-1} \hat\Psi (zI + \hat \Psi)^{-1} \hat\Psi_{OOS}(zI + \hat \Psi)^{-1}\beta\\
\approx\ &  \frac{1}{T} \|\hat q_{OOS}\|^2-\frac{1}{T} \|\tilde\eps_{OOS}\|^2\,. 
\end{aligned}
\end{equation}
Plugging this expression into \eqref{eq:sigma2_MSE2}, we have
\begin{equation}
\begin{aligned}
\gs_{MSE}^2\ =\ & \frac{2\dfrac{T}{T_{OOS}}\gs_\eps^4+\gs_\eps^4\gs_V^2+\gs_\eps^2 \gs_I^2+\gs_\eps^2\gs_{I,OOS}^2}{(1+\widehat\cL)^2}\\
=\ & \frac{2\dfrac{T}{T_{OOS}}\gs_\eps^4+\gs_\eps^4\gs_V^2+4\gs_\eps^2 \Bigr[ \dfrac{1}{T}\|\hat q_{OOS}\|^2\ -\ \gs_\eps^2\Big(\dfrac12 \gs_V^2\ +\ \dfrac{T}{T_{OOS}}\widehat\cL\Big)\Bigr]+4\gs_\eps^2\dfrac{T}{T_{OOS}} \Big(\widehat\cB(z)+\ \gs_\eps^2 \widehat\cL \Big)}{(1+\widehat\cL)^2}\\
=\ & \frac{2\dfrac{T}{T_{OOS}}\gs_\eps^4-\gs_\eps^4\gs_V^2+4\gs_\eps^2 \dfrac{1}{T}\|\hat q_{OOS}\|^2\ + 4 \gs_\eps^2 \dfrac{T}{T_{OOS}} \widehat\cB(z) }{(1+\widehat\cL)^2}
\end{aligned}
\end{equation}
The proof of Proposition \ref{super-short-estim-ridge} is complete. 
\end{proof}

\subsection{The Final Pivotal Estimator} 

By the continuous mapping theorem, in estimating $\gs_{R^2},$ we can replace any quantity with its consistent estimator, and we can replace $\gs_{R^2}$ with any consistent upper bound. We use this observation repeatedly below. We have  
\begin{equation}
\begin{aligned}
\gS_{1,2}\ =\ & 2\gs_\eps^4+4\gs_\eps^2 \beta'\hat\Psi_{OOS}(\beta-\hat\beta)\\
\underbrace{\approx}_{\eqref{betahatbeta}}\ &2\gs_\eps^4+4\gs_\eps^2 \Bigr(\beta'\hat\Psi_{OOS}\beta-E_{OOS}[y'S]\hat\beta \Bigr)\\
\approx\ & 2\gs_\eps^4+4\gs_\eps^2 \Bigr(MSE_{OOS}(0)-\gs_\eps^2-E_{OOS}[y'S]\hat\beta \Bigr)\\
=\ & -2\gs_\eps^4+4\gs_\eps^2 \Bigr(MSE_{OOS}(0)-E_{OOS}[y'S]\hat\beta \Bigr)\\
\gS_{2,2}\ =\ & 2\gs_\eps^4+4\gs_\eps^2 \beta'\Psi_{OOS}\beta\\
\approx\ & 2\gs_\eps^4+4\gs_\eps^2 \Bigr(MSE_{OOS}(0)-\gs_\eps^2 \Bigr)\\
\approx\ & -2\gs_\eps^4+4\gs_\eps^2 MSE(0)\\
\widehat\cB(z)\ \approx\ & \widehat{MSE}-(1+\widehat\cL)\gs_\eps^2
\end{aligned}
\end{equation}
Thus, 
\begin{equation}
\begin{aligned}
&(1+\widehat\cL)^2\hat\gs_{R^2}MSE(0)^4\\
=\ & MSE(0)^2\gS_{1,1}+(\widehat{MSE})^2\gS_{2,2}-2MSE(0)\widehat{MSE} \gS_{1,2}\\
=\ & MSE(0)^2 \frac{T_{OOS}}{T}\gs^2_{MSE}(1+\widehat\cL)^2+(\widehat{MSE})^2 \Bigr(-2\gs_\eps^4+4\gs_\eps^2 MSE(0) \Bigr)\\
&-2MSE(0)\widehat{MSE} \Bigr[-2\gs_\eps^4+4\gs_\eps^2 \Bigr(MSE_{OOS}(0)-E_{OOS}[y'S]\hat\beta \Bigr) \Bigr]\\
=\ & MSE(0)^2\frac{T_{OOS}}{T}\Bigr[
2\frac{T}{T_{OOS}}\gs_\eps^4-\gs_\eps^4\gs_V^2+4\gs_\eps^2  \frac{1}{T}  \|\hat q_{OOS}\|^2\ +4\gs_\eps^2\frac{T}{T_{OOS}} \Big(
\widehat{MSE}-(1+\widehat\cL)\gs_\eps^2
\Big)\Bigr]\\
&+(\widehat{MSE})^2 \Bigr(-2\gs_\eps^4+4\gs_\eps^2 MSE(0) \Bigr) -2MSE(0)\widehat{MSE} \Bigr[-2\gs_\eps^4+4\gs_\eps^2 \Bigr(MSE_{OOS}(0)-E_{OOS}[y'S]\hat\beta \Bigr)\Bigr]\\
=\ & A_2\gs_\eps^4\ +\ A_1\gs_\eps^2
\end{aligned}
\end{equation}
where
\begin{equation}
\begin{aligned}
A_2\ =\ & MSE(0)^2\frac{T}{T_{OOS}}\Big(
2\frac{T}{T_{OOS}}-\gs_V^2-4\frac{T}{T_{OOS}} 
(1+\widehat\cL)
\Big)\\ 
&-\ 2(\widehat{MSE})^2\ +\ 4MSE(0)\widehat{MSE};\\
A_1\ =\ & MSE(0)^2\frac{T}{T_{OOS}}\Bigr(4 \frac{1}{T}\|\hat q_{OOS}\|^2
+4\frac{T}{T_{OOS}}
\widehat{MSE}\Bigr)\\
&+4(\widehat{MSE})^2 MSE(0)\ -\ 8MSE(0)\widehat{MSE} \Bigr(MSE_{OOS}(0)-E_{OOS}[y'S]\hat\beta \Bigr)\,.
\end{aligned}
\end{equation}
We can now build an asymptotically consistent, pivotal upper bound on $\gs_{R^2}^2$ as 
\begin{equation}
\gs_{R^2}^2\ \le\ \min_{\gs_\eps^2\in [0,\min_z MSE_{OOS}(\hat\beta(z))/(1+\hat\cL(z))]}(A_2\gs_\eps^4\ +\ A_1\gs_\eps^2)\ +\ O(T^{-1/2})\,.
\end{equation}

\section{Proof of Proposition \ref{prop:theoldllg}}

Our goal is to show that 
\begin{equation}
\widehat\cL=\dfrac1T\tr \Bigr(\Psi\hat\Psi(zcI+\hat\Psi)^{-2}\Bigr)\,
\end{equation}
and 
\begin{equation}\label{final-form-a}
\widetilde\cL\ =\ \frac{\dfrac1T\tr \bigl((zcI+SS'/T)^{-2} \bigr)}{\Bigr(\dfrac1T\tr \bigl((zcI+SS'/T)^{-1} \bigr)\Bigr)^2}-1\,. 
\end{equation}
satisfy $\widehat\cL\approx\widetilde\cL\,.$ From Proposition \ref{prop:exptrace}, we know that 
\begin{equation}
\dfrac1T\tr \Bigr(\Psi(zcI+\hat\Psi)^{-1}\Bigr)\ \approx\ \hat\xi(z;c)\ =\ -1\ +\ \frac{1}{1-c+cz\hat m(-z)}\,.
\end{equation}
and
\begin{equation}
-\dfrac1T\tr \Bigr(\Psi(zcI+\hat\Psi)^{-2}\Bigr)\ =\ \hat\xi'(z;c)\ =\ -\ \frac{c(\hat m(-z)-z\hat m'(-z))}{(1-c+cz\hat m(-z))^2}\,,
\end{equation}
so that 
\begin{equation}
\begin{aligned}
1+\widehat\cL &=\ 1+\hat\xi(z;c)+z\hat\xi'(z;c)\\
&\approx \frac{1}{1-c+cz\hat m(-z)}-\frac{cz(\hat m(-z)-z\hat m'(-z))}{(1-c+cz\hat m(-z))^2}\\
&=\ \frac{1-c+cz^2\hat m'(-z)}{(1-c+cz\hat m(-z))^2}\,.
\end{aligned}
\end{equation}
Let 
\begin{equation}
\tilde m(z)\ =\ (1-c)z^{-1}+c \hat m(-z)\,. 
\end{equation}
Then, 
\begin{equation}
\frac{d}{dz}\tilde m(-z)\ =\ -(1-c)z^{-2}-c\hat m'(-z)
\end{equation}
and, hence, we get 
\begin{equation}
\begin{aligned}
&1+\widehat\cL \approx\ \frac{-\frac{d}{dz}\tilde m(-z)}{\tilde m(-z)^2}\,. 
\end{aligned}
\end{equation}
Now, by direct calculation, the matrices $SS'$ and $S'S$ have the same eigenvalues, up to $P-T$ zero eigenvalues. As a result, 
\begin{equation}
\tilde m(-z)\ =\ (1-c)z^{-1}+c  P^{-1}\tr((zI+S'S/T)^{-1})\ =\ T^{-1}\tr((zI+SS'/T)^{-1})\,,
\end{equation}
and the claim follows.

\section{Bayesian Risk and Optimal Ridge}
\label{bayes}

In this section, we provide Bayesian foundation for ridge regression and discuss when ridge and its modifications are Bayes-optimal and, hence, cannot be dominated by any other machine learning model (linear or nonlinear). What is special about our setting is that, in high dimensions, due to concentration phenomena (where random objects in high dimensions become non-random), our bounds hold almost surely and not just in expectation. 

\subsection{Bayesian Optimality}

Suppose that nature samples a parameter vector $\theta\in \R^D$ from a distribution $p(\theta)d\theta.$ The agent observes i.i.d. samples $X_t\sim p(X|\theta),\ \bX=(X_t)_{t=1}^T.$ His objective is to solve a utility optimization problem 
\begin{equation}
\max_{\pi(\bX)}E[U(\pi(\bX),\theta)]\ =\ \max_{\pi(\bX)}\int E[U(\pi(\bX),\theta)|\theta]p(\theta)d\theta\,. 
\end{equation}
In the real world, the the dimension $D$ is large enough, no agent can know (or have any reasonable algorithm to finy) the true distribution $p(\theta)$ from which nature samples $\theta.$ A rational agent is thus forced to choose a prior $p^{subjective}(\theta),$ and then learn from the observations of $\bX.$ When $D$ is large relative to $T$, this learning will not converge due to limits to learning. This convergence breaks down {\it even if the agent has the optimal prior $p(\theta)$.} In this section, we discuss the implications of these fundamental results for limits to learning. 

Suppose that the agent does have the (infeasible) optimal prior $p(\theta)$. By the law of iterated expectations, we can rewrite his objective as 
\begin{equation} \label{eq-iter}
E[U(\pi(\bX),\theta)]\ =\ E[E[U(\pi(\bX),\theta)|\bX]]\,.
\end{equation}
Let 
\begin{equation}
\pi_*(\bX)\ =\ \arg\max_{\pi}E[U(\pi,\theta)|\bX]\ =\ \arg\max_{\pi}\int U(\pi,\theta)p(\theta|\bX)d\theta
\end{equation}
be the optimal Bayesian policy. Then, we get the simple, classical result. 

\begin{theorem}[Bayesian Policies are Optimal]\label{th1} Suppose that 
\begin{equation} \label{eq: opt-assum}
E[|E[U(\pi_*(\bX),\theta)|\bX]|]\ <\ \infty\,.
\end{equation}
Then, 
\begin{equation}
\pi_*(\bX)\ \in\ \arg\max_{\pi(\bX)}E[U(\pi(\bX),\theta)]
\end{equation}
\end{theorem}

\subsection{Bayesian Optimality and The Best Feasible Policy}

In this subsection, we apply Theorem \ref{th1} to the linear predictive setting studied in the main body of the paper. 

\begin{lemma}[Bayesian updating]\label{lem1} Suppose that 
\begin{equation}
y_t\ =\ \beta'S_{t-1}\ +\ \eps_t\,.
\end{equation}
Suppose also that the agent's prior about $\beta$ is $\beta\sim N(0,\Sigma_\beta)$. 
Let $y=(y_\tau)_{\tau=1}^t,$ and $S=(S_\tau)_{\tau=0}^{t-1}\in \R^{t\times P}.$ The agent's posterior distribution is Gaussian, $\beta|\cF_t\ \sim\ N(\hat\beta_t,\ \hat\Sigma_t),$ with 
\begin{equation}
\begin{aligned}
&\hat \beta_{1,t}\ =\ (\gs_\eps^2\Sigma_\beta^{-1}+ S' S)^{-1}S'y\\
&\hat\Sigma_{1,t}\ =\ (\gs_\eps^2 \gS_\beta^{-1}+S'S)^{-1}\gs_\eps^2\,.
\end{aligned}
\end{equation}
\end{lemma}

\begin{proof}[Proof of Lemma \ref{lem1}] In the vector form, we have 
\begin{equation}
d\ =\ S\beta\ +\ \eps
\end{equation}
The agent believes that the vector $\beta\in \R^P$ is sampled at time zero from $\mathcal{N}(0,\Sigma_\beta)$. Signals have a covariance matrix $\psi_t.$ By the Gaussian projection theorem: 
\begin{equation}
\begin{aligned}
&E[\beta|d]\ =\ \Cov[\beta,d|S]\Cov[d|S]^{-1}d\\
&\Var[\beta|d]\ =\ \Var[\beta]\ -\ \Cov[\beta,d|S] \Cov[d|S]^{-1} \Cov[\beta,d|S]'
\end{aligned}
\end{equation}
We have 
\begin{equation}
\begin{aligned}
&\Cov[d|S]\ =\ (\gs_\eps^2I_{t\times t}+S\Sigma_\beta S')\\
&\Cov[\beta,d|S]\ =\ E[\beta d']\ =\ \Sigma_\beta S'\\
\end{aligned}
\end{equation}
and hence
\begin{equation}
\Var[\beta|d]\ =\ \Sigma_\beta\ -\ \Sigma_\beta S' (\gs_\eps^2I_{t\times t}+S\Sigma_\beta S')^{-1} S\Sigma_\beta.
\end{equation}
Thus, his posterior after $t$ observations is thus $\beta\approx N(\hat\beta_{1,t},\hat\Sigma_{1,t}),$ where 
\begin{equation}
\begin{aligned}
&\hat \beta_{1,t}\ =\ \Sigma_\beta S'(\gs_\eps^2I_{t\times t}+S\Sigma_\beta S')^{-1}d\\
&\hat\Sigma_{1,t}\ =\ \Sigma_\beta\ -\ \Sigma_\beta S'(\gs_\eps^2I_{t\times t}+S\Sigma_\beta S')^{-1}S\Sigma_\beta
\end{aligned}
\end{equation}
Now, define $\tilde S=S\Sigma_\beta^{1/2}.$ Then, 
\begin{equation}\label{sigma-beta1}
\begin{aligned}
&\Sigma_\beta S'(\gs_\eps^2I_{t\times t}+S\Sigma_\beta S')^{-1}\\
=\ & \Sigma_\beta^{1/2}\tilde S'(\gs_\eps^2I_{t\times t}+\tilde S\tilde S')^{-1}S'\\
=\ & \Sigma_\beta^{1/2}(\gs_\eps^2I_{P\times P}+\tilde S'\tilde S)^{-1}\tilde S'\\
=\ & (\gs_\eps^2\Sigma_\beta^{-1}+ S' S)^{-1}S'
\end{aligned}
\end{equation}
so that 
\begin{equation}
\begin{aligned}
&\hat \beta_{1,t}\ =\ (\gs_\eps^2\Sigma_\beta^{-1}+ S' S)^{-1}S'y\\
&\hat\Sigma_{1,t}\ =\ \Sigma_\beta\ -\ \Sigma_\beta S'(\gs_\eps^2I_{t\times t}+S\Sigma_\beta S')^{-1}S\Sigma_\beta=\ (\gs_\eps^2 \gS_\beta^{-1}+S'S)^{-1}\gs_\eps^2\,.
\end{aligned}
\end{equation}
\end{proof}

We can now prove the following result.
\begin{proposition}\label{prop:optimal prior} Suppose that $\beta$ is sampled at time zero from $N(0,\gS_\beta).$ Then, 
\begin{equation}
\pi_t^*\ =\ S_t'(\gs_\eps^2\Sigma_\beta^{-1}+ S' S)^{-1}S'y
\end{equation}
is Bayes optimal in the sence that 
\begin{equation}
E_\beta[(y_{t+1}-\pi_t^*)^2]\ \le\ E_\beta[(y_t-\Pi(S,y))^2]
\end{equation}
for any map $\Pi:\R^{t(P+1)}\to\R.$
\end{proposition}

The proposition \ref{prop:optimal prior} states a classic result: If an economic agent knows the true optimal prior from which the $\beta$ vector is sampled, then the Bayes rule is optimal on average: No other machine learning algorithm $\Pi(\cdot,\cdot)$ can beat Bayesian learning with an optimally chosen prior. 

By direct calculation, \eqref{sigma-beta1} implies the following result. 

\begin{lemma}\label{lem-equiv} Prediction $\pi_t^*$ coincides with that in a linear ridge regression with $z=\gs_\eps^2$ and $S_t$ replaced by $\tilde S_t=\Sigma_\beta^{1/2}S_t:$
\begin{equation}
\pi_t^*\ =\ \tilde S_t'(\gs_\eps^2I+ \tilde S' \tilde S)^{-1}\tilde S'y
\end{equation}
\end{lemma}

Lemma \ref{lem-equiv} combined with Theorem \ref{llg1} implies the following result.

\begin{theorem}\label{llg1-feas} For any Machine Learning model $\Pi(S,y)$, we have 
\begin{equation}
E_\beta[(y_{T+1}-\Pi(S,y))^2]\ \ge\ \underbrace{E_\beta[\lim\inf (y_{T+1}-\pi_T^*)^2]}_{best\ feasible\ MSE}\ \ge\ (1+\cL_\beta(z;c))\gs_\eps^2\,,
\end{equation}
where 
\begin{equation}
\cL_\beta(z;c)\ =\ \lim \frac{ \frac{1}{T}  \tr((\gs_\eps^2I+\tilde S\tilde S')^{-2}))}{( \frac{1}{T}  \tr((\gs_\eps^2I+\tilde S\tilde S')^{-1}))^2}-1\,.
\end{equation}
Furthermore, the best feasible MSE asymptotics is given by 
\begin{equation}\label{287}
E_T[(y_{T+1}-\pi_T^*)^2]\ - \underbrace{(1+\cL_\beta(z;c))((Z_*^\beta)^2\tr\Bigr(\tilde\Psi(\tilde\Psi+Z_*^\beta I)^{-2}\Sigma_\beta\Bigr)+\gs_\eps^2)}_{best\ feasible\ asymptotic\ MSE}\ \to\ 0
\end{equation}
in probability, where $Z_*^\beta$ is defined with $\tilde S$ instead of $S,$ and where $\tilde\Psi\ =\ \Sigma_\beta^{1/2}\Psi\Sigma_\beta^{1/2}=E[\tilde S'\tilde S].$
\end{theorem}

\begin{proof}[Proof of Theorem \ref{llg1-feas}] From the proof of Theorem \ref{llg1}, we know that the asymptotic MSE is given by 
\begin{equation}
(1+\cL_\beta(z;c))(Z_*^2\beta'\Sigma_\beta^{1/2}\Psi\Sigma_\beta^{1/2}(\Sigma_\beta^{1/2}\Psi\Sigma_\beta^{1/2}+Z_*^\beta I)^{-2}\beta+\gs_\eps^2)\,. 
\end{equation}
By the concentration of quadratic forms, the first term is converging to 
\begin{equation}
\beta'\tilde\Psi(\tilde\Psi+Z_*^\beta I)^{-2}\beta\approx\ 
\tr\Bigr(\tilde\Psi(\tilde\Psi+Z_*^\beta I)^{-2}\Sigma_\beta\Bigr)
\end{equation}
where we have defined 
\begin{equation}
\tilde\Psi\ =\ \Sigma_\beta^{1/2}\Psi\Sigma_\beta^{1/2}\,. 
\end{equation}

\end{proof}

It would be great to develop techniques for computing the best feasible MSE in \eqref{287}. However, this would require estimating $\Sigma_\beta$ and this is a highly complex task that we leave for future research. 

\section{GARCH Simulation}\label{app:garch}

We draw innovations
\begin{equation} \label{eq: GARCH1}
z_t \sim \mathcal{N}(0,1), \qquad t = 1,\dots,T,    
\end{equation}
and compute the conditional variance according to
\begin{equation}\label{eq: GARCH2}
\sigma_t^2
\;=\;
\omega \,+\, \alpha\, y_{t-1}^2 \,+\, \beta\, \sigma_{t-1}^2,
\qquad t \ge 2,    
\end{equation}
with an initialization such as
\begin{equation}\label{eq: GARCH3}
\sigma_1^2 \;=\; \frac{\omega}{1 - \alpha - \beta}.    
\end{equation}
The GARCH(1,1) observations are then generated as
\begin{equation}\label{eq: GARCH4}
y^{GARCH(1,1)}_t \;=\; \sigma_t\, z_t, 
\qquad t = 1,\dots,T.    
\end{equation}
We use $\omega = 0.5, \alpha = 0.05, \beta = 0.9$.

\section{Recursive Ridge}\label{app:rec-ridge}

We proceed following the approach of \cite{yan2017fundamental, chen2023high, li2025machine}: Given the basic \cite{welch2008comprehensive} signals transformed using the Procedure \ref{proc:construction}, we build pairwise sums and products of these variables and their non-linear transformations. Then, we pre-select the most powerful signals based on their in-sample importance using a modification of the empirical Bayes methodology of \cite{chen2023high}. And then, we run a ridge regression on these pre-selected signals and follow the same procedure as with the random feature ridge regression from the main text.

\section{Plots with $z_{ref}=1$}
\label{plots:z_ref=1}

In the main text, we use \eqref{zref} with $z_{ref}=0.01$ because, to achieve a large $\widehat\cL,$ we need a small $z.$ Here, we report the results with $z_{ref}=1$ in \eqref{zref} to show how $R^2_{OOS}$ improves, while the LLG-correction becomes negligible.

\begin{figure}
\centering
\includegraphics[width=1\linewidth]{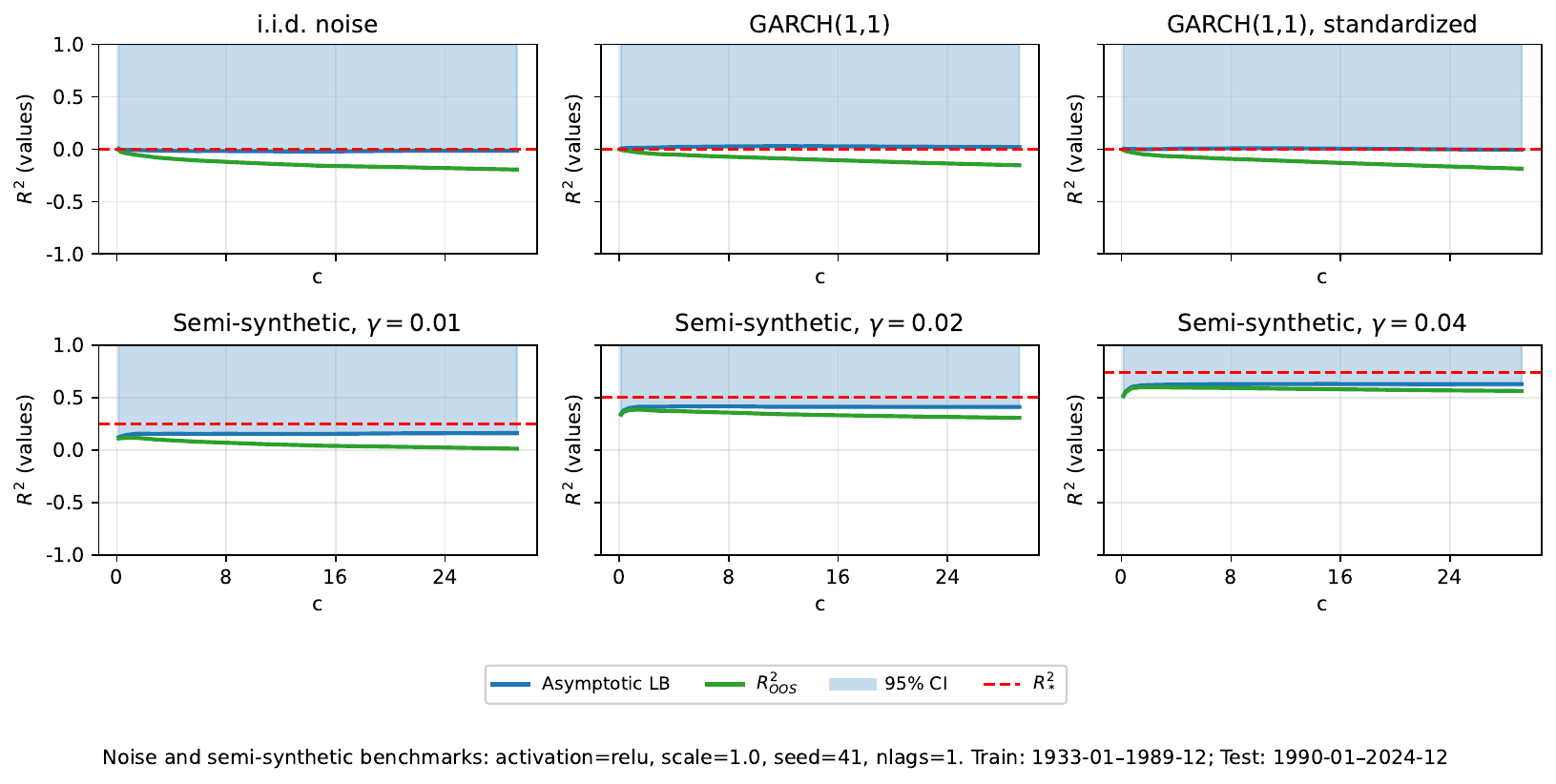}
\caption{Semi-synthetic simulation \eqref{semi-synth}, with activation=$\text{ReLu}$ and $z_{ref}=1$ in \eqref{zref}. In-sample period is 1933-01 to 1989-12; OOS period is 1990-01 to 2024-12. i.i.d. noise has $y_{t+1}=\eps_{t+1}\sim \cN(0,1)$. GARCH(1,1) has $y_{t+1}=\eps_{t+1}$ being a GARCH(1,1) noise defined in Appendix \ref{app:garch}. Asymptotic Lower Bound is given by \eqref{llg3-first}. The shaded region is the one-sided confidence band for $R^2_*.$ The lower bound of the shaded region is \eqref{low-conf}. Horizontal axis is statistical complexity $c=P_1/T,$ where $P_1$ is the number of random features \eqref{def:random-feat}, increasing from $P_1=100$ to $P_1=20000.$ 
$R^2_*$ is computed in \eqref{semi-r2}. Values of $R^2_{OOS}<-1$ are not shown. $\gamma$ values are selected to achieve $R^2_*$ of $0,\ 0.25,\ 0.5,\ 0.75$, respectively.}
\label{fig:semi1_z_ref=1}
\end{figure}

\begin{figure}
\centering
\includegraphics[width=1\linewidth]{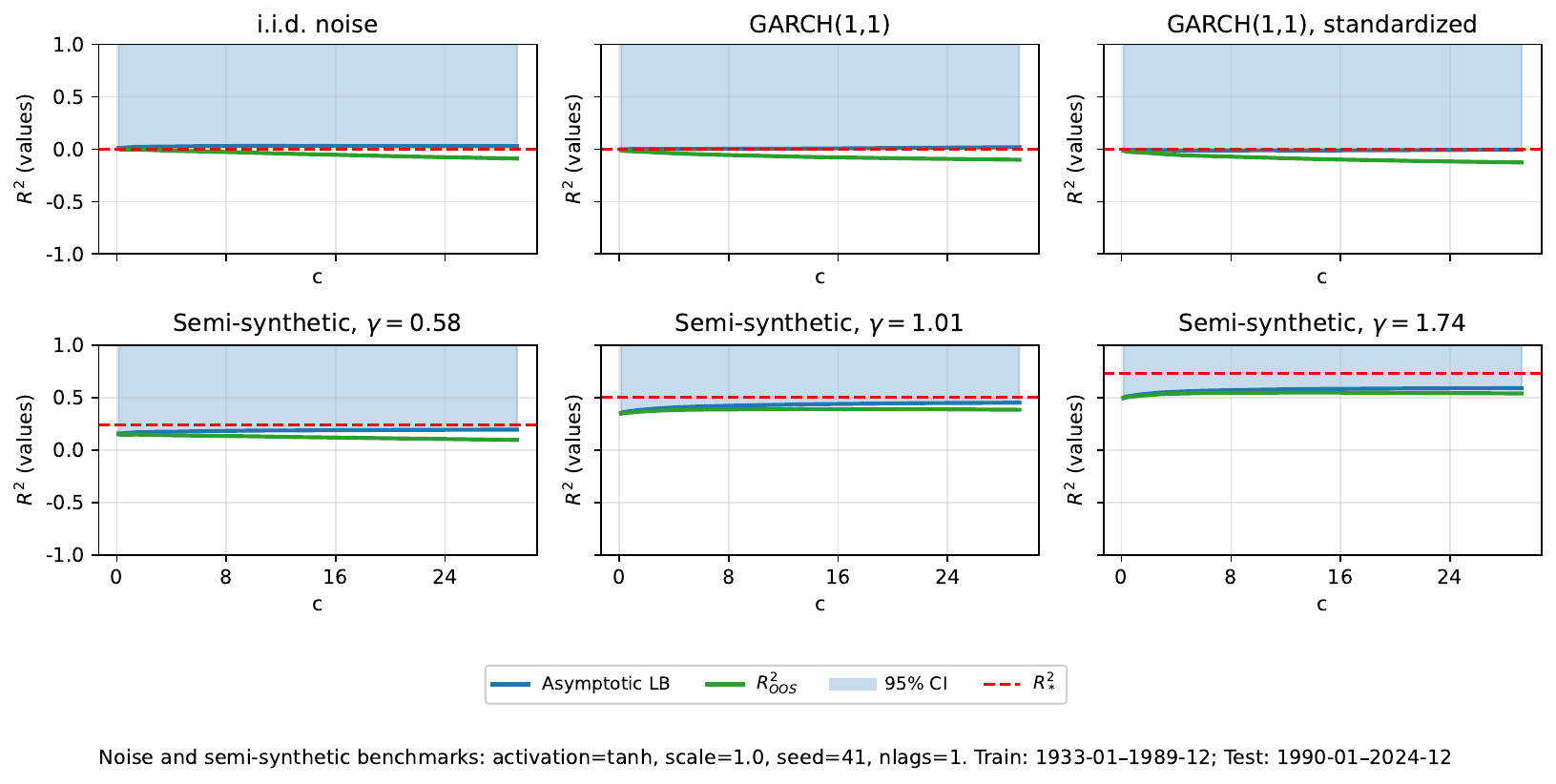}
\caption{Semi-synthetic simulation \eqref{semi-synth}, with activation=$\text{tanh}$ and $z_{ref}=1$ in \eqref{zref}. In-sample period is 1933-01 to 1989-12; OOS period is 1990-01 to 2024-12. i.i.d. noise has $y_{t+1}=\eps_{t+1}\sim \cN(0,1)$. GARCH(1,1) has $y_{t+1}=\eps_{t+1}$ being a GARCH(1,1) noise defined in Appendix \ref{app:garch}. Asymptotic Lower Bound is given by \eqref{llg3-first}. The shaded region is the one-sided confidence band for $R^2_*.$ The lower bound of the shaded region is \eqref{low-conf}. Horizontal axis is statistical complexity $c=P_1/T,$ where $P_1$ is the number of random features \eqref{def:random-feat}, increasing from $P_1=100$ to $P_1=20000.$ 
$R^2_*$ is computed in \eqref{semi-r2}. Values of $R^2_{OOS}<-1$ are not shown. $\gamma$ values are selected to achieve $R^2_*$ of $0,\ 0.25,\ 0.5,\ 0.75$, respectively.}
\label{fig:semi2_z_ref=1}
\end{figure}

\begin{figure}
\centering
\includegraphics[width=1.\linewidth]{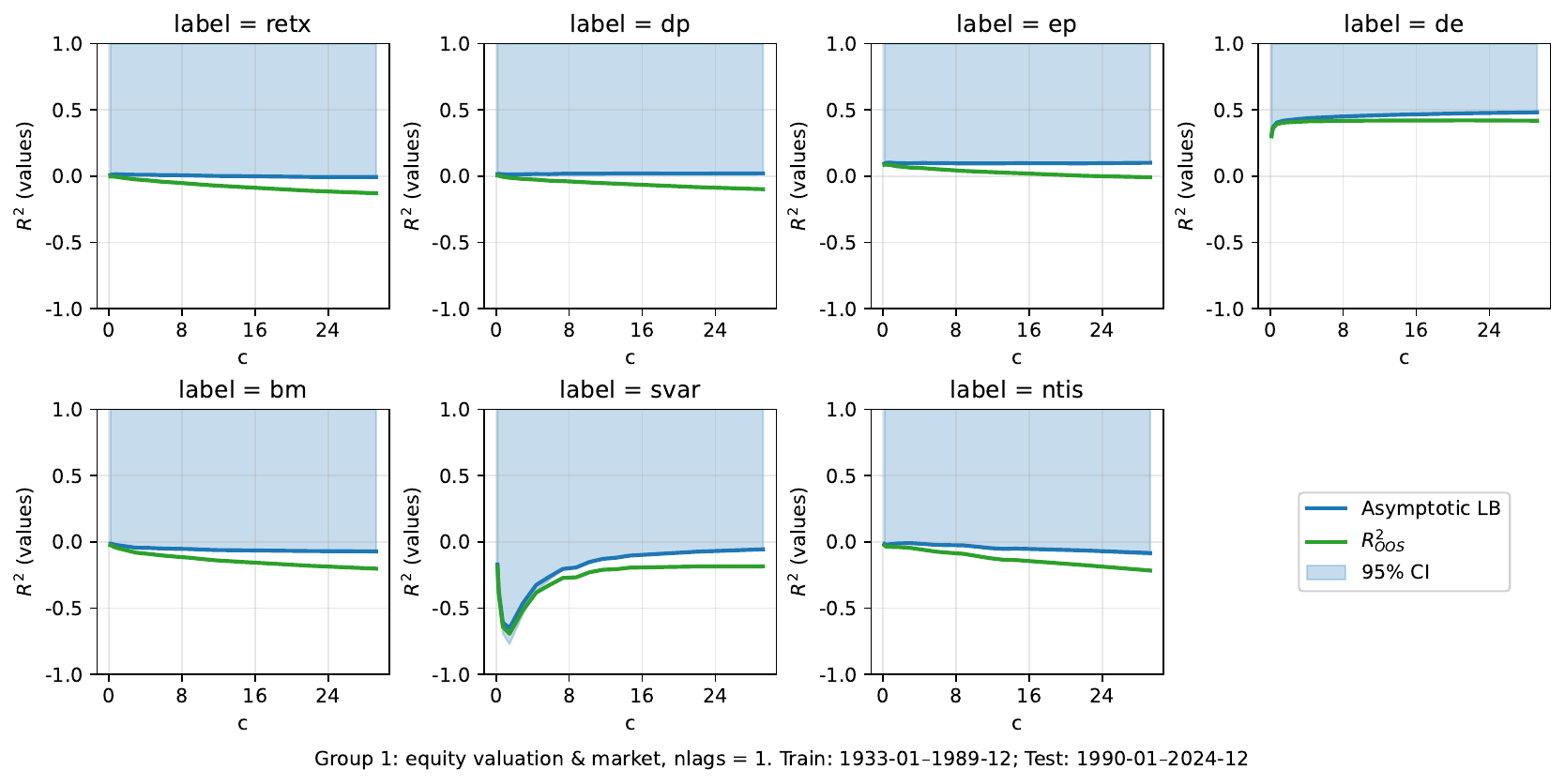}
\caption{Predicting \cite{welch2008comprehensive} variables from {\bf Group one} processed according to Procedure \ref{proc:construction}. Signals are the \cite{welch2008comprehensive} 14 variables and excess returns.  In-sample period is 1933-01 to 1989-12; OOS period is 1990-01 to 2024-12. Asymptotic Lower Bound is given by \eqref{llg3-first}. The shaded region is the one-sided confidence band for $R^2_*.$ The lower bound of the shaded region is \eqref{low-conf}. Horizontal axis is statistical complexity $c=P_1/T,$ where $P_1$ is the number of random features \eqref{def:random-feat} with {\bf activation=}$\text{tanh}$ and $z_{ref}=1$, increasing from $P_1=100$ to $P_1=20000.$ Values of $R^2_{OOS}<-1$ are not shown.}
\label{fig:g1-tanh_z_ref=1}
\end{figure}

\begin{figure}
\centering
\includegraphics[width=1.\linewidth]{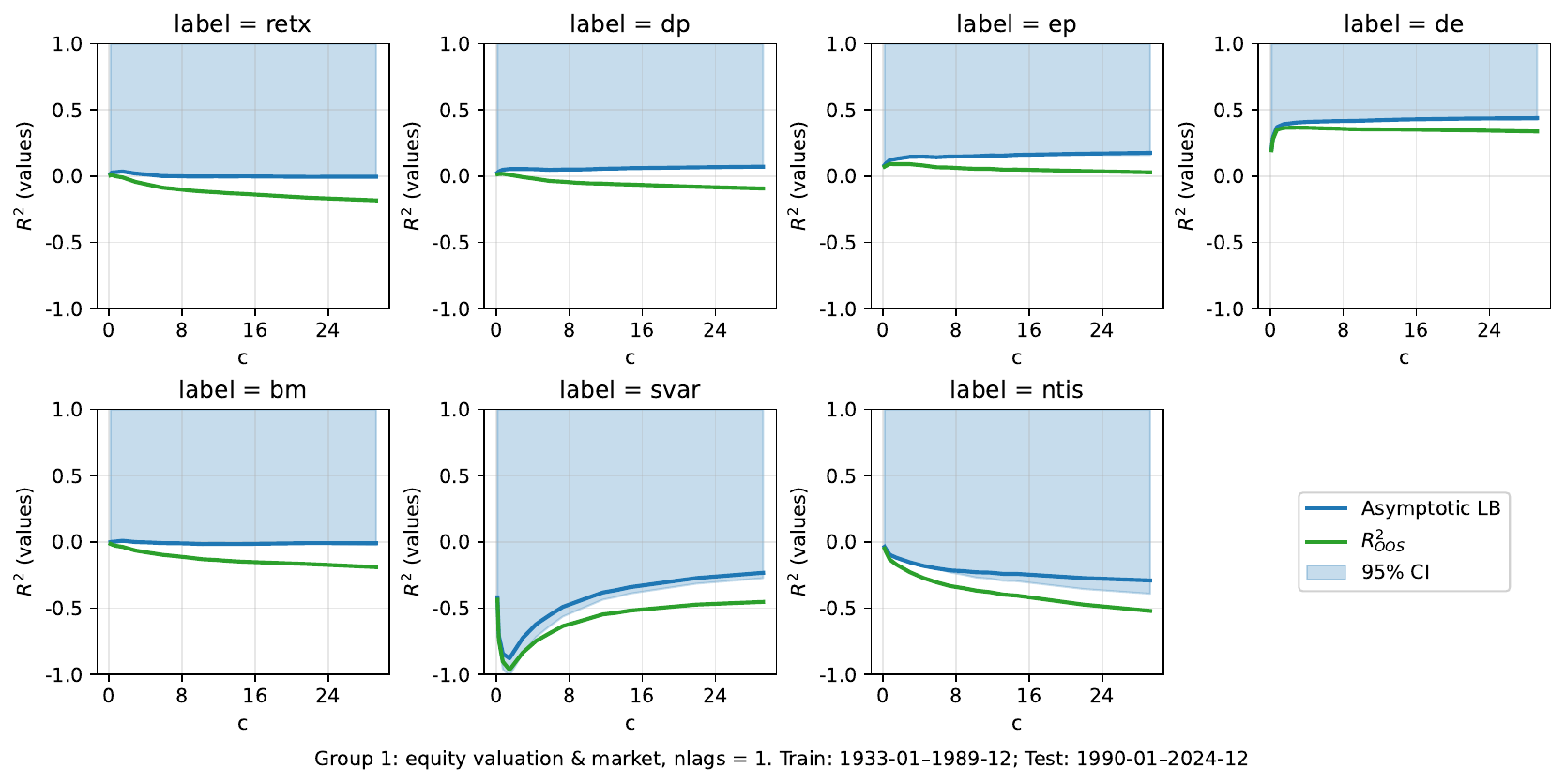}
\caption{Predicting \cite{welch2008comprehensive} variables from {\bf Group one} processed according to Procedure \ref{proc:construction}.  Signals are the \cite{welch2008comprehensive} 14 variables and excess returns.  In-sample period is 1933-01 to 1989-12; OOS period is 1990-01 to 2024-12. Asymptotic Lower Bound is given by \eqref{llg3-first}. The shaded region is the one-sided confidence band for $R^2_*.$ The lower bound of the shaded region is \eqref{low-conf}. Horizontal axis is statistical complexity $c=P_1/T,$ where $P_1$ is the number of random features \eqref{def:random-feat} with {\bf activation=}$\text{ReLu}$ and $z_{ref}=1$, increasing from $P_1=100$ to $P_1=20000.$ Values of $R^2_{OOS}<-1$ are not shown.}
\label{fig:g1-relu_z_ref=1}
\end{figure}

\begin{figure}
\centering
\includegraphics[width=1.\linewidth]{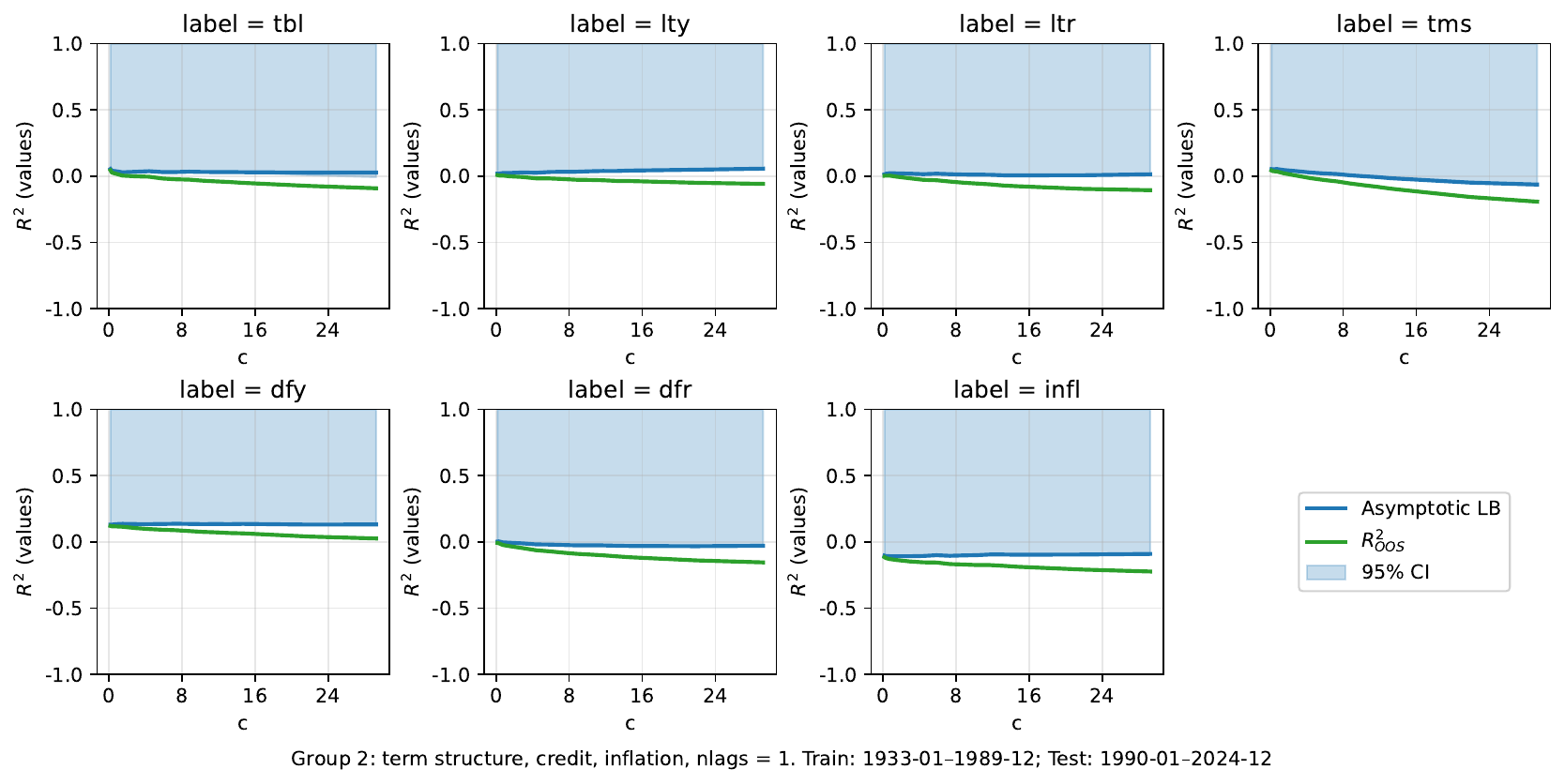}
\caption{Predicting \cite{welch2008comprehensive} variables from {\bf Group two} processed according to Procedure \ref{proc:construction}.  Signals are the \cite{welch2008comprehensive} 14 variables and excess returns.  In-sample period is 1933-01 to 1989-12; OOS period is 1990-01 to 2024-12. Asymptotic Lower Bound is given by \eqref{llg3-first}. The shaded region is the one-sided confidence band for $R^2_*.$ The lower bound of the shaded region is \eqref{low-conf}. Horizontal axis is statistical complexity $c=P_1/T,$ where $P_1$ is the number of random features \eqref{def:random-feat} with {\bf activation=}$\text{tanh}$ and $z_{ref}=1$, increasing from $P_1=100$ to $P_1=20000.$ Values of $R^2_{OOS}<-1$ are not shown.}
\label{fig:g2-tanh_z_ref=1}
\end{figure}

\begin{figure}
\centering
\includegraphics[width=1.\linewidth]{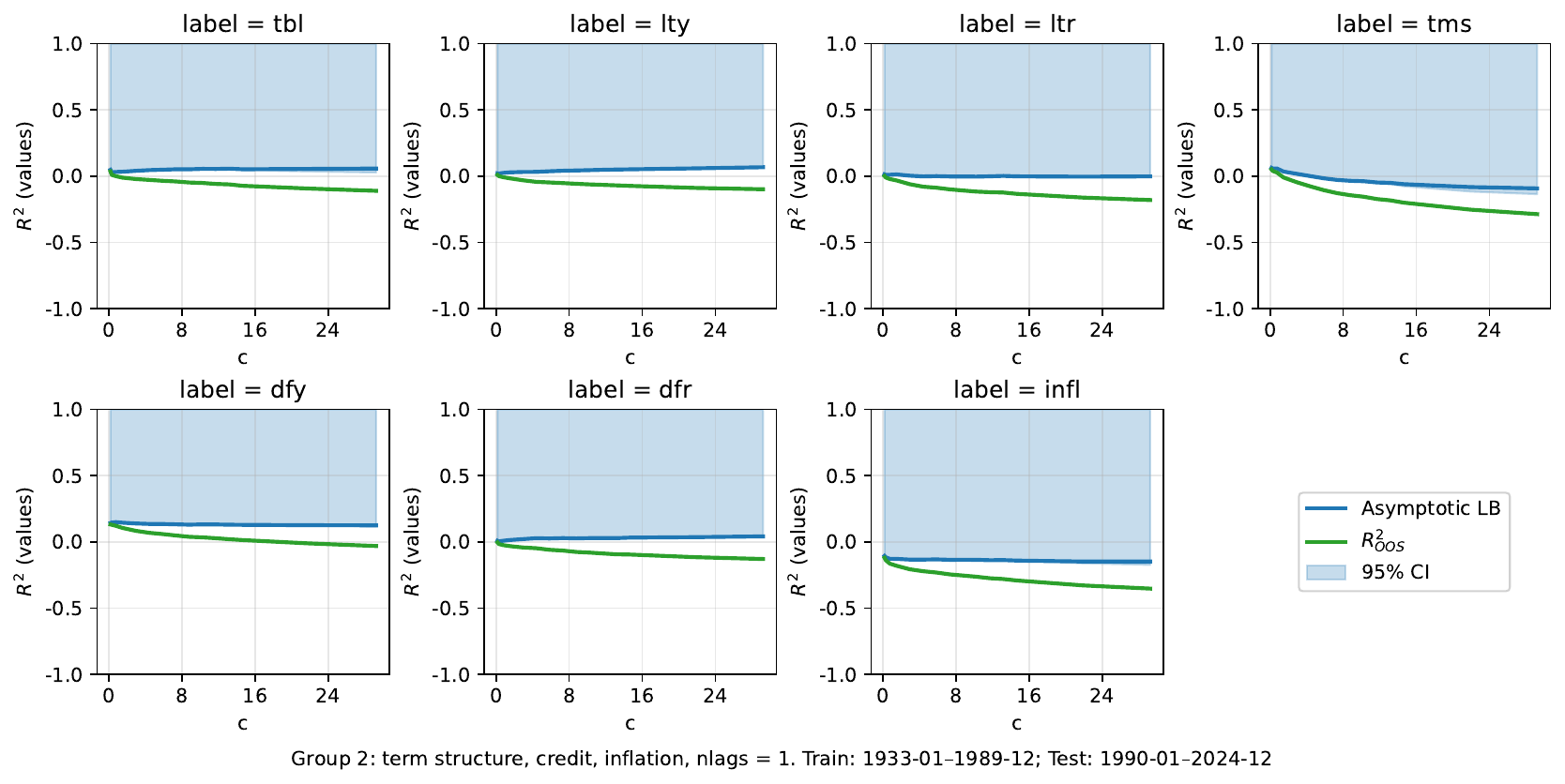}
\caption{Predicting \cite{welch2008comprehensive} variables from {\bf Group two} processed according to Procedure \ref{proc:construction}.  Signals are the \cite{welch2008comprehensive} 14 variables and excess returns.  In-sample period is 1933-01 to 1989-12; OOS period is 1990-01 to 2024-12. Asymptotic Lower Bound is given by \eqref{llg3-first}. The shaded region is the one-sided confidence band for $R^2_*.$ The lower bound of the shaded region is \eqref{low-conf}. Horizontal axis is statistical complexity $c=P_1/T,$ where $P_1$ is the number of random features \eqref{def:random-feat} with {\bf activation=}$\text{ReLu}$ and $z_{ref}=1$, increasing from $P_1=100$ to $P_1=20000.$ Values of $R^2_{OOS}<-1$ are not shown.}
\label{fig:g2-relu_z_ref=1}
\end{figure}

\end{document}